%% file: main.tex
\documentclass[twoside,11pt]{article}
\pdfoutput=1

\usepackage[preprint, nohyperref]{jmlr2e}

\usepackage{caption}
\usepackage[utf8]{inputenc}
\usepackage{url,amsmath}
\usepackage{algorithm, algorithmic}
\usepackage{color}
\usepackage{todonotes}
\usepackage{balance}
\usepackage[british]{babel}
\usepackage{csquotes}
\usepackage{pgfplots}
\usepackage[inline]{enumitem}
\usepackage{thmtools}
\usepackage{mathtools}
\usepackage{amssymb}
\usepackage{tikz}
\usepackage{xfrac}

\usepackage{microtype}
\usepackage{subcaption}

\usepackage[colorlinks=false,allbordercolors={1 1 1}]{hyperref}
\usepackage{cleveref}

\DeclareMathSymbol{\shortminus}{\mathbin}{AMSa}{"39}
\newcommand{\smn}{\shortminus}
\newcommand{\vgeq}{\rotatebox[origin=c]{-90}{$\geqslant$}}
\newcommand{\vgreat}{\rotatebox[origin=c]{-90}{$>$}}
\newcommand{\vequal}{\rotatebox[origin=c]{-90}{$=$}}

\usetikzlibrary{3d,patterns,arrows.meta,matrix,fit,backgrounds,graphs,graphs.standard,math,calc}
\pgfdeclarepatternformonly{south east lines}{
\pgfqpoint{-0pt}{-0pt}}{\pgfqpoint{6pt}{6pt}}{\pgfqpoint{6pt}{6pt}}{
        \pgfsetlinewidth{0.4pt}
        \pgfpathmoveto{\pgfqpoint{0pt}{6pt}}
        \pgfpathlineto{\pgfqpoint{6pt}{0pt}}
        \pgfpathmoveto{\pgfqpoint{.2pt}{-.2pt}}
        \pgfpathlineto{\pgfqpoint{-.2pt}{.2pt}}
        \pgfpathmoveto{\pgfqpoint{6.2pt}{5.8pt}}
        \pgfpathlineto{\pgfqpoint{5.8pt}{6.2pt}}
        \pgfusepath{stroke}}
        
\definecolor{edblue}{RGB}{101, 143, 157}
\definecolor{edred}{RGB}{146,59,60}

\newtheorem{notation}[theorem]{Notation}

\newtheorem{claim}[theorem]{Claim}

\jmlrheading{1}{2021}{1-48}{06/21}{00/00}{meila00a}{Edward Wagstaff, Fabian B. Fuchs, Martin Engelcke, Michael A. Osborne and Ingmar Posner}

\ShortHeadings{Universal Approximation of Functions on Sets}{Wagstaff, Fuchs, Engelcke, Osborne and Posner} 
\firstpageno{1}

\begin{document}

\title{Universal Approximation of Functions on Sets}

\author{\name Edward Wagstaff \email ed@robots.ox.ac.uk \\
        \name Fabian B. Fuchs \email fabian@robots.ox.ac.uk \\
        \name Martin Engelcke \email martin@robots.ox.ac.uk \\
        \name Michael A. Osborne \email mosb@robots.ox.ac.uk \\
        \name Ingmar Posner \email ingmar@robots.ox.ac.uk \\
        \addr Department of Engineering Science\\
        University of Oxford\\
        Oxford, UK}

\editor{Not yet known}

\maketitle

\begin{abstract}%   <- trailing '%' for backward compatibility of .sty file
% Summarize the whole paper in one sentence

Modelling functions of sets, or equivalently, \emph{permutation-invariant} functions, is a long-standing challenge in machine learning.
Deep Sets is a popular method which is known to be a universal approximator for continuous set functions.
We provide a theoretical analysis of Deep Sets which shows that this universal approximation property is only guaranteed if the model's latent space is sufficiently high-dimensional.
If the latent space is even one dimension lower than necessary, there exist piecewise-affine functions for which Deep Sets performs no better than a naïve constant baseline, as judged by worst-case error.
Deep Sets may be viewed as the most efficient incarnation of the Janossy pooling paradigm. 
We identify this paradigm as encompassing most currently popular set-learning methods.
Based on this connection, we discuss the implications of our results for set learning more broadly, and identify some open questions on the universality of Janossy pooling in general.

\end{abstract}


\input{01_introduction}
\input{02_related_work}
\input{03_deepsets_representation}

\input{04_function_approximation}
\input{05_other_methods}
\input{06_function_representation}
\input{07_conclusion}

\section*{Acknowledgements}

We thank André Henriques, Jan Steinebrunner and Tom Zeman for a helpful mathematical discussion, including proofs that the function $\Gamma_N$ defined in \Cref{eq:def_of_gamma} has a zero when $N=2$.
This research was funded by the EPSRC AIMS Centre for Doctoral Training at the University of Oxford.

\clearpage
\input{08_appendices}

\clearpage
\bibliography{references}

\end{document}

%% file: 01_introduction.tex
\section{Introduction}

Many applications of machine learning work with collections of inputs or features which are best modelled as \emph{sets}. 
Crucially, these collections have no intrinsic ordering, which distinguishes them from other commonly-encountered types of data such as images or audio. 
Examples are plentiful: a point cloud obtained from a LIDAR sensor; a group of atoms which together form a molecule; or a collection of objects appearing in an image. 
Although the data for these problems is not ordered, standard machine learning techniques usually impose an ordering on the data, either by representing it as a vector or matrix and applying order-sensitive operations, or by iterating over data points in a stateful fashion.
Any function with set-valued inputs is insensitive to this imposed ordering, in the sense that the output does not change when the input is reordered.
This property is known as \emph{permutation invariance}.

% Describe Deep Sets
Tasked with learning a permutation-invariant target function, we may build this invariance into our model -- we can regard this as giving the model an inductive bias.
A family of techniques have been proposed in recent years to build and train deep permutation-invariant machine learning models.
Most notably, the seminal work by \citet{Zaheer2017} introduces \emph{Deep Sets}.
The central idea of Deep Sets is to process the individual elements of a set in parallel using a shared encoding function before aggregating these encodings using a \emph{symmetric function} -- for example summation, averaging, or max-pooling.
We refer to the vector space in which this aggregation happens as the model's \emph{latent space}.
This design ensures that the model is exactly permutation-invariant.
A similar model is also explored in \citet{Qi2017}.

% Describe other methods
While easy to implement and parallelise, processing each input individually and pooling globally hinders relational reasoning \citep{Santoro2017,Battaglia2018}.
Arguably the most popular alternatives are self-attention mechanisms, which perform weighted summation to aggregate information via input-dependent attention weights.
Self-attention explicitly performs relational reasoning by processing elements in pairs, rather than individually -- working with relationships between pairs of elements is therefore an in-built feature of the model.
This stands in contrast to Deep Sets, where no facility for relational reasoning is built into the model (though such reasoning can still be learned).
While self-attention has famously been widely applied in natural language processing \citep{Vaswani2017}, it has also been applied to sets \citep{Lee2018} and is popular in the literature on graph learning \citep{velikovi2017graph}.

Although Deep Sets and self-attention may not immediately appear to be closely related, both can be viewed as special cases of $k$-ary \emph{Janossy pooling} \citep{Murphy2018}.
Deep Sets is the most restricted instance of Janossy pooling, which makes it a good starting point for fuller theoretical characterisation of these methods.
In this work, we contribute to this theoretical understanding by considering how the dimensionality of the Deep Sets model's latent space affects its expressive capacity.
This extends the partial characterisation provided in \citet{Zaheer2017}.

Firstly, we follow \citet{Zaheer2017} in considering the conditions under which Deep Sets is capable of \emph{exactly representing} any continuous target function (i.e. \emph{universal representation}). 
We show that the sufficient condition on the model's latent space dimension given in \citet{Zaheer2017} is essentially the best possible -- it can be lowered, but only by $1$, and this weakened sufficient condition is in fact also a \emph{necessary} condition.
Secondly, we consider universal \emph{approximation}, a weaker and more practically important property.
We demonstrate that this is subject to the same necessary and sufficient conditions as universal representation.
In fact, we show that if the latent dimension is even $1$ lower than necessary, there exist piecewise-affine functions for which the model's worst-case error is no better than the worst-case error of a model which simply outputs $0$.
Finally, we consider the implications of our findings for other related architectures such as self-attention, and discuss the universality properties of other set-learning methods.

% List of contributions
In summary, this work makes the following contributions:
\begin{itemize}
    \item \Cref{sec:related_work} provides an overview of the most popular deep learning architectures for learning functions on sets, including Deep Sets, and discusses their unification as instances of the \emph{$k$-ary Janossy pooling} paradigm;
    \item \Cref{sec:universal_representation} proves a necessary and sufficient condition for \emph{representing} arbitrary continuous functions on sets with the Deep Sets architecture;
    \item \Cref{sec:universal_approximation} furthers our analysis of the Deep Sets architecture by proving a necessary and sufficient condition for \emph{approximating} arbitrary continuous functions on sets;
    \item \Cref{sec:other_methods} discusses other models for deep learning on sets;
    \item \Cref{sec:universal_representation_other} discusses universal approximation criteria for architectures beyond Deep Sets, and identifies some open questions of interest.
\end{itemize}

\Cref{sec:universal_representation} of this paper is a revised and abridged version of work previously presented in \citet{wagstaff19}.
This paper provides a more extensive and comprehensive discussion of the subject and extends the work in \citet{wagstaff19} by providing additional results for architectures other than Deep Sets and a theoretical characterisation of universal function approximation.

%% file: 02_related_work.tex
\section{Models For Learning on Sets}
\label{sec:related_work}

In this section, we examine deep learning architectures for data that consists of unordered sets of elements.
We are centrally concerned with the property of permutation invariance as introduced above, but we briefly note that functions working with set-valued data may also be \emph{permutation-equivariant}. 
In the simplest case, equivariance appears in the context of mapping sets to sets. 
Although the ordering of elements is arbitrary, we may still want the ordering of the input set to be consistent with the ordering of the output set -- that is, any permutation of the input results in a corresponding permutation of the output.
This consistency of ordering is \emph{permutation equivariance}.
\citet{romero2021group} provide a full mathematical definition of permutation equivariance, along with a detailed analysis of the equivariance properties of self-attention.
We do not consider permutation-equivariant set functions in detail in this paper.

Returning to permutation \emph{invariance}, a core consideration for the design of permutation-invariant models is how to maintain maximum expressivity (in the sense of being able to model a broad class of functions) while also ensuring permutation invariance.
\citet{Murphy2018} introduce \emph{Janossy pooling} as a unifying framework of methods that learn either strictly permutation-invariant functions or suitable approximations.
Janossy pooling is known to be highly expressive, and in fact it is trivially shown to be universal in the sense that any permutation-invariant function may be represented within the Janossy pooling framework.

In its most general form, Janossy pooling considers all possible permutations $\pi$ of the input elements $x_i$. Each permutation $\pi(\mathbf{x})$ is separately passed through the same permutation-sensitive function $\phi$.
The outputs for the different permutations are then aggregated by computing the average (or by another global pooling operation).
If two inputs $\mathbf{x}$ and $\mathbf{y}$ are permutations of one another, this process will give the same output for both inputs.
As noted above, we refer to the space in which the aggregation happens as the model's \emph{latent space}.

Mathematically, the procedure outlined above can be written as
\begin{equation}
    \widehat{f}(\mathbf{x}) = \frac{1}{|S_M|} \sum_{\pi \in S_M} \phi \bigl( \pi \left( \mathbf{x} \right) \bigr)
    \label{eq:janossy}
\end{equation}
where $\mathbf{x}$ has $M$ elements, and $S_M$ is the group of all permutations $\pi$ of $M$ elements. A permutation-invariant function $\widehat{f}$ is thereby constructed from a permutation-sensitive function $\phi$. 
This permutation-sensitive function is typically implemented as a neural network, but other function approximators such as Gaussian processes may also be used.
The architecture is illustrated in \Cref{fig:permuting}.
This figure also illustrates that a second function may optionally be used to post-process the output of the aggregation operation:

\begin{equation}
    f(\mathbf{x}) = \rho \bigl( \widehat{f}\left( \mathbf{x} \right) \bigr)
    \label{eq:janossy_augmented}
\end{equation}

This second function $\rho$ does not need to follow any constraints to guarantee invariance because its input, $\widehat{f}(\mathbf{x})$, is already permutation-invariant. 
In other words, the ordering information is already lost by the point that $\rho$ is reached in \Cref{fig:permuting}.

\begin{figure}
\includegraphics[width=\textwidth]{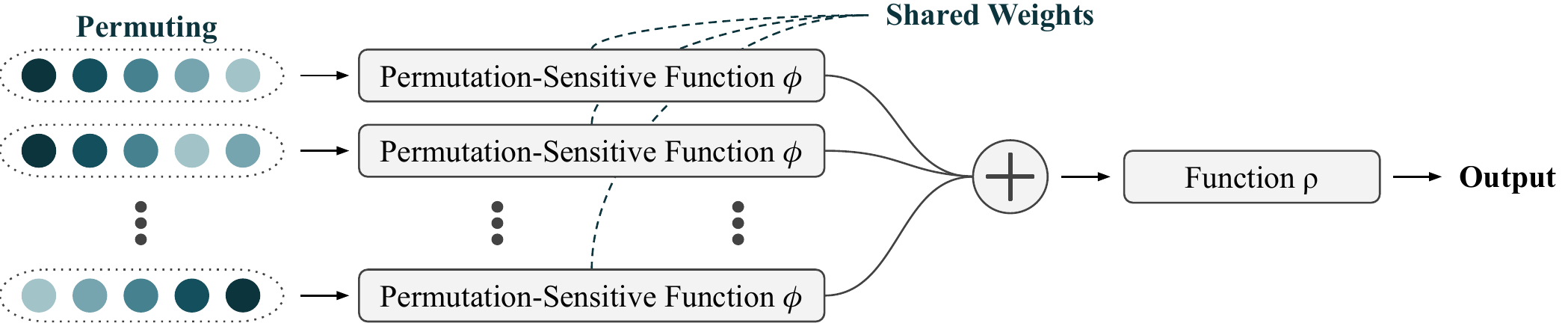}
\caption[Janossy pooling.]{The Janossy pooling paradigm applies the same permutation-sensitive network to each possible permutation of the input set. When using a permutation-invariant pooling operation, such as averaging or summation, information about the ordering is eliminated. A second neural network may be used to predict the final output. This paradigm guarantees permutation invariance of the output without imposing restrictions on any of the neural network components.}
\label{fig:permuting}
\end{figure}

One drawback of this scheme is that it can become prohibitively expensive for large set sizes $M$, because of the large number of terms in the sum in \Cref{eq:janossy}. The computational complexity scales at least linearly with the cardinality of $S_M$, which is $M!$.
To remedy this, \citet{Murphy2018} discuss several options for reducing the computational complexity: (i) sorting; (ii) sampling; (iii) restricting permutations to $k$-tuples.

Rather than considering all permutations, sorting only considers a single canonical permutation, which is obtained by sorting the input.
The sorting operation may be hand-specified, or it may be learned.
Sampling, by contrast, aggregates over a randomly-sampled subset of permutations (so that the sum in Equation \ref{eq:janossy} is over a randomly chosen subset of $S_M$).
While the output in this case is only approximately rather than strictly permutation invariant, \citet{Murphy2018} show that this works reasonably well empirically.
Nevertheless, most models that are used in practice \citep[e.g.][]{Zaheer2017,Qi2017a,Lee2018} fall into the third category of restricting permutation to $k$-tuples.
The theoretical analysis conducted in our work focuses on this third category, which we describe in more detail in the reaminder of this section.

\subsection{Limiting the Number of Elements in Permutations}

\Cref{eq:janossy} considers all possible $M$-tuples from an $M$ element set.
To save computation, one can instead consider all $k$-tuples\footnote{We require that the tuples we discuss here consist of distinct elements from our input set.
So in this case, $(2, 2)$ is not a valid 2-tuple from the set $\{1, 2, 3\}$. 
This restriction is not always observed, for instance self-attention \emph{does} allow such 2-tuples, but for brevity we restrict ourselves to discussing the case of distinct elements.} with $k < M$.
The models obtained by setting $k=1$ and $k=2$ are visualised in \Cref{fig:k12}.
Mathematically, letting $\mathbf{x}_{\{k\}}$ denote a $k$-tuple from $\mathbf{x}$, we have
\begin{align}
\label{eq:k_ary_janossy}
\widehat{f}(\mathbf{x}) = \frac{1}{P(M,k)}\sum_{\mathbf{x}_{\{k\}}} \phi(\mathbf{x}_{\{k\}}) \text{ , where } P(M,k) = \frac{M!}{(M-k)!}.
\end{align}
For clarity, we provide an example for $M=4$ and $k=2$. If the input set is $\{w, x, y, z\}$, the sum will be over all 2-tuples from the set, namely:
\begin{align*}
~(w, x),~(x, w),(w, y),&~(y, w),~(w, z),~(z, w), \\
(x, y),~(y, x),~(x, z),&~(z, x),~(y, z),~(z, y)
\end{align*}
For sufficiently small $k$, this gives a sum with far fewer than $M!$ terms, leading to a computationally tractable method in many practical cases. 
For fixed $k$ and varying $M$, the number of terms in the sum is $\mathcal{O}(M^k)$.
Setting $k=1$ therefore provides a model whose cost is linear in the size of the input set.
Increasing $k$ comes at the cost of additional computational complexity, but as we will discuss in the following subsection, also allows for more explicit relational reasoning.

\begin{figure}[t]
    % \vspace{-0.2cm}
	\centering
% 	\hspace{-1cm}
	\begin{subfigure}[h]{0.46\textwidth}
		\centering
		\includegraphics[width=0.99\textwidth]{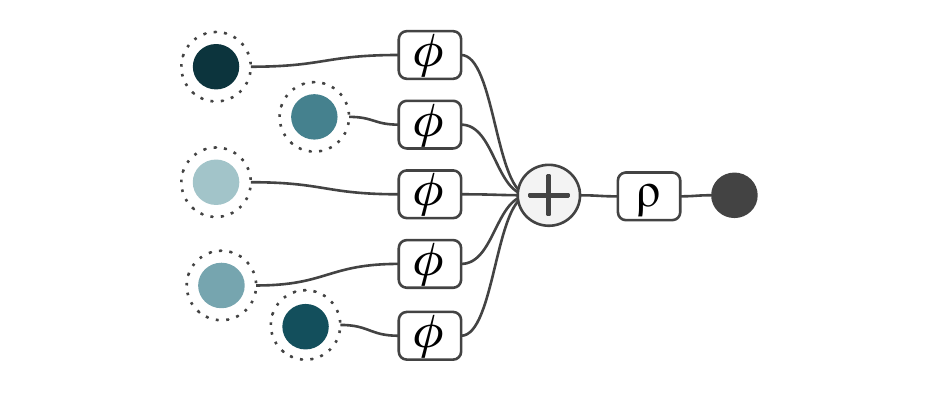}
		\caption{Janossy pooling with $k=1$ (\emph{Deep Sets})}
		\label{fig:Janossy_K1}
	\end{subfigure} 
	\begin{subfigure}[h]{0.46\textwidth}
		\centering
		\includegraphics[width=0.99\textwidth]{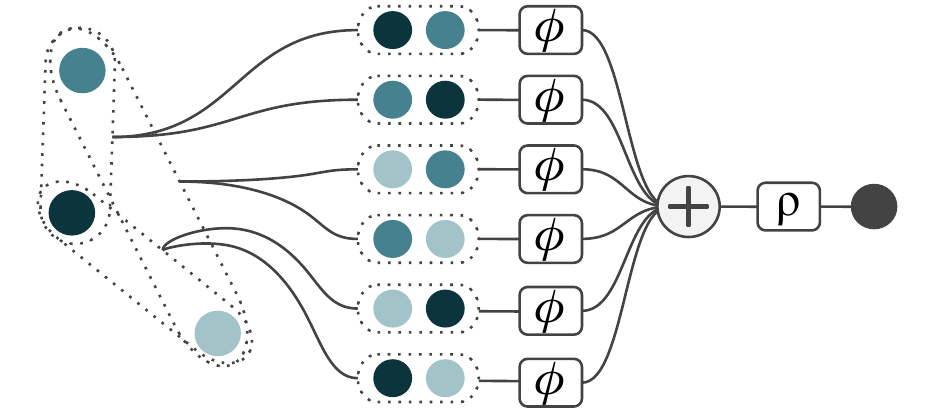}
		\caption{Janossy pooling with $k=2$}
		\label{fig:Janossy_K2}
	\end{subfigure} 
	\begin{subfigure}[h]{0.46\textwidth}
	    \vspace{+1cm}
		\centering
		\includegraphics[width=0.99\textwidth]{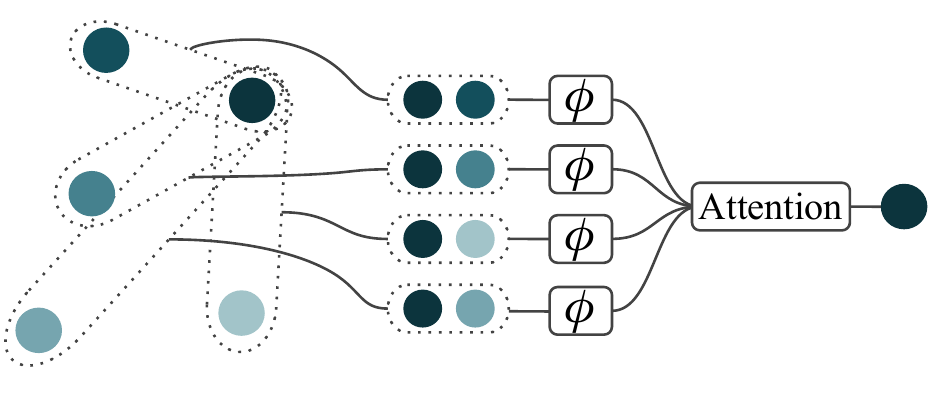}
		\vspace{-0.8cm}
		\caption{Self-attention}
		\label{fig:Janossy_K2Att}
	\end{subfigure} 
	\caption[Different versions and variations of Janossy pooling.]{Different versions and variations of Janossy pooling. Permutation invariance is guaranteed by processing all combinations of $k$ elements and then aggregating via a sum (or softmax in the case of attention). Self-attention, a variant of Janossy pooling with $k=2$, focuses on one node at a time (the darkest node here), computing an output for this specific node. It is often employed for all nodes in parallel in a permutation-equivariant manner, mapping sets of points to sets of points \citep{Lee2018}.}
	\label{fig:k12}
\end{figure}

\subsubsection{Expressivity and Interactions}
\label{sec:expresivity_and_interactions}

The terms \textit{relational reasoning} and \textit{interactions} are often encountered in the deep learning literature \citep{Battaglia2016,Santoro2017,fuchs2019endtoend}.
We use the expression \textit{interactions between elements in a set} to refer to the fact that the output may depend not only on the individual contribution of each element, but may also depend on the fact that multiple elements appear \textit{together in the same set}. 
Relational reasoning describes the act of modelling and using these interactions. 
To illustrate this with a simple example, consider the task of assessing how well a set of ingredients go together for cooking a meal.
If we set $k=1$, the function $\phi$ can take into account relevant individual attributes, but will be unable to spot any clashes between ingredients (like garlic and vanilla).
Increasing $k$ allows $\phi$ to see multiple elements at once, and therefore perform relational reasoning about pairs of ingredients, enabling a more expressive model of what tastes good.
If we view $\phi$ as an ``encoder'' and $\rho$ as a ``decoder'', $\phi$ is capable of encoding information about interactions, which $\rho$ can then make use of during decoding.

If we set $k=1$, this is no longer the case.
\citet{Zaheer2017} show that the encoder $\phi$ can encode the entire set when $k=1$, and therefore the decoder $\rho$ can in principle recover all the input information and perform relational reasoning from there.
However, any such relational reasoning is not built into the model, and must be entirely learned.
\citet{Murphy2018} speculate that this additional burden on what must be learned by $\rho$ can contribute to difficulties in training the $k=1$ model on some tasks which depend heavily on relational reasoning.

Many current neural network architectures on sets and graphs resemble Janossy pooling with $k=2$.
This choice of $k$ represents a tradeoff between computational complexity (lower $k$ is better) and ease of training for relational reasoning (higher $k$ is better\footnote{Though we note that high $k$ being ``better'' here is, at least theoretically speaking, a speculative statement. Nevertheless, the fact that $k=2$ is widely used in practice despite the increased computational complexity strongly suggests that increasing $k$ does indeed improve performance in some contexts.}).
Most famously, self-attention algorithms (schematically depicted in \Cref{fig:Janossy_K2Att}) compare two elements of the set at a time, typically by performing a scalar product \citep{Vaswani2017, Lee2018}.
The results of the scalar products are used as attention weights for aggregating information from different points via a weighted, permutation-invariant sum.
While this mechanism is very similar to Janossy pooling with $k=2$, some additional architecture choices are often made. For example, a softmax might be used to ensure that the attention weights are normalised.

\subsubsection{Deep Sets}

The model obtained by setting $k=1$ is in fact a popular and well-known special case of Janossy pooling -- Deep Sets \citep{Zaheer2017}.
As discussed above, choosing $k=1$ may hinder relational reasoning, but Deep Sets is nevertheless a widely-used architecture for permutation-invariant models, for example forming the basis of two well-known point cloud classification approaches, PointNet \citep{Qi2017} and PointNet++ \citep{Qi2017a}.
As the computationally cheapest instance of Janossy pooling, its linear scaling in the number of inputs makes it particularly well suited to problems with large input sets -- a point cloud for example may consist of thousands of points.
Importantly, it is known to be universal, guaranteeing (at least in principle) that it is capable of computing any permutation-invariant target function.
However, the conditions under which universality holds are not fully characterised, either for Janossy pooling more broadly or for Deep Sets specifically.
\citet{Zaheer2017} and \citet{han2019universal} provide sufficient conditions for Deep Sets to be universal, but do not show whether universality fails if these conditions are violated -- that is, they do not give \emph{necessary} conditions for universality, and leave open the possibility that universality is preserved under weaker conditions.
In the following two sections, we provide proofs of necessary conditions for the universality of Deep Sets.

%% file: 03_deepsets_representation.tex
\section{Universal Function Representation with Deep Sets}
\label{sec:universal_representation}

This first section of our theoretical investigation focuses on function \emph{representation}, building directly on the proofs and results from \citet{Zaheer2017}.
By function representation, we mean the ability of a model to \emph{exactly represent} a given target function.
By \emph{universal} representation, we mean the ability of a model to represent \emph{all} target functions from a given function class (e.g. the class of all permutation-invariant functions).
The original analysis from \citet{Zaheer2017} demonstrates the universality of Deep Sets by considering universal representation, and in this section we present a direct continuation and refinement of this analysis.
Nevertheless, it must be noted that function representation is a stronger property than required in practice, and we will return to analyse the weaker property of function \emph{approximation} in \Cref{sec:universal_approximation}.
Proofs of all novel results from this section are provided in \Cref{sec:proofs_appendix}.

\input{03a_preliminaries}

\input{03b_continuity}

\input{03c_practical_representation}

%% file: 03a_preliminaries.tex
\subsection{Preliminaries}
\label{sec:preliminaries}

We begin by introducing the necessary definitions and notation, before detailing our contribution and its relationship to the original results from \citet{Zaheer2017}.
Importantly, we will establish the concept of \emph{sum-decomposition}, which is a concise mathematical description of the Deep Sets architecture.

\subsubsection{Definitions and Notation}

\begin{notation}
We denote sets of inputs in boldface with subscript-indexed elements, e.g. the set $\mathbf{x}$ has elements $x_1, \ldots, x_M$.
\end{notation}

\begin{notation}
Throughout, the variable $M$ refers to the number of elements in the input sets under consideration.
\end{notation}

\begin{notation}
We denote the group of all permutations of $M$ elements by $S_M$. A permutation $\pi \in S_M$ may be thought of as a bijection from $\{ 1, \ldots, M \}$ to itself.
\end{notation}

\begin{definition}
A function $f(\mathbf{x})$ is \emph{permutation-invariant} if $f(x_1,\dots,x_M) = f \bigl(x_{\pi(1)},\dots,x_{\pi(M)} \bigr)$ for all $\pi \in S_M$.
\end{definition}

\begin{definition}
\label{def:sum-decomp}
We say that a function $f$ is \emph{sum-decomposable} if there are functions $\rho$ and $\phi$ such that

\begin{equation}
f(\mathbf{x}) = \rho \bigl( \sum_i \phi(x_i) \bigr).
\label{eq:main}
\end{equation}

In this case, we say that $(\rho, \phi)$ is a \emph{sum-decomposition} of $f$.
\end{definition}

Note that sum-decomposition is equivalent to Janossy pooling with $k=1$, i.e. Deep Sets. Comparing \Cref{eq:main} with Equations (\ref{eq:janossy_augmented}) and (\ref{eq:k_ary_janossy}), these expressions differ only in dividing $\phi$ by an integer, and can be made equal simply by rescaling $\phi$.

\begin{notation}
Given a sum-decomposition $(\rho, \phi)$, we write $\Phi(\mathbf{x}) := \sum_i \phi(x_i)$. With this notation, we can write \Cref{eq:main} as $f(\mathbf{x}) = \rho\big(\Phi(\mathbf{x})\big)$. We may also refer to the function $\rho \circ \Phi$ as a sum-decomposition.
\end{notation}

\begin{definition}
Let $(\rho, \phi)$ be a sum-decomposition. Write $Z$ for the domain of $\rho$ (which is also the codomain of $\phi$, and the space in which the summation happens in Equation \ref{eq:main}). We refer to $Z$ as the \emph{latent space} of the sum-decomposition $(\rho, \phi)$.
\end{definition}

\begin{definition}
Given a space $Z$, we say that $f$ is \emph{sum-decomposable via $Z$} if $f$ has a sum-decomposition whose latent space is $Z$.
\end{definition}

\begin{definition}
We say that $f$ is \emph{continuously sum-decomposable} when there exists a sum-decomposition $(\rho, \phi)$ of $f$ such that both $\rho$ and $\phi$ are continuous. $(\rho, \phi)$ is then a \emph{continuous sum-decomposition} of $f$.
\end{definition}

\begin{notation}
Denote the power set of a space $X$ (that is, the set of all subsets of $X$) by $2^X$.
\end{notation}

%%%%%%%%%%%%%%%%%%%%%%%%%%%%%%%%%%%%%%%%%%%%%%%%%%

\subsubsection{Our Contribution}

As noted above, the Deep Sets model uses sum-decomposition to represent permutation-invariant functions.
Our contribution is to show that the universality of sum-decomposition is dependent on the dimensionality of the latent space $Z$.
Specifically, we show that $\dim(Z)$ must be at least $M$, where $M$ is the number of elements in the input sets.
We refine the analysis from \citet{Zaheer2017}, relying on the observation that some of the mappings used for the proofs of universality in the original analysis are highly discontinuous, and cannot be computed in practice.
By considering continuous mappings,\footnote{This function class includes, for instance, all functions which can be computed by neural networks with continuous activation functions, or by Gaussian processes with continuous kernels.} we provide a proof of the above lower bound on $\dim(Z)$.
Because this section directly follows on from the analysis given in \citet{Zaheer2017}, we first give a brief reproduction of the statement and proof of two key theorems from that work.

%%%%%%%%%%%%%%%%%%%%%%%%%%%%%%%%%%%%%%%%%%%%%%%%%%

\subsubsection{Background Theorems}
\label{sec:original_theorems}

\citet{Zaheer2017} consider two cases. First, where $\mathbf{x}$ is a subset of, or drawn from, a \emph{countably infinite} universe $\mathfrak{U}$. Second, the case where $\mathfrak{U}$ is \emph{uncountably infinite}.

\begin{theorem}[Countable case]
\label{ori_countable_theorem}
Let $f: 2^{\mathfrak{U}} \to \mathbb{R}$ where $\mathfrak{U}$ is countable.
Then $f$ is sum-decomposable via $\mathbb{R}$.
\end{theorem}

\begin{proof}
Since $\mathfrak{U}$ is countable, each $x \in \mathfrak{U}$ can be mapped to a unique element in $\mathbb{N}$ by a bijective function $c(x): \mathfrak{U} \to \mathbb{N}$.
If we can choose $\phi$ so that $\Phi$ is invertible, then we can set $\rho = f \circ \Phi^{-1}$, giving

\vspace{-5mm}
\begin{gather*}
f = \rho \circ \Phi
\end{gather*}

i.e. f is sum-decomposable via $\mathbb{R}$.

Now consider $\phi(x) = 4^{-c(x)}$. 
Under this mapping, each set $\mathbf{x} \subset \mathfrak{U}$ corresponds to a unique real number.
The real number $r := \Phi(\mathbf{x})$ can be decoded to the set $\mathbf{x}$ by looking at the base 4 expansion of $r$.
The element $c^{-1}(n) \in \mathfrak{U}$ belongs to $\mathbf{x}$ if and only if the $n$-th digit of $r$ is $1$.
This decoding procedure shows that $\Phi$ is invertible, and the conclusion follows.
\end{proof}

For the uncountable case, we consider $M$-element sets from the universe $\mathfrak{U}=[0,1]$.

\begin{theorem}[Uncountable case]
\label{ori_uncountable_theorem}
Let $M \in \mathbb{N}$, and let $f: [0,1]^M \to \mathbb{R}$ be a continuous permutation-invariant function.
Then $f$ is continuously sum-decomposable via $\mathbb{R}^{M+1}$.
\end{theorem}

Another way of stating \Cref{ori_uncountable_theorem}, which brings the terminology more in line with Deep Sets, is as follows:

\edef\mythmcount{\value{theorem}}
\setcounterref{theorem}{ori_uncountable_theorem}
\addtocounter{theorem}{-1}
\begin{theorem}[Deep Sets terminology]
Deep Sets can represent any continuous permutation-invariant function of $M$ elements if the dimension of the model's latent space is at least $M+1$.
\end{theorem}
\setcounter{theorem}{\mythmcount}

\begin{proof}
The proof by \citet{Zaheer2017} of \Cref{ori_uncountable_theorem} is more involved than for \Cref{ori_countable_theorem}.
We do not include it here in detail, but we summarise the main points as follows.
\begin{enumerate}
    \item Show that the mapping $\Phi: [0,1]^M \to \mathbb{R}^{M+1}$ defined by $\Phi_q(\mathbf{x}) = \sum_{i=1}^{M} (x_i)^q$ for $q = 0,\dots,M$ is injective and continuous\footnote{In the original proof, $\Phi$ is denoted $E$.}.
    \item Show that $\Phi$ has a continuous inverse.
    \item Define $\rho: \mathbb{R}^{M+1} \to \mathbb{R}$ by $\rho = f \circ \Phi^{-1}$.
    \item Define $\phi(x): \mathbb{R} \to \mathbb{R}^{M+1}$ by $\phi_q(x) = x^q$.
    \item Note that, by definition of $\rho$ and $\phi$, $(\rho, \phi)$ is a continuous sum-decomposition of $f$ via $\mathbb{R}^{M+1}$.
\end{enumerate}
\end{proof}

\pagebreak

In light of these theorems, our key conclusions can be summarised as follows:

\begin{itemize}
    \item In \Cref{sec:continuity} we argue that the guarantee of sum-decomposability via $\mathbb{R}$ given by \Cref{ori_countable_theorem} cannot hold in practice. This is because the necessary mappings are highly discontinuous and cannot be computed.
    \item In \Cref{sec:continuous} we prove that the guarantee of sum-decomposability via $\mathbb{R}^{M+1}$ given by \Cref{ori_uncountable_theorem} is essentially the best possible. The dimension of the latent space must be at least $M$, or universality is lost.
\end{itemize}

Additionally, we contribute some results on universal representation with discontinuous sum-decompositions.
These results are not relevant for practice\footnote{One of our proofs depends on the axiom of choice, bringing us firmly beyond the realm of computability!} but are included for mathematical interest, and can be found in Appendix \ref{sec:proofs_appendix}.

%% file: 03b_continuity.tex
\subsection{The Importance of Continuity}
\label{sec:continuity}

\Cref{ori_countable_theorem} appears to give a strong guarantee on the universality of sum-decomposition via $\mathbb{R}$.
The theorem applies even when the domain is allowed to be infinite, and so any task implemented on real hardware, where inputs belong to large but finite domains, satisfies the conditions of the theorem.
Nevertheless, an insurmountable problem arises when trying to apply the construction from the proof of \Cref{ori_countable_theorem}, with the consequence that \Cref{ori_countable_theorem} cannot be applied on real hardware. 
That is, sum-decomposition via $\mathbb{R}$ does \emph{not} provide a guarantee of universal function representation in practice.
We illustrate the problem by considering an example with the domain of all $256 \times 256$ pixel images in 8-bit greyscale, which we denote $\mathcal{I}$.

Suppose, as required by the construction used in the proof, that we have a counting function $c$ which assigns an integer to each element of $\mathcal{I}$.
There are $2^{2^{19}}$ elements in $\mathcal{I}$, so even if $c$ counts all ``interesting'' images\footnote{That is, images which are not just random noise, but have meaningful visual content.} first, there must be ``interesting'' images mapping to very large numbers under $c$.
Consider for instance the $2^{2^{16}}$th image $\iota := c^{-1}(2^{2^{16}})$, and suppose that, given a set of images $\mathbf{x}$, our target function $f$ needs to compute whether or not $\iota \in \mathbf{x}$.
Under the encoding/decoding procedure from the proof of \Cref{ori_countable_theorem}, the decoder function $\rho$ is given as input the real number $r := \Phi(\mathbf{x})$.
Recalling that the construction has $\phi(x) := 4^{-c(x)}$, the presence or absence of the image $\iota$ in the set $\mathbf{x}$ will affect the value of $r$ by one part in $4^{{2^{2^{16}}}}$.
Any decoder $\rho$ implemented on real hardware will therefore be unable to distinguish between sets containing $\iota$ and sets without $\iota$.

The essential problem here is an extreme degree of \emph{discontinuity}.
Intuitively speaking, a function is continuous if, at every point in the domain, the variation of the output is insensitive to small variations in the input.
In the example above, a successful decoder $\rho$ must be sensitive to vanishingly small perturbations in the input, across its entire input space, and this property renders it impossible to compute in practice.
% TODO ED I think we can prove a theorem here that says rho is highly discontinuous for any sum-decomposition of max on Q^2.
We can avoid this problem by requiring that the functions under consideration are all continuous.
This requirement is ubiquitous in work on the function approximation properties of machine learning models, including the universal approximation theorem for neural networks \citep{cybenko1989approximation}, the universality of Gaussian processes with certain kernels \citep{Rasmussen2006GPML}, the literature on approximating functions on sets \citep{han2019universal, segol2020universal}, and indeed in function approximation results outside machine learning \citep[such as the well-known Stone-Weierstrass theorem,][]{stone1948generalized}.
In practice, neural networks with continuous activation functions\footnote{I.e. essentially all popular activation functions, including ReLU, ELU, sigmoid, and tanh.} can only compute continuous functions, and the same is true of Gaussian processes with continuous kernels.

To continue our analysis of sum-decomposition, we therefore adopt the assumption that the target functions and sum-decompositions under consideration are all continuous.
This aligns with the conditions in \Cref{ori_uncountable_theorem}, and in the next section we show that the conclusion of \Cref{ori_uncountable_theorem} is essentially the strongest result that can be achieved under these assumptions.

%% file: 03c_practical_representation.tex
\subsection{Function Representation With Continuous Mappings}
\label{sec:continuous}

The following theorem is our central result on the function representation properties of continuous sum-decomposition.

\begin{restatable}{theorem}{maxnotdecomp}
\label{thm:max_not_decomposable}
Let $M, N \in \mathbb{N}$, with $M > N$. Then there exist permutation-invariant continuous functions $f : \mathbb{R}^M \to \mathbb{R}$ which are {\upshape\bfseries not} continuously sum-decomposable via $\mathbb{R}^N$.
\end{restatable}

Restated in more practical terms, this implies that for Deep Sets to be capable of representing \emph{arbitrary} continuous functions on sets of size $M$, the dimension of the model's latent space (denoted $N$ in the theorem statement above) must be at least $M$.
A similar statement is also true for models based on the similar concept of \emph{max-decomposition}, as considered by \citet{Qi2017} -- details are provided in \Cref{app:max-decomp}.

In terms of concrete recommendations for practitioners, \Cref{thm:max_not_decomposable} suggests that when deploying a model based on Deep Sets, the choice of latent space should depend on the number of elements in the input sets under consideration.
Referring back to our depiction of Deep Sets in \Cref{fig:Janossy_K1}, this means that the cardinality of the input sets should be taken into account when choosing the length of the vectors output by the $\phi$ blocks.
Choosing a latent space which is not sufficiently high-dimensional may mean that the target function cannot be successfully modelled.
It must be noted, however, that \Cref{thm:max_not_decomposable} places restrictions on function \emph{representation}, but not on function \emph{approximation}.
Function approximation comes closer to describing what is needed in practice, and we will return to address this point in \Cref{sec:universal_approximation}, where we prove that this restriction on the latent space dimension also holds for function \emph{approximation} with Deep Sets.
We discuss the practical implications of \Cref{thm:max_not_decomposable} in more detail in \Cref{sec:implications_and_limitations}.

In addition to showing that an $M$-dimensional latent space is necessary, we adapt the proof from \citet{Zaheer2017} of \Cref{ori_uncountable_theorem} to strengthen the result in two ways. 
Proofs of these two results can be found in \Cref{sec:proofs_appendix}.
Firstly, we lower the bound on the sufficient latent space dimension by $1$, concluding that an $M$-dimensional latent space is \emph{necessary and sufficient} for universal representation.

\begin{restatable}{theorem}{oriunc}
\label{cor:ori_uncountable_theorem}
Let $M \in \mathbb{N}$, and let $f: \mathbb{R}^{M} \to \mathbb{R}$ be a continuous permutation-invariant function. Then $f$ is continuously sum-decomposable via $\mathbb{R}^M$.
\end{restatable}

Secondly, we show that the model can deal with variable set sizes $\leq M$. 
That is, the model is not only capable of modelling functions on sets of size \emph{exactly} $M$, but also on sets of size \emph{up to} $M$.
At least in principle, the model is therefore applicable in settings where the input sets may contain variable numbers of points.

\begin{restatable}[Variable set size]{theorem}{arbitrary}
\label{thm:arbitrary_set_sizes}
Denote the set of subsets of $[0,1]$ containing at most $M$ elements by $[0,1]^{\leq M}$. Let $f: [0,1]]^{\leq M} \to \mathbb{R}$ be continuous\footnote{Note that we must take some care with the notion of continuity here -- see \Cref{sec:cont_set_fun_remark}.} and permutation-invariant. Then $f$ is continuously sum-decomposable via $\mathbb{R}^M$.
\end{restatable}

\subsubsection{Proof of \Cref{thm:max_not_decomposable}}

Our proof of \Cref{thm:max_not_decomposable} relies on showing that the permutation-invariant function \texttt{max}, applied to $M$-element sets of real numbers, is not sum-decomposable via $\mathbb{R}^{M-1}$.
To show this, we need the following lemma.

\begin{lemma}
\label{lem:injection}
Let $M, N \in \mathbb{N}$, and suppose $\phi : \mathbb{R} \to \mathbb{R}^N$, $\rho : \mathbb{R}^N \to \mathbb{R}$ are functions such that

\begin{equation}
\label{eq:max_rep}
\texttt{\emph{max}}(\mathbf{x}) = \rho \bigl( \sum_i \phi(x_i) \bigr).
\end{equation}

Recall that $\Phi(\mathbf{x}) = \sum_i \phi(x_i)$, and write $\Phi_M$ for the restriction of $\Phi$ to sets of size $M$.

Then $\Phi_M$ is injective for all $M$.
\end{lemma}

\begin{proof}
We proceed by induction. The base case $M=1$ is clear.

Now let $M \in \mathbb{N}$, and suppose that $\Phi_{M-1}$ is injective. Suppose there are sets $\mathbf{x}, \mathbf{y}$ such that $\Phi_M(\mathbf{x}) = \Phi_M(\mathbf{y})$. First note that, by \Cref{eq:max_rep}, we must have

\begin{equation}
\label{eq:equal_max}
\texttt{max}(\mathbf{x}) = \texttt{max}(\mathbf{y}).
\end{equation}

So now write

\begin{equation}
\label{eq:decomp}
\mathbf{x} = \{x_\texttt{max}\} \cup \mathbf{x}_\text{rem} ~ ; ~ \mathbf{y} = \{y_\texttt{max}\} \cup \mathbf{y}_\text{rem},
\end{equation}

where $x_\texttt{max} = \texttt{max}(\mathbf{x})$, and $y_\texttt{max} = \texttt{max}(\mathbf{y})$. But now we have

\begin{displaymath}
\begin{split}
\Phi_M(\mathbf{x}) & = \Phi_{M-1}(\mathbf{x}_\text{rem}) + \phi(x_\texttt{max}) \\
          & = \Phi_{M-1}(\mathbf{y}_\text{rem}) + \phi(y_\texttt{max}) \\
          & = \Phi_M(\mathbf{y}).
\end{split}
\end{displaymath}

From the central equality, and \Cref{eq:equal_max}, we have

\begin{displaymath}
\Phi_{M-1}(\mathbf{x}_\text{rem}) = \Phi_{M-1}(\mathbf{y}_\text{rem}).
\end{displaymath}

Now by injectivity of $\Phi_{M-1}$, we have $\mathbf{x}_\text{rem} = \mathbf{y}_\text{rem}$. Combining this with Equations \eqref{eq:equal_max} and \eqref{eq:decomp}, we must have $\mathbf{x} = \mathbf{y}$, and so $\Phi_M$ is injective.
\end{proof}

Equipped with this lemma, we can now prove \Cref{thm:max_not_decomposable}.

\begin{proof}
We proceed by contradiction. Suppose that functions $\phi$ and $\rho$ exist satisfying \Cref{eq:max_rep}. Define $\Phi_M : \mathbb{R}^M \to \mathbb{R}^N$ by

\begin{displaymath}
\Phi_M(\mathbf{x}) := \Sigma_{i=1}^M \phi(x_i).
\end{displaymath}

Denote by $\mathbb{R}_\text{ord}^M$ the set of all $\mathbf{x} \in \mathbb{R}^M$ such that $x_1 < x_2 < ... < x_M$. Let $\Phi_M^\text{ord}$ be the restriction of $\Phi_M$ to $\mathbb{R}_\text{ord}^M$. Since $\Phi_M^\text{ord}$ is a sum of continuous functions, it is also continuous, and by Lemma \ref{lem:injection}, it is injective.

Now note that $\mathbb{R}_\text{ord}^M$ is a convex open subset of $\mathbb{R}^M$, and is therefore homeomorphic to $\mathbb{R}^M$. Therefore, our continuous injective $\Phi_M^\text{ord}$ can be used to construct a continuous injection from $\mathbb{R}^M$ to $\mathbb{R}^N$. But it is known that no such continuous injection exists when $M>N$. Therefore our decomposition from \Cref{eq:max_rep} cannot exist.
\end{proof}

\subsubsection{Implications and Limitations}
\label{sec:implications_and_limitations}

In light of Theorems \ref{thm:max_not_decomposable} and \ref{cor:ori_uncountable_theorem}, we now have a necessary and sufficient condition for universal function representation with Deep Sets.
Given a continuous target function $f$ on sets of size at most $M$, $f$ can be represented using a Deep Sets-based model if and only if the model's latent space is at least $M$-dimensional.
We now provide a brief discussion of the implications of this result and the limitations which must be taken into account when considering how the result applies in practice.

We can concisely summarise the practical implications of Theorems \ref{thm:max_not_decomposable} and \ref{cor:ori_uncountable_theorem} as follows. 
When deploying a Deep Sets-based model, the choice of latent space should depend on the cardinality of the sets being processed. 
\textbf{Larger input sets demand a larger latent space}.
We emphasise that, although our results give a precise figure for the necessary and sufficient latent dimension $M$, this should not be interpreted as meaning that exactly $M$ dimensions should always be used in practice, for the following reasons.

First note that \Cref{thm:max_not_decomposable} does not imply that \emph{all} functions require an $M$-dimensional latent space.
Some functions can be represented in a lower dimensional space -- take for instance the sum of all elements, which is trivially represented by setting both $\phi$ and $\rho$ to be the identity function.
The statement rather says that \emph{some} functions require an $M$-dimensional latent space.
It is still possible that, in a given application, the target function does not require $M$ dimensions.
While we do not characterise exactly which functions require an $M$-dimensional sum-decomposition, we do note that such functions need not be ``badly-behaved'' or difficult to specify.
Our proof specifically demonstrates that even $\text{\tt{max}}$, which is trivial to specify, is not continuously sum-decomposable with a latent space dimension less than $M$.

Second, although \Cref{cor:ori_uncountable_theorem} shows that an $M$-dimensional latent space suffices to model any target function, we also know from \Cref{thm:max_not_decomposable} that $M$ dimensions is the \emph{bare minimum} needed to guarantee this property.
In practice, a model using only this minimum capacity of $M$ dimensions is not guaranteed to provide good results.
One reason for this is that, while the necessary functions $\phi$ and $\rho$ certainly \emph{exist}, we have no guarantee that they can be successfully \emph{learned}.
It is entirely possible, even in light of our results, that increasing the latent dimension above $M$ leads to more efficient training and superior results in practice.

A final, major weakness of \Cref{thm:max_not_decomposable} is that, in following the mathematical framework developed in \citet{Zaheer2017}, it addresses function \emph{representation} rather than function \emph{approximation}.
This leaves open the possibility that, while an $M$-dimensional latent space is needed for exact representation, we can achieve arbitrarily good approximation of any target function using only a $1$-dimensional latent space.
Addressing this weakness requires significantly more complicated mathematics, and we devote the next section to proving that an analogue of \Cref{thm:max_not_decomposable} does indeed hold for universal approximation, and $M$ dimensions are still required.

%% file: 04_function_approximation.tex
\section{Universal Function Approximation with Deep Sets}
\label{sec:universal_approximation}

\input{04a_discussion}

\input{04b_proof}

\input{04c_additional}

%% file: 04a_discussion.tex
In the previous section, we considered the problem of function \emph{representation} -- showing that there exists an instance of a given model which exactly computes a target function. But exactly representing target functions is generally \emph{not} our goal -- typically, we are concerned with \emph{approximating} a target function. In this section, we extend our analysis of the Deep Sets architecture to assess its ability to perform approximation.

To see concretely why function approximation requires a separate analysis, we return to the \texttt{max} function on a set of $M$ elements. As proved above, for an exact sum-decomposition, the latent space must be at least $M$-dimensional. But as noted by \cite{Zaheer2017} in an appendix to their work, \texttt{max} can be approximated arbitrarily well with only a $2$-dimensional latent space. In fact, one dimension will suffice. Let $a \in \mathbb{R}$ and set $\phi$ and $\rho$ as follows:
\begin{align}
\label{eq:phi_rho_max_approx}
\phi(x) &= e^{ax} & \rho(x) &= \frac{\log x}{a}.
\end{align}
Thus, our sum-decomposition approximation for \texttt{max} is
\begin{equation}
\label{eq:max_1d_approx}
\widehat{f}_a(X) = \frac{1}{a} \text{log} \left( \sum_{x \in X} e^{ax} \right).
\end{equation}
For input sets of size $M$, it is easily shown that
\begin{equation*}
\texttt{max}(X) \leq \widehat{f}_a(X) \leq \texttt{max}(X) + \frac{\text{log} M}{a}.
\end{equation*}
That is, the worst-case error in this approximation is $\frac{\text{log} M}{a}$. For fixed set size $M$, we can therefore achieve arbitrarily good approximation by increasing the value of $a$, even though we have only used a $1$-dimensional latent space. 

In the specific case of \texttt{max}, this demonstrates that there is a wide gap between what is necessary for representation and what is necessary for approximation. \texttt{max} is in some sense ``as hard as possible'' to represent, requiring $M$ latent dimensions, but ``as easy as possible'' to approximate, requiring only $1$ latent dimension.
This naturally raises the question of whether the approximation problem is somehow fundamentally easier. We state a more precise form of this question as follows: \emph{is it the case that every permutation-invariant function of $M$ elements has an approximate sum-decomposition via fewer than $M$ dimensions}? Our analysis answers this question in the negative -- there exist functions which cannot be closely approximated with a lower-dimensional latent space. Moreover, low-dimensional sum-decomposition is guaranteed to fail very badly for these functions -- in terms of worst-case error, approximation with sum-decomposition performs as badly as approximation with a constant function.

\subsection{A Necessary Condition For Function Approximation by Sum-Decomposition}

The question above mentions \emph{approximate sum-decomposition} -- in order to state and prove our result, we must define precisely what we mean by this. In the following definitions, let $M$ be a positive integer, $U \subset \mathbb{R}^M$ be compact,\footnote{This compactness requirement is necessary when discussing universal function approximation. Of particular relevance here is the fact that the universal approximation theorem for neural networks also requires a compact domain \citep{cybenko1989approximation, funahashi1989approximate, hornik1989multilayer}. This compactness requirement is a fundamental constraint which is also necessary for other function approximation results, for example the Stone-Weierstrass theorem \citep{stone1948generalized}.} and $f : U \to \mathbb{R}$. We maintain the convention that $\Phi(\mathbf{x}) = \sum_{i=1}^M \phi(x_i)$.

\begin{definition}
Let $\epsilon > 0$. $(\phi, \rho)$ is a \emph{within-$\epsilon$ sum-decomposition} of $f$ if $|f(\mathbf{u}) - \rho(\Phi(\mathbf{u}))| < \epsilon$ for every $\mathbf{u} \in U$.
\end{definition}

For example, if we let $a = \frac{\text{log}M}{\epsilon}$, then the sum-decomposition defined by \Cref{eq:phi_rho_max_approx} is a within-$\epsilon$ continuous sum-decomposition of \texttt{max} via $\mathbb{R}$.

\begin{definition}
A sequence $(\phi, \rho)_k = \{(\phi_k, \rho_k) ; k \in \mathbb{N} \}$ is an \emph{approximate sum-decomposition} of $f$ if, for any $\epsilon > 0$, there is some $K \in \mathbb{N}$ such that $(\phi_K, \rho_K)$ is a within-$\epsilon$ sum-decomposition of $f$. We also require that $(\phi_k, \rho_k)$ is a within-$\epsilon$ sum-decomposition of $f$ for every $k \geq K$. Put more loosely, $(\phi_k, \rho_k)$ is a sequence of ever-closer approximations to $f$. The existence of an approximate sum-decomposition of $f$ guarantees that $f$ can be approximated arbitrarily closely by sum-decomposition.

Given a set $Y$, we say that the approximate sum-decomposition $(\phi, \rho)_k$ is \emph{via $Y$} if $(\phi_k, \rho_k)$ is via $Y$ for every $k$.
\end{definition}

For example, letting $a=1,2,\ldots$, the sequence $(\phi, \rho)_a$ as defined by \Cref{eq:phi_rho_max_approx} is a continuous approximate sum-decomposition of \texttt{max} via $\mathbb{R}$.

With these definitions in hand, we now state our main result on approximation.

\begin{theorem}
\label{thm:main_approximation_theorem}
Let $M,N \in \mathbb{N}$ with $M>N$, and let $I_M := [-1, 1]^M \subset \mathbb{R}^M$. Then there exists a continuous permutation-invariant function $f: I_M \to \mathbb{R}$ which has no continuous approximate sum-decomposition via $\mathbb{R}^N$.
\end{theorem}

Our proof of this theorem actually provides stronger conclusions than stated above. The following paragraphs explain these additional conclusions in detail, and a theorem statement incorporating these conclusions is included in \Cref{sec:approximation_appendix}.

Since the target function $f$ is continuous on a compact domain, we know that its output is bounded, say by $y_\text{min} \leq f(\mathbf{u}) \leq y_\text{max}$. A worst-case error\footnote{By worst-case error we mean $\text{max}_\mathbf{u} |f(\mathbf{u}) - \widehat{f}(\mathbf{u})|$.} of $E := \frac{y_\text{max} - y_\text{min}}{2}$ can trivially be achieved by approximating $f$ with the constant function $\widehat{f}(\mathbf{u}) = \frac{y_\text{max} + y_\text{min}}{2}$. We show that there exist permutation-invariant functions $f_*$ on $\mathbb{R}^M$ which have no within-$E$ continuous sum-decomposition via $\mathbb{R}^N$. In other words, any attempt to sum-decompose $f_*$ via $\mathbb{R}^N$ is doomed to have the same worst-case error as simply approximating $f_*$ by a constant value.

In addition, we explicitly construct a permutation-invariant function $f_*$ for which sum-decomposition via $\mathbb{R}^N$ must have worst-case error $E$ as defined above. Letting $U=[\smn1,1]^M$, this $f_*$ has a simple form:\footnote{The extra $\smn1$ for even $M$ is to ensure that $f_*$ is bounded between $\smn1$ and $1$. This isn't necessary, indeed the bias term could be removed from \Cref{eq:defn_bad_target} without affecting our result, but reasoning about $f_*$ is simpler if its range is fixed.}

\begin{equation}
\label{eq:defn_bad_target}
\begin{aligned}
f_*(\mathbf{u}) &:= \mathbf{w}^\intercal\texttt{sort}(\mathbf{u}) + b \\
w_i &:= (\smn1)^{i+1} \\
b &:= \begin{cases}
\smn1 & M~\text{even} \\
0 & M~\text{odd}
\end{cases}
\end{aligned}
\end{equation}

Here, \texttt{sort} sorts the elements of $\mathbf{u}$ in descending order. This definition may look a little arbitrary, but we will justify it in \Cref{sec:approximation_proof} and provide some intuition for why it is hard to approximate in \Cref{sec:contours}. The function $f_*$ is trivial to represent using the sorting framework mentioned in \Cref{sec:related_work},\footnote{We discuss sorting in more detail in \Cref{sec:sorting}.} but impossible to approximate using a continuous sum-decomposition via $\mathbb{R}^N$. This gives mathematical support to the proposition that there is no true ``one size fits all'' approach -- for a given task, performance may vary greatly between different methods of achieving permutation invariance.

Finally we note that poor approximation behaviour is not confined to a single point. This follows from the fact that both $f_*$ and the sum-decomposition are continuous, and so the approximation error also varies continuously -- if the error is high at one point, it must also be high in some region around that point.

Having stated the key conclusions, we now present the proof of \Cref{thm:main_approximation_theorem}. We omit some details, particularly in the proof of Lemma \ref{lem:nu_n}. Full details are included in \Cref{sec:approximation_appendix}.

%% file: 04b_proof.tex
\subsection{Proving Necessity for Universal Approximation}
\label{sec:approximation_proof}

\begin{figure}
     \centering
     \begin{subfigure}{0.45\textwidth}
         \centering
         \input{figures/delta_3_boundaries.pgf}
         \caption{The boundary of $\Delta_3$.}
         \label{fig:delta_3}
     \end{subfigure}
     \hfill
     \begin{subfigure}{0.45\textwidth}
         \centering
         \input{figures/f_star.pgf}
         \caption{The image of $\Delta_3$ under $f_*$.}
         \label{fig:f_star_delta_3}
     \end{subfigure}
     \hfill
     \begin{subfigure}{0.45\textwidth}
         \centering
         \input{figures/phi.pgf}
         \caption{The image of $[\smn1,1]$ under $\phi$.}
         \label{fig:example_of_2d_phi}
     \end{subfigure}
     \hfill
     \begin{subfigure}{0.45\textwidth}
         \centering
         \input{figures/delta_3_under_Phi.pgf}
         \caption{The image of $\Delta_3$ under $\Phi$.}
         \label{fig:image_of_delta_3}
     \end{subfigure}
        \caption[A visualisation of sum-decomposition when the input sets are of size $3$ and the latent space is $2$-dimensional.]{A visualisation of sum-decomposition when the input sets are of size $3$ and the latent space is $2$-dimensional. For the purposes of this illustration, we pick $\phi$ to send the interval $[\smn1,1]$ to a semicircle. With this choice of $\phi$, we see that two disjoint regions of the surface of $\Delta_3$, highlighted with solid red lines, intersect under the application of $\Phi$. By contrast, $f_*$ separates these highlighted regions, sending them to opposite ends of the interval $[\smn1,1]$. Since these regions are already ``mixed together'' by $\Phi$, there is no function $\rho$ which can separate them out again to approximate $f_*$, so sum-decomposition of $f_*$ fails. Our proof demonstrates that an intersection under $\Phi$ which causes this problem for $f_*$ is forced for any continuous $\phi$ and in any number of dimensions, as long as the dimension of the latent space is less than the number of inputs.
        }
        \label{fig:topology_cartoon}
\end{figure}
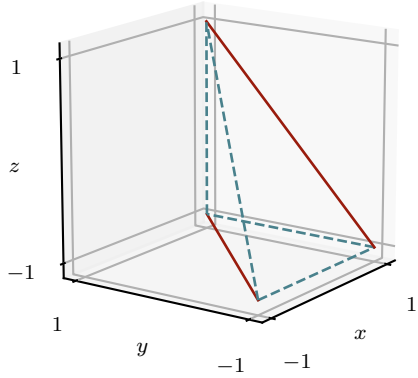
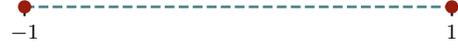
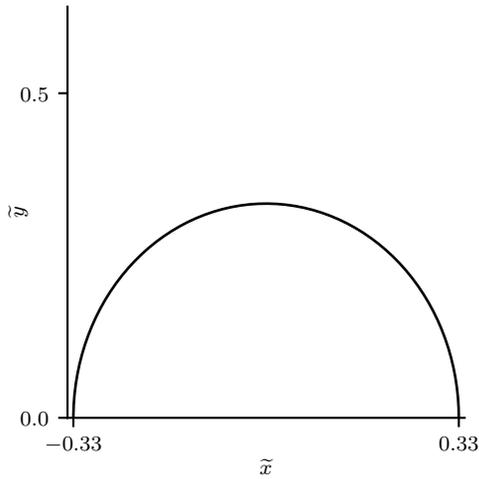
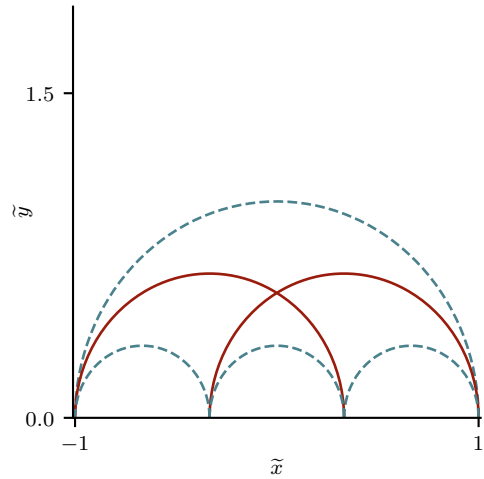

The overall proof strategy is best understood by thinking topologically -- that is, by considering how the encoding map $\Phi$ and the target function $f$ deform the input space, squashing it onto lower-dimensional output spaces. To understand this perspective we will consider an example with $M=3$ and $N=2$, since this is high-dimensional enough to illustrate the key points and low-dimensional enough to visualise. For this example, we choose $\phi$ to send $[\smn1, 1]$ to a semicircle. This choice is made purely for easy visualisation of $\phi$ and $\Phi$.

Any continuous target function $f:[\smn1, 1]^3 \to \mathbb{R}$ may be thought of as deforming a cube into a line segment. We actually do not need to consider the whole cube -- since $f$ is permutation-invariant, it is fully determined by its behaviour on the set of points with descending-ordered coordinates. In general we will denote this set $\Delta_n$,

\begin{equation}
\label{eq:delta_n_defn}
\Delta_n := \left\lbrace \mathbf{x} \in \mathbb{R}^n : 1 \geq x_1 \geq x_2 \geq \ldots \geq x_n \geq \smn1 \right\rbrace.
\end{equation}

Considering $\Delta_3$ instead of the whole cube allows us to think more freely about $f$, without worrying about permutation invariance -- as long as $f$ is continuous on $\Delta_3$, it can be extended in a continuous and permutation-invariant way to $[\smn1, 1]^3$. In general, if we want to define a function $f$ by its behaviour on $\Delta_n$ and implicitly extend it to the rest of the space, we can simply write it as a function of $\texttt{sort}(\mathbf{x})$. This accounts for the appearance of $\texttt{sort}$ in \Cref{eq:defn_bad_target}, which can be thought of as saying ``$f_*(\mathbf{x}) = \mathbf{w}^\intercal \mathbf{x} + b$ on $\Delta_M$, and $f_*$ is permutation-invariant''.

In general $\Delta_n$ is an $n$-simplex, so $\Delta_3$ is a tetrahedron. From our topological perspective, sum-decomposition via $\mathbb{R}^2$ attempts to break down the tetrahedron-to-line deformation $f$ by first sending $\Delta_3$ onto the plane under $\Phi$, then sending the plane onto a line under $\rho$. \Cref{fig:example_of_2d_phi,fig:delta_3,fig:image_of_delta_3} visualise the action of $\Phi$, showing how the boundaries of $\Delta_3$ are mapped to the plane for our example choice of $\phi$. Crucially, note that some of these boundaries intersect under $\Phi$. 
Once two regions have been brought together by $\Phi$, they cannot be separated again -- $\rho$ cannot send the point of intersection to two places at once. Therefore if $f$ \emph{does} separate these regions, sending them to opposite ends of $[\smn1, 1]$ as illustrated in \Cref{fig:f_star_delta_3}, then the sum-decomposition $\rho \circ \Phi$ must fail with worst-case error at least $1$.

\Cref{fig:topology_cartoon} only shows that this intersection problem occurs for our example choice of $\phi$ and $M$, but we prove below that the general case also holds -- for every $M$, and for any continuous $\phi$ which maps into $\mathbb{R}^{M-1}$, there are two disjoint regions of the surface of $\Delta_M$ which intersect under $\Phi$.
We denote the two regions $X_M^{+1}$ and $X_M^{\smn1}$ -- the solid red regions in \Cref{fig:delta_3} are $X_3^{+1}$ and $X_3^{\smn1}$. Our $f_*$ from \Cref{eq:defn_bad_target} is defined so that it separates $X_M^{+1}$ and $X_M^{\smn1}$, and it is just one among many continuous functions which have this property. In some sense it is the simplest, since it is affine on $\Delta_M$.
Our proof that $\Phi(X_M^{+1})$ and $\Phi(X_M^{\smn1})$ always intersect relies on the structure of $\Phi$ as a sum over inputs.

\edef\myindent{\the\parindent}
\noindent\begin{minipage}{\textwidth}
\setlength{\parindent}{\myindent}

We define $X_M^{+1}$ and $X_M^{\smn1}$ as follows:

\begin{align}
X_M^{+1} & := \left\{ \mathbf{x} \in \Delta_M : 1 = x_1 \geq x_2 = x_3 \geq \ldots \text{\sfrac{$=$}{$\geq$}} ~x_M~ \text{\sfrac{$\geq$}{$=$}} \smn1 \right\} \label{eq:XN_plus_defn}
\\
X_M^{\smn1} & := \left\{ \mathbf{x} \in \Delta_M : 1 \geq x_1 = x_2 \geq x_3 = \ldots \text{\sfrac{$\geq$}{$=$}} ~x_M~ \text{\sfrac{$=$}{$\geq$}} \smn1 \right\} \label{eq:XN_minus_defn}
\end{align}

The rightmost equality $x_M=\smn1$ applies to $X_M^{+1}$ when $M$ is even, and $X_M^{\smn1}$ when $M$ is odd. We can view these equations as modifying the definition of $\Delta_M$ (Equation \ref{eq:delta_n_defn}) by changing either the odd- or the even-numbered $\geq$ signs into $=$ signs. It follows directly from the definition of $f_*$ (Equation \ref{eq:defn_bad_target}) that $f_*(X_M^{+1})=1$ and $f_*(X_M^{\smn1})=\smn1$.

\end{minipage}

The above discussion can be summarised by the following claim:

\begin{claim}
\label{clm:central_claim_approx_thm}
Let $N=M-1$. There exist two subsets $X_M^{+1}$ and $X_M^{\smn1}$ of $\Delta_M$ such that:

\begin{enumerate}[leftmargin=50pt]
    \item \textbf{$f_*$ separates $X_M^{+1}$ from $X_M^{\smn1}$:} $f_*(X_M^{+1})=1$ and $f_*(X_M^{\smn1})=\smn1$.
    \item \textbf{$\Phi$ cannot separate $X_M^{+1}$ from $X_M^{\smn1}$:} For any continuous $\phi: [\smn1, 1] \to \mathbb{R}^N$, $\Phi(X_M^{\smn1}) \cap \Phi(X_M^{+1})$ is non-empty.
\end{enumerate}

\Cref{thm:main_approximation_theorem} follows immediately, since the above points imply that $f_*$ has no within-1 continuous sum-decomposition via $\mathbb{R}^N$.

\end{claim}

As noted, point 1 is true by definition, so we only need to prove point 2 to complete the proof of \Cref{thm:main_approximation_theorem}. The proof of point 2 has two parts:

\begin{enumerate}[leftmargin=3\parindent]
    \item \textbf{Lemma \ref{lem:gamma_zero_implication}}: Construct a function $\Gamma_N$ which has a zero only if $\Phi(X_M^{\smn1}) \cap \Phi(X_M^{+1})$ is non-empty.
    \item \textbf{Lemma \ref{lem:nu_n}:} Show that $\Gamma_N$ has a zero.
\end{enumerate}

The main idea in constructing $\Gamma_N$ is to subtract $\Phi(X_M^{\smn1})$ from $\Phi(X_M^{+1})$. 
This clearly evaluates to zero only where the two sets intersect.
The subtraction idea is slightly obscured in the formal statement and proof of the following Lemma, but it is the key idea on which the Lemma is based.
The subscript on $\Gamma$ is $N$, rather than $M$, because there are a total of $N = M-1$ degrees of freedom in $X_M^{+1}$ and $X_M^{\smn1}$ -- this is most easily seen by considering $M=1$, which gives zero degrees of freedom.\footnote{\Cref{eq:XN_plus_defn,eq:XN_minus_defn} imply that $X_1^{+1} = \{1\}$ and $X_1^{\smn1} = \{\smn1\}$.}

Straightforwardly using this idea results in a function whose input space is $X_M^{+1} \times X_M^{\smn1}$, but we will instead define $\Gamma_N$ to have input space $\Delta_N$.
As shown below, any point $\mathbf{z} \in \Delta_N$ may be used to construct a pair of points $(\mathbf{x}^+, \mathbf{x}^-) \in X_M^{+1} \times X_M^{\smn1}$, and the subtraction idea can then be applied to this pair of points.
We define $\Gamma_N$ in this way because $\Delta_N$ turns out to have useful structure which we can exploit for the proof of Lemma \ref{lem:nu_n}.

Importantly, the following statement of Lemma \ref{lem:gamma_zero_implication} requires that $\phi(\smn1)=\mathbf{0}$. From now on we will assume this without loss of generality, since otherwise we can reason identically about $\widetilde{\phi}(x) := \phi(x) - \phi(\smn1)$, which does satisfy $\widetilde{\phi}(\smn1)=\mathbf{0}$.

\begin{lemma}
\label{lem:gamma_zero_implication}
Let $\phi:[\smn1,1]\to\mathbb{R}^N$, with $\phi(\smn1)=\mathbf{0}$. Define $\Gamma_N:\Delta_N \to \mathbb{R}^N$ by

\begin{equation}
\label{eq:def_of_gamma}
\Gamma_N(\mathbf{x}) := \sum_{i=1}^N (\shortminus1)^i \phi(x_i) + \frac{\phi(1)}{2}.
\end{equation}

Suppose $\Gamma_N$ has a zero in $\Delta_N$. Then there exist $\mathbf{x}^+ \in X_M^{+1}$, $\mathbf{x}^- \in X_M^{\smn1}$ with

\begin{equation*}
\Phi(\mathbf{x}^+) = \Phi(\mathbf{x}^-).
\end{equation*}

\end{lemma}

\begin{proof}
Let $\mathbf{z} \in \Delta_N$ with $\Gamma_N(\mathbf{z})=\mathbf{0}$. Define $\mathbf{x}^+ \in X_M^{+1}$ and $\mathbf{x}^- \in X_M^{\smn1}$:

\begin{align}
\begin{split}
\label{eq:x_plus_minus_from_z}
x^+_i &= z_i \quad \text{for even } i \\
x^-_i &= z_i \quad \text{for odd } i
\end{split}
\end{align}

The remaining coordinates of $\mathbf{x}^+$ and $\mathbf{x}^-$ are fixed by the equalities in \Cref{eq:XN_minus_defn,eq:XN_plus_defn}. Note in particular that \Cref{eq:XN_plus_defn} implies that $x^+_1=1$. Taking these constraints together with the condition that $\phi(\smn1)=\mathbf{0}$, we obtain

\begin{align}
\begin{split}
\label{eq:Phi_plus_minus}
\Phi(\mathbf{x}^+) &= \phi(1) + 2\sum_{i\text{ even}} \phi(x_i^+) \\
\Phi(\mathbf{x}^-) &= 2\sum_{i\text{ odd}} \phi(x_i^-)
\end{split}
\end{align}

Plugging $\mathbf{z}$ into \Cref{eq:def_of_gamma} and pulling the odd terms of the sum to the left hand side (and recalling that $\Gamma_N(\mathbf{z}) = \mathbf{0}$), we obtain

\begin{equation}
\sum_{i\text{ odd}} \phi(z_i) = \sum_{i\text{ even}} \phi(z_i) + \frac{\phi(1)}{2}.
\end{equation}

Now we apply \Cref{eq:x_plus_minus_from_z} to obtain

\begin{equation}
\sum_{i\text{ odd}} \phi(x_i^-) = \sum_{i\text{ even}} \phi(x_i^+) + \frac{\phi(1)}{2}.
\end{equation}

This equates the right hand sides of the two lines of \Cref{eq:Phi_plus_minus}, so the left hand sides must also be equal, i.e. $\Phi(\mathbf{x}^+) = \Phi(\mathbf{x}^-)$.

\end{proof}

All that remains is to show that $\Gamma_N$ has a zero. To understand $\Gamma_N$, we exploit the alternating sum form of \Cref{eq:def_of_gamma}, and in particular we will see how this constrains the behaviour of $\Gamma_N$ on the surface of $\Delta_N$. Each face of the surface is itself a simplex\footnote{At this stage of the argument, it is helpful to have a way of visualising high-dimensional simplices and their faces. We discuss such a visualisation in \Cref{sec:visualising_simplices}.} of dimension $N-1$, defined by changing one of the inequalities in \Cref{eq:delta_n_defn} into an equality, i.e. by applying the constraint $x_i = x_{i+1}$. We denote the corresponding face $\smash{\Delta_N^{(i)}}$. For $1 \leq i < N$, this constraint leads to the cancellation of two consecutive terms of the sum in \Cref{eq:def_of_gamma}. Each of these faces therefore has the same image under $\Gamma_N$ -- for example, the following points belong to different faces of $\Delta_4$, but because equal pairs cancel out they all go to the same point under $\Gamma_4$.

\begin{align*}
\big(\tfrac{2}{3}, \tfrac{2}{3}, \tfrac{1}{2}, 0\big) \in \Delta_4^{(1)} &&
\big(\tfrac{1}{2}, \tfrac{1}{3}, \tfrac{1}{3}, 0\big)  \in \Delta_4^{(2)} && 
\big(\tfrac{1}{2}, 0, \smn\tfrac{1}{2}, \smn\tfrac{1}{2}\big) \in \Delta_4^{(3)}
\end{align*}

%Note that we lose a degree of freedom here -- $\Delta_4^{(I)}$ is $3$-dimensional, but only $2$ coordinates affect the value of $\Gamma_4$. 
% Each of these faces therefore has a $2$-dimensional image. In general, $\Delta_N^{(I)}$ has $N-1$ dimensions, but for $1 \leq I < N$, the image $\Gamma_N(\Delta_N^{(I)})$ has only $N-2$ dimensions.\footnote{$\Gamma_N(\Delta_N^{(I)})$ can be even lower-dimensional than this -- $\phi$ can reduce the dimensionality by, for example, mapping everything to $\mathbf{0}$.}

There are two more faces to consider -- $\Delta_N^{(0)}$ and $\Delta_N^{(N)}$, corresponding to the constraints $x_1=1$ and $x_N=\smn1$ respectively.
%These constraints do not involve equal pairs of coordinates, so we do not lose a degree of freedom, and $\Gamma_N(\Delta_N^{(0)})$ and $\Gamma_N(\Delta_N^{(N)})$ have $N-1$ dimensions. They therefore account for almost all of the surface of $\Gamma_N(\Delta_N)$. In fact, they cover the \emph{entire} surface -- for example, we can find points in $\Delta_4^{(0)}$ and $\Delta_4^{(4)}$ which match up with the three points previously considered:
%\begin{align*}
%\big(1, 1, \tfrac{1}{2}, 0\big) \in \Delta_4^{(0)} &&
%\big(\tfrac{1}{2}, 0, \smn1, \smn1 \big) \in \Delta_4^{(4)}
%\end{align*}
%Since $\Delta_4^{(0)}$ and $\Delta_4^{(4)}$ account for the surface of $\Gamma_N(\Delta_N)$, we concentrate our attention on these sets.
We can observe an interesting behaviour of $\Gamma_N$ here, which is in fact a symmetry: if we apply a ``coordinate left shift'' to a point in $\Delta_N^{(0)}$, the output of $\Gamma_N$ is multiplied by $\smn1$. By ``coordinate left shift'', we mean the bijection $\alpha: \Delta_N^{(0)} \to \Delta_N^{(N)}$ defined as follows:

\begin{align*}
\alpha(\mathbf{x})_i = 
\begin{cases}
x_{i+1} & i < N \\
\smn1 & i = N
\end{cases}
\end{align*}

Applying $\alpha$ moves each coordinate one place to the left, deleting the $1$ in the leftmost place and padding with a $\smn1$ in the rightmost place. For example, consider the point $(1, \tfrac{2}{3}, 0, \smn\tfrac{1}{2}) \in \Delta_4^{(0)}$:

\begin{equation*}
\alpha\Big( (1, \tfrac{2}{3}, 0, \smn\tfrac{1}{2}) \Big) = (\tfrac{2}{3}, 0, \smn\tfrac{1}{2}, \smn1) \in \Delta_4^{(4)}
\end{equation*}

Now apply $\Gamma_4$ to $(1, \tfrac{2}{3}, 0, \smn\tfrac{1}{2})$ and $(\tfrac{2}{3}, 0, \smn\tfrac{1}{2}, \smn1)$:

\begin{align*}
\Gamma_4\Big( (1, \tfrac{2}{3}, 0, \smn\tfrac{1}{2}) \Big) &= -\phi(1) + \phi(\tfrac{2}{3}) - \phi(0) + \phi(\smn\tfrac{1}{2}) + \frac{\phi(1)}{2}  \\
\Gamma_4\Big( (\tfrac{2}{3}, 0, \smn\tfrac{1}{2}, \smn1) \Big) &= \quad\quad\quad -\phi(\tfrac{2}{3})   + \phi(0) - \phi(\smn\tfrac{1}{2}) + \phi(\smn1) + \frac{\phi(1)}{2}
\end{align*}
 
Simplifying the $\phi(1)$ terms, and recalling that $\phi(\smn1)=0$, we obtain

\begin{equation*}
 \Gamma_4\Big( (\tfrac{2}{3}, 0, \smn\tfrac{1}{2}, \smn1) \Big) = -\Gamma_4\Big( (1, \tfrac{2}{3}, 0, \smn\tfrac{1}{2}) \Big)
\end{equation*}

It is easily seen that this generalises to any $\mathbf{x} \in \Delta_N^{(0)}$:

\begin{equation}
\label{eq:gamma_alpha_negation}
\Gamma_N\big( \alpha(\mathbf{x}) \big) = -\Gamma_N\big( \mathbf{x} \big)
\end{equation}

We can exploit this symmetry to show that $\Gamma_N$ has a zero. To do this, we will need the Borsuk-Ulam theorem (\citealp{Borsuk1933} -- see \citealp{lloyd1978degree} for a detailed English-language treatment). One statement of the theorem is as follows.

\begin{theorem}[Borsuk-Ulam]
\label{thm:borsuk_ulam}
Let $I_n := [\smn1, 1]^n$ be the unit $n$-cube. Let $\partial I_n$ be the surface of $I_n$, $\partial I_n := \{ \mathbf{x} \in I_n : x_i = \pm1 \text{ for some }i \}$.
Suppose $f:I_n\to\mathbb{R}^n$ is continuous, and for every $\mathbf{x} \in \partial I_n$ we have

\begin{equation}
\label{eq:borsuk_ulam_condition}
f(\smn\mathbf{x}) = - f(\mathbf{x})
\end{equation}

Then there is some $\mathbf{z} \in I_n$ such that $f(\mathbf{z})=0$.
\end{theorem}

There is a clear similarity between \Cref{eq:borsuk_ulam_condition,eq:gamma_alpha_negation}, but we cannot directly apply Borsuk-Ulam to $\Gamma_N$.
We need to bridge the two equations, relating the left-shift $\alpha$ on the left hand side of \Cref{eq:gamma_alpha_negation} to the negation on the left hand side of \Cref{eq:borsuk_ulam_condition}. We can do this by finding a continuous function $\nu_N: I_N \to \Delta_N$ which turns negation on $\partial I_N$ into a left-shift on $\partial \Delta_N$:\footnote{Here, we are glossing over the fact that $\alpha$ is only defined on part of $\partial\Delta_N$. We will come back to this when proving Lemma \ref{lem:nu_n}.}

\begin{equation}
\label{eq:bridge}
\nu_N(\smn\mathbf{x}) = \alpha(\nu_N(\mathbf{x}))
\end{equation}

That is, we want to be able to apply Borsuk-Ulam to $\Gamma_N \circ \nu_N$:

\begin{align*}
f(\smn\mathbf{x}) := \Gamma_N\Big( \nu_N(\smn\mathbf{x}) \Big) &= \Gamma_N\Big( \alpha\big( \nu_N(\mathbf{x}) \big) \Big) \\
&= -\Gamma_N\Big( \nu_N(\mathbf{x}) \Big) = -f(\mathbf{x})
\end{align*}

Showing that a suitable $\nu_N$ exists will complete the proof.

\begin{restatable}{lemma}{lemmanu}
\label{lem:nu_n}
Let $n \in \mathbb{N}$. Let $\Gamma_n$ be defined as in \Cref{eq:def_of_gamma}. Then there exists a continuous function $\nu_n : I_n \to \Delta_n$ such that, for $\mathbf{x} \in \partial I_n$

\begin{equation}
\label{eq:gamma_borsuk_ulam_condition}
\Gamma_n\big( \nu_n( \smn\mathbf{x} ) \big) = -\Gamma_n\big( \nu_n( \mathbf{x} ) \big)
\end{equation}

\end{restatable}

We present a partial proof here, with full details available in \Cref{sec:approximation_appendix}.
\vspace{\baselineskip}

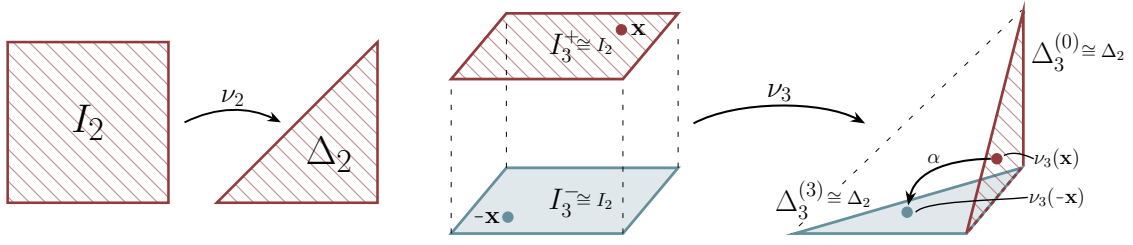
\begin{figure}
    \centering
    \begin{subfigure}{0.33\textwidth}
        \resizebox{\textwidth}{!}{\input{figures/nu_cartoon_2d.tikz}}
    \end{subfigure}
    \hspace{0.04\textwidth}
    \begin{subfigure}{0.6\textwidth}
        \resizebox{\textwidth}{!}{\input{figures/nu_cartoon_3d.tikz}}
    \end{subfigure}
    
    \caption[Beginning the construction of $\nu_n$ from $\nu_{n-1}$.]{For the first part of the construction of $\nu_n$, we use $\nu_{n-1}$ to define $\nu_n$ on the top and bottom faces of $I_n$. The $n=3$ case is illustrated here. $\nu_n$ uses $\nu_{n-1}$ to send the top face of $I_n$ to $\Delta_n^{(0)}$, and the bottom face to $\Delta_n^{(n)}$. We construct $\nu_n$ so that opposite points $\mathbf{x}$ and $\smn\mathbf{x}$ on the top and bottom faces go to points in $\Delta_n$ which are related by the left-shift function $\alpha$ -- that is, $\nu_n(\smn\mathbf{x}) = \alpha(\nu_n(\mathbf{x}))$.}
    \label{fig:nu_cartoon}
\end{figure}
\begin{figure}
    \centering
    \begin{subfigure}{0.3\textwidth}
        \centering
        \resizebox{\textwidth}{!}{\input{figures/vertical_cartoon.tikz}}
    \end{subfigure}
    \hspace{0.1\textwidth}
    \begin{subfigure}{0.25\textwidth}
        \centering
        \resizebox{\textwidth}{!}{\input{figures/x_bar_cartoon.tikz}}
    \end{subfigure}    
    
    \caption[Continuing the construction of $\nu_n$ from $\nu_{n-1}$.]{For the second part of the construction of $\nu_n$, we note that any point $\mathbf{x}$ lies on a ``vertical'' line, i.e. an $n$-th-axis-aligned line, meeting the top face at $\mathbf{x}^+$ and the bottom face at $\mathbf{x}^-$. We complete the construction of $\nu_n$ by interpolating along these lines. For a pair of opposite points $\mathbf{y}=\smn\mathbf{x}$, we see that $\mathbf{y}^- = \smn\mathbf{x}^+$. To ensure that $\nu_n$ has the desired properties, we must pay particular attention to the surface of $I_n$ without the top and bottom faces, i.e. $\partial I_n \backslash (I_n^+ \cup I_n^-)$. If $\mathbf{x}$ lies on this part of the surface (as in this figure), then the point $\overline{\mathbf{x}}$, which is the projection of $\mathbf{x}$ along the vertical, lies on the surface $\partial I_{n-1}$.}
    \label{fig:vertical_lines}
\end{figure}
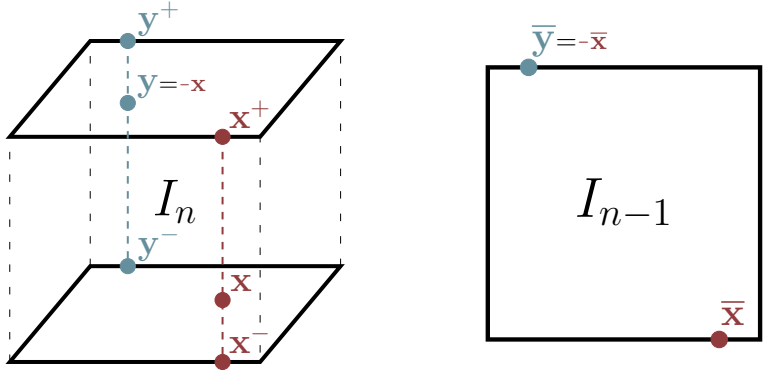

\begin{proof}
First, we must point out that the proof strategy is not quite as simple as suggested above. Since $\alpha$ is only defined on $\Delta_n^{(0)}$, and not on the rest of $\partial \Delta_n$, we cannot in fact satisfy \Cref{eq:bridge} for every $\mathbf{x}$ on $\partial I_n$. Nevertheless, we can still satisfy \Cref{eq:gamma_borsuk_ulam_condition} for every $\mathbf{x}$ on $\partial I_n$, and we reach this conclusion by showing that \Cref{eq:bridge} holds on part of $\partial I_n$.

\clearpage

We proceed by induction -- the base case $\nu_1(\mathbf{x}) := \mathbf{x}$ is trivial. Consider the ``top'' and ``bottom'' faces of $I_n$:

\begin{align*}
    \phantom{\text{The top face.}} && I_n^+ &:= \{ \mathbf{x} \in I_n : x_n = 1 \} && \text{The top face.} \\
    \phantom{\text{The bottom face.}} && I_n^- &:= \{ \mathbf{x} \in I_n : x_n = \smn1 \} && \text{The bottom face.}
\end{align*}

Note that $I_n^+$ and $I_n^-$ are naturally equivalent to $I_{n-1}$ -- we can map bijectively from either face to $I_{n-1}$ by dropping the last coordinate of $\mathbf{x}$. For any point $\mathbf{x}$, we'll denote the corresponding point in $I_{n-1}$ by $\overline{\mathbf{x}}$ -- that is, $\overline{\mathbf{x}}$ is $\mathbf{x}$ with the last coordinate removed. 

Note also that $\Delta_n^{(0)}$ and $\Delta_n^{(n)}$ are naturally equivalent to $\Delta_{n-1}$ -- any point $\mathbf{x} \in \Delta_n^{(0)}$ can be written as $(1, \mathbf{x}^\prime)$, where $\mathbf{x}^\prime \in \Delta_{n-1}$. Similarly any point in $\Delta_n^{(n)}$ can be written as $(\mathbf{x}^\prime, \smn1)$. We can use these equivalences to the lower-dimensional case to come up with a partial definition of $\nu_n$ in terms of $\nu_{n-1}$. 

The idea of how to begin constructing $\nu_n$ using these equivalences is illustrated in \Cref{fig:nu_cartoon}. In the following, we will choose the definition of $\nu_n$ on $I_n^+$ and $I_n^-$ to satisfy what is shown in this figure. That is, by viewing $I_n^+$ as a copy of $I_{n-1}$, and $\Delta_n^{(0)}$ as a copy of $\Delta_{n-1}$, we can map from $I_n^+$ to $\Delta_n^{(0)}$ using $\nu_{n-1} : I_{n-1} \to \Delta_{n-1}$. Crucially, note that the truth of \Cref{eq:bridge} for $\mathbf{x} \in I_n^+$ depends on how $\nu_n$ behaves on $I_n^-$. There is only one way to define $\nu_n$ on $I_n^-$ so that \Cref{eq:bridge} is true on $I_n^+$, and this is the definition that we choose.

Written mathematically, we define $\nu_n$ on $I_n^+$ and $I_n^-$ as follows:

\begin{align*}
\nu_n(\mathbf{x}) &:= 
\begin{cases}
\big( 1, \nu_{n-1}(\overline{\mathbf{x}}) \big) \quad \text{for } \mathbf{x} \in I_n^+ \\
\alpha\big( \nu_n(\smn\mathbf{x}) \big) = \big( \nu_{n-1}(\smn\overline{\mathbf{x}}), \smn1 \big) \quad \text{for } \mathbf{x} \in I_n^-
\end{cases}
\end{align*}

So by definition, $\nu_n$ satisfies \Cref{eq:bridge} for $\mathbf{x} \in I_n^+$. This immediately implies \Cref{eq:gamma_borsuk_ulam_condition} on $I_n^+$ and, by symmetry, on $I_n^-$. 

We have established that \Cref{eq:gamma_borsuk_ulam_condition} holds on $I_n^+ \cup I_n^-$ for our partially defined $\nu_n$. We now need to extend $\nu_n$ to the rest of $I_n$. We will define this extension by interpolating along the vertical lines joining $I_n^+$ and $I_n^-$. By ``vertical'', we mean parallel to the $n$-th coordinate axis.

Consider a point $\mathbf{x}$ lying on the part of $\partial I_n$ that we have not yet considered, i.e. $\partial I_n \backslash (I_n^+ \cup I_n^-)$. Any such point lies on a vertical line which meets $\partial I_n^+$ and $\partial I_n^-$ at points $\mathbf{x}^+$ and $\mathbf{x}^-$, as depicted in \Cref{fig:vertical_lines}. We refer to these vertical line segments on $\partial I_n$ as \emph{surface verticals}. Crucially, $\partial I_n \backslash (I_n^+ \cup I_n^-)$ is entirely covered by surface verticals.

We want to ensure that \Cref{eq:gamma_borsuk_ulam_condition} holds on all of $\partial I_n$, and because we have already shown this for $I_n^+ \cup I_n^-$, it only remains to ensure that it holds on all of the surface verticals. 
One way to achieve this is to make $\Gamma_n \circ \nu_n$ constant on each surface vertical. 
To see that this would imply \Cref{eq:gamma_borsuk_ulam_condition}, consider a pair of antipodal\footnote{That is, such that $\mathbf{y} = \smn\mathbf{x}$.} points $\mathbf{x}$ and $\mathbf{y}$ on $\partial I_n \backslash (I_n^+ \cup I_n^-)$. 
As illustrated in \Cref{fig:vertical_lines}, $\mathbf{x}^+$ and $\mathbf{y}^-$ are also antipodal. 
Since $\mathbf{x}^+$ and $\mathbf{y}^-$ belong to $I_n^+ \cup I_n^-$, we know that \Cref{eq:gamma_borsuk_ulam_condition} holds: 

\begin{equation*}
\Gamma_n\big( \nu_n( \mathbf{x}^+ ) \big) = -\Gamma_n\big( \nu_n( \mathbf{y}^- ) \big)
\end{equation*}

If $\Gamma_n \circ \nu_n$ is constant on surface verticals, it follows that $\Gamma_n\big( \nu_n(\mathbf{x}^+) \big) = \Gamma_n\big( \nu_n(\mathbf{x}) \big)$ and $\Gamma_n\big( \nu_n(\mathbf{y}^-) \big) = \Gamma_n\big( \nu_n(\mathbf{y}) \big)$, and therefore \Cref{eq:gamma_borsuk_ulam_condition} also holds for $\mathbf{x}$ and $\mathbf{y}$:

\begin{equation*}
\Gamma_n\big( \nu_n( \mathbf{x} ) \big) = \Gamma_n\big( \nu_n( \mathbf{x}^+ ) \big) = -\Gamma_n\big( \nu_n( \mathbf{y}^- ) \big) = -\Gamma_n\big( \nu_n( \mathbf{y} ) \big)
\end{equation*}

Because the ends of a surface vertical, $\mathbf{x}^+$ and $\mathbf{x}^-$, belong to $I_n^+$ and $I_n^-$, the behaviour here is already fixed by our partial definition of $\nu_n$. Therefore, to have any hope of holding $\Gamma_n \circ \nu_n$ constant on surface verticals, we must first show that $\Gamma_n\big( \nu_n(\mathbf{x}^+) \big) = \Gamma_n\big( \nu_n( \mathbf{x}^- ) \big)$ under our existing partial definition of $\nu_n$.

By expanding the definition of $\Gamma_n$, consider $\Gamma_n\big( \nu_n(\mathbf{x}^+) \big)$:

\begin{align*}
\Gamma_n\big( \nu_n(\mathbf{x}^+) \big) &= \Gamma_n\Big( \big( 1, \nu_{n-1}(\overline{\mathbf{x}}) \big) \Big) \\
&= -\phi(1) - \Big( \Gamma_{n-1}\big( \nu_{n-1}( \overline{ \mathbf{x}}) \big) - \frac{\phi(1)}{2} \Big) + \frac{\phi(1)}{2} \\
&= -\Gamma_{n-1}\big( \nu_{n-1}( \overline{ \mathbf{x} } ) \big)
\end{align*}

Similarly we find that for $\mathbf{x}^-$:

\begin{align*}
\Gamma_n\big( \nu_n( \mathbf{x}^- ) \big) &= \Gamma_{n-1}\big( \nu_{n-1}( \smn\overline{ \mathbf{x} } ) \big)
\end{align*}

But note that $\overline{\mathbf{x}} \in \partial I_{n-1}$, as illustrated in \Cref{fig:vertical_lines}. By induction, \Cref{eq:gamma_borsuk_ulam_condition} therefore applies:

\begin{align}
\Gamma_{n-1}\big( \nu_{n-1}( \smn\overline{ \mathbf{x} } ) \big) &= -\Gamma_{n-1}\big( \nu_{n-1}( \overline{ \mathbf{x} } ) \big) \nonumber \\
\Gamma_n\big( \nu_n(\mathbf{x}^-) \big) &= \Gamma_n\big( \nu_n( \mathbf{x}^+ ) \big) \label{eq:gamma_equivalent_each_end}
\end{align}

We have shown that $\Gamma_n \circ \nu_n$ takes the same value at either end of each surface vertical. It remains to extend $\nu_n$ so that $\Gamma_n \circ \nu_n$ remains constant along the entire surface vertical. We can perform this extension by interpolating along the vertical between $I_n^+$ and $I_n^-$. Linear interpolation does \emph{not} respect this condition of constancy -- the definition and justification of the correct interpolation scheme are given in \Cref{sec:approximation_appendix}.

\end{proof}

With $\nu_N$ in hand, the proof of \Cref{thm:main_approximation_theorem} is complete. We briefly summarise the chain of reasoning from this section, which is now fully justified by the proof of the above Lemma.

\begin{enumerate}[leftmargin=3\parindent]
    \item We are attempting to model permutation-invariant functions on the input space $[\smn1,1]^M$ using sum-decomposition, where the dimension of the latent space, $N$, is one less than that of the input space, $M$.
    \item We have shown that, for any choice of continuous encoding map $\phi : [\smn1,1] \to \mathbb{R}^N$, the function $\Gamma_N$ (which depends on $\phi$) has a zero.
    \item Because of the way $\Gamma_N$ was constructed, this implies that there are two disjoint regions of the input space\footnote{Strictly speaking, two disjoint regions of the ordered input space $\Delta_M$.} $[\smn1,1]^M$ which collide under the application of $\Phi$.
    \item Any target function which widely separates these regions is therefore impossible to approximate by sum-decomposition -- there will always be a pair of points which collide under the sum-decomposition but are widely separated by the target.
    \item One example of such a function, which therefore cannot be approximated via $\mathbb{R}^N$, is given by \Cref{eq:defn_bad_target}.
\end{enumerate}

%% file: figures/delta_3_boundaries.pgf
%% Creator: Matplotlib, PGF backend
%%
%% To include the figure in your LaTeX document, write
%%   \input{<filename>.pgf}
%%
%% Make sure the required packages are loaded in your preamble
%%   \usepackage{pgf}
%%
%% and, on pdftex
%%   \usepackage[utf8]{inputenc}\DeclareUnicodeCharacter{2212}{-}
%%
%% or, on luatex and xetex
%%   \usepackage{unicode-math}
%%
%% Figures using additional raster images can only be included by \input if
%% they are in the same directory as the main LaTeX file. For loading figures
%% from other directories you can use the `import` package
%%   \usepackage{import}
%%
%% and then include the figures with
%%   \import{<path to file>}{<filename>.pgf}
%%
%% Matplotlib used the following preamble
%%   \usepackage{fontspec}
%%   \setmainfont{DejaVuSerif.ttf}[Path=/usr/local/lib/python3.7/site-packages/matplotlib/mpl-data/fonts/ttf/]
%%   \setsansfont{DejaVuSans.ttf}[Path=/usr/local/lib/python3.7/site-packages/matplotlib/mpl-data/fonts/ttf/]
%%   \setmonofont{DejaVuSansMono.ttf}[Path=/usr/local/lib/python3.7/site-packages/matplotlib/mpl-data/fonts/ttf/]
%%
\begingroup%
\makeatletter%
\begin{pgfpicture}%
\pgfpathrectangle{\pgfpointorigin}{\pgfqpoint{2.700000in}{2.700000in}}%
\pgfusepath{use as bounding box, clip}%
\begin{pgfscope}%
\pgfsetbuttcap%
\pgfsetmiterjoin%
\definecolor{currentfill}{rgb}{1.000000,1.000000,1.000000}%
\pgfsetfillcolor{currentfill}%
\pgfsetlinewidth{0.000000pt}%
\definecolor{currentstroke}{rgb}{1.000000,1.000000,1.000000}%
\pgfsetstrokecolor{currentstroke}%
\pgfsetdash{}{0pt}%
\pgfpathmoveto{\pgfqpoint{0.000000in}{0.000000in}}%
\pgfpathlineto{\pgfqpoint{2.700000in}{0.000000in}}%
\pgfpathlineto{\pgfqpoint{2.700000in}{2.700000in}}%
\pgfpathlineto{\pgfqpoint{0.000000in}{2.700000in}}%
\pgfpathclose%
\pgfusepath{fill}%
\end{pgfscope}%
\begin{pgfscope}%
\pgfsetbuttcap%
\pgfsetmiterjoin%
\definecolor{currentfill}{rgb}{1.000000,1.000000,1.000000}%
\pgfsetfillcolor{currentfill}%
\pgfsetlinewidth{0.000000pt}%
\definecolor{currentstroke}{rgb}{0.000000,0.000000,0.000000}%
\pgfsetstrokecolor{currentstroke}%
\pgfsetstrokeopacity{0.000000}%
\pgfsetdash{}{0pt}%
\pgfpathmoveto{\pgfqpoint{0.120000in}{0.120000in}}%
\pgfpathlineto{\pgfqpoint{2.580000in}{0.120000in}}%
\pgfpathlineto{\pgfqpoint{2.580000in}{2.580000in}}%
\pgfpathlineto{\pgfqpoint{0.120000in}{2.580000in}}%
\pgfpathclose%
\pgfusepath{fill}%
\end{pgfscope}%
\begin{pgfscope}%
\pgfsetbuttcap%
\pgfsetmiterjoin%
\definecolor{currentfill}{rgb}{0.950000,0.950000,0.950000}%
\pgfsetfillcolor{currentfill}%
\pgfsetfillopacity{0.500000}%
\pgfsetlinewidth{1.003750pt}%
\definecolor{currentstroke}{rgb}{0.950000,0.950000,0.950000}%
\pgfsetstrokecolor{currentstroke}%
\pgfsetstrokeopacity{0.500000}%
\pgfsetdash{}{0pt}%
\pgfpathmoveto{\pgfqpoint{2.239993in}{0.828144in}}%
\pgfpathlineto{\pgfqpoint{1.239541in}{1.028950in}}%
\pgfpathlineto{\pgfqpoint{1.236192in}{2.177490in}}%
\pgfpathlineto{\pgfqpoint{2.261013in}{2.032078in}}%
\pgfusepath{stroke,fill}%
\end{pgfscope}%
\begin{pgfscope}%
\pgfsetbuttcap%
\pgfsetmiterjoin%
\definecolor{currentfill}{rgb}{0.900000,0.900000,0.900000}%
\pgfsetfillcolor{currentfill}%
\pgfsetfillopacity{0.500000}%
\pgfsetlinewidth{1.003750pt}%
\definecolor{currentstroke}{rgb}{0.900000,0.900000,0.900000}%
\pgfsetstrokecolor{currentstroke}%
\pgfsetstrokeopacity{0.500000}%
\pgfsetdash{}{0pt}%
\pgfpathmoveto{\pgfqpoint{0.507268in}{0.735735in}}%
\pgfpathlineto{\pgfqpoint{1.239541in}{1.028950in}}%
\pgfpathlineto{\pgfqpoint{1.236192in}{2.177490in}}%
\pgfpathlineto{\pgfqpoint{0.485282in}{1.965043in}}%
\pgfusepath{stroke,fill}%
\end{pgfscope}%
\begin{pgfscope}%
\pgfsetbuttcap%
\pgfsetmiterjoin%
\definecolor{currentfill}{rgb}{0.925000,0.925000,0.925000}%
\pgfsetfillcolor{currentfill}%
\pgfsetfillopacity{0.500000}%
\pgfsetlinewidth{1.003750pt}%
\definecolor{currentstroke}{rgb}{0.925000,0.925000,0.925000}%
\pgfsetstrokecolor{currentstroke}%
\pgfsetstrokeopacity{0.500000}%
\pgfsetdash{}{0pt}%
\pgfpathmoveto{\pgfqpoint{1.546244in}{0.502850in}}%
\pgfpathlineto{\pgfqpoint{2.239993in}{0.828144in}}%
\pgfpathlineto{\pgfqpoint{1.239541in}{1.028950in}}%
\pgfpathlineto{\pgfqpoint{0.507268in}{0.735735in}}%
\pgfusepath{stroke,fill}%
\end{pgfscope}%
\begin{pgfscope}%
\pgfsetrectcap%
\pgfsetroundjoin%
\pgfsetlinewidth{0.803000pt}%
\definecolor{currentstroke}{rgb}{0.000000,0.000000,0.000000}%
\pgfsetstrokecolor{currentstroke}%
\pgfsetdash{}{0pt}%
\pgfpathmoveto{\pgfqpoint{2.239993in}{0.828144in}}%
\pgfpathlineto{\pgfqpoint{1.546244in}{0.502850in}}%
\pgfusepath{stroke}%
\end{pgfscope}%
\begin{pgfscope}%
\definecolor{textcolor}{rgb}{0.000000,0.000000,0.000000}%
\pgfsetstrokecolor{textcolor}%
\pgfsetfillcolor{textcolor}%
\pgftext[x=2.06in,y=0.44in,,]{\color{textcolor}\rmfamily\fontsize{8.000000}{9.600000}\selectfont \(\displaystyle x\)}%
\end{pgfscope}%
\begin{pgfscope}%
\pgfsetbuttcap%
\pgfsetroundjoin%
\pgfsetlinewidth{0.803000pt}%
\definecolor{currentstroke}{rgb}{0.690196,0.690196,0.690196}%
\pgfsetstrokecolor{currentstroke}%
\pgfsetdash{}{0pt}%
\pgfpathmoveto{\pgfqpoint{1.593641in}{0.525074in}}%
\pgfpathlineto{\pgfqpoint{0.557118in}{0.755695in}}%
\pgfpathlineto{\pgfqpoint{0.536484in}{1.979529in}}%
\pgfusepath{stroke}%
\end{pgfscope}%
\begin{pgfscope}%
\pgfsetbuttcap%
\pgfsetroundjoin%
\pgfsetlinewidth{0.803000pt}%
\definecolor{currentstroke}{rgb}{0.690196,0.690196,0.690196}%
\pgfsetstrokecolor{currentstroke}%
\pgfsetdash{}{0pt}%
\pgfpathmoveto{\pgfqpoint{2.198886in}{0.808869in}}%
\pgfpathlineto{\pgfqpoint{1.195994in}{1.011513in}}%
\pgfpathlineto{\pgfqpoint{1.191610in}{2.164877in}}%
\pgfusepath{stroke}%
\end{pgfscope}%
\begin{pgfscope}%
\pgfsetrectcap%
\pgfsetroundjoin%
\pgfsetlinewidth{0.803000pt}%
\definecolor{currentstroke}{rgb}{0.000000,0.000000,0.000000}%
\pgfsetstrokecolor{currentstroke}%
\pgfsetdash{}{0pt}%
\pgfpathmoveto{\pgfqpoint{1.584896in}{0.527019in}}%
\pgfpathlineto{\pgfqpoint{1.611154in}{0.521177in}}%
\pgfusepath{stroke}%
\end{pgfscope}%
\begin{pgfscope}%
\definecolor{textcolor}{rgb}{0.000000,0.000000,0.000000}%
\pgfsetstrokecolor{textcolor}%
\pgfsetfillcolor{textcolor}%
\pgftext[x=1.729790in,y=0.337758in,,top]{\color{textcolor}\rmfamily\fontsize{8.000000}{9.600000}\selectfont \(\displaystyle -1\)}%
\end{pgfscope}%
\begin{pgfscope}%
\pgfsetrectcap%
\pgfsetroundjoin%
\pgfsetlinewidth{0.803000pt}%
\definecolor{currentstroke}{rgb}{0.000000,0.000000,0.000000}%
\pgfsetstrokecolor{currentstroke}%
\pgfsetdash{}{0pt}%
\pgfpathmoveto{\pgfqpoint{2.190453in}{0.810573in}}%
\pgfpathlineto{\pgfqpoint{2.215772in}{0.805457in}}%
\pgfusepath{stroke}%
\end{pgfscope}%
\begin{pgfscope}%
\definecolor{textcolor}{rgb}{0.000000,0.000000,0.000000}%
\pgfsetstrokecolor{textcolor}%
\pgfsetfillcolor{textcolor}%
\pgftext[x=2.328343in,y=0.634593in,,top]{\color{textcolor}\rmfamily\fontsize{8.000000}{9.600000}\selectfont \(\displaystyle 1\)}%
\end{pgfscope}%
\begin{pgfscope}%
\pgfsetrectcap%
\pgfsetroundjoin%
\pgfsetlinewidth{0.803000pt}%
\definecolor{currentstroke}{rgb}{0.000000,0.000000,0.000000}%
\pgfsetstrokecolor{currentstroke}%
\pgfsetdash{}{0pt}%
\pgfpathmoveto{\pgfqpoint{0.507268in}{0.735735in}}%
\pgfpathlineto{\pgfqpoint{1.546244in}{0.502850in}}%
\pgfusepath{stroke}%
\end{pgfscope}%
\begin{pgfscope}%
\definecolor{textcolor}{rgb}{0.000000,0.000000,0.000000}%
\pgfsetstrokecolor{textcolor}%
\pgfsetfillcolor{textcolor}%
\pgftext[x=0.92in,y=0.36in,,]{\color{textcolor}\rmfamily\fontsize{8.000000}{9.600000}\selectfont \(\displaystyle y\)}%
\end{pgfscope}%
\begin{pgfscope}%
\pgfsetbuttcap%
\pgfsetroundjoin%
\pgfsetlinewidth{0.803000pt}%
\definecolor{currentstroke}{rgb}{0.690196,0.690196,0.690196}%
\pgfsetstrokecolor{currentstroke}%
\pgfsetdash{}{0pt}%
\pgfpathmoveto{\pgfqpoint{2.192601in}{2.041785in}}%
\pgfpathlineto{\pgfqpoint{2.173282in}{0.841534in}}%
\pgfpathlineto{\pgfqpoint{1.476716in}{0.518434in}}%
\pgfusepath{stroke}%
\end{pgfscope}%
\begin{pgfscope}%
\pgfsetbuttcap%
\pgfsetroundjoin%
\pgfsetlinewidth{0.803000pt}%
\definecolor{currentstroke}{rgb}{0.690196,0.690196,0.690196}%
\pgfsetstrokecolor{currentstroke}%
\pgfsetdash{}{0pt}%
\pgfpathmoveto{\pgfqpoint{1.298351in}{2.168670in}}%
\pgfpathlineto{\pgfqpoint{1.300291in}{1.016757in}}%
\pgfpathlineto{\pgfqpoint{0.570129in}{0.721644in}}%
\pgfusepath{stroke}%
\end{pgfscope}%
\begin{pgfscope}%
\pgfsetrectcap%
\pgfsetroundjoin%
\pgfsetlinewidth{0.803000pt}%
\definecolor{currentstroke}{rgb}{0.000000,0.000000,0.000000}%
\pgfsetstrokecolor{currentstroke}%
\pgfsetdash{}{0pt}%
\pgfpathmoveto{\pgfqpoint{1.482723in}{0.521221in}}%
\pgfpathlineto{\pgfqpoint{1.464678in}{0.512851in}}%
\pgfusepath{stroke}%
\end{pgfscope}%
\begin{pgfscope}%
\definecolor{textcolor}{rgb}{0.000000,0.000000,0.000000}%
\pgfsetstrokecolor{textcolor}%
\pgfsetfillcolor{textcolor}%
\pgftext[x=1.382121in,y=0.317144in,,top]{\color{textcolor}\rmfamily\fontsize{8.000000}{9.600000}\selectfont \(\displaystyle -1\)}%
\end{pgfscope}%
\begin{pgfscope}%
\pgfsetrectcap%
\pgfsetroundjoin%
\pgfsetlinewidth{0.803000pt}%
\definecolor{currentstroke}{rgb}{0.000000,0.000000,0.000000}%
\pgfsetstrokecolor{currentstroke}%
\pgfsetdash{}{0pt}%
\pgfpathmoveto{\pgfqpoint{0.576405in}{0.724181in}}%
\pgfpathlineto{\pgfqpoint{0.557554in}{0.716562in}}%
\pgfusepath{stroke}%
\end{pgfscope}%
\begin{pgfscope}%
\definecolor{textcolor}{rgb}{0.000000,0.000000,0.000000}%
\pgfsetstrokecolor{textcolor}%
\pgfsetfillcolor{textcolor}%
\pgftext[x=0.474156in,y=0.530939in,,top]{\color{textcolor}\rmfamily\fontsize{8.000000}{9.600000}\selectfont \(\displaystyle 1\)}%
\end{pgfscope}%
\begin{pgfscope}%
\pgfsetrectcap%
\pgfsetroundjoin%
\pgfsetlinewidth{0.803000pt}%
\definecolor{currentstroke}{rgb}{0.000000,0.000000,0.000000}%
\pgfsetstrokecolor{currentstroke}%
\pgfsetdash{}{0pt}%
\pgfpathmoveto{\pgfqpoint{0.507268in}{0.735735in}}%
\pgfpathlineto{\pgfqpoint{0.485282in}{1.965043in}}%
\pgfusepath{stroke}%
\end{pgfscope}%
\begin{pgfscope}%
\definecolor{textcolor}{rgb}{0.000000,0.000000,0.000000}%
\pgfsetstrokecolor{textcolor}%
\pgfsetfillcolor{textcolor}%
\pgftext[x=0.25in,y=1.315887in,,]{\color{textcolor}\rmfamily\fontsize{8.000000}{9.600000}\selectfont \(\displaystyle z\)}%%
\end{pgfscope}%
\begin{pgfscope}%
\pgfsetbuttcap%
\pgfsetroundjoin%
\pgfsetlinewidth{0.803000pt}%
\definecolor{currentstroke}{rgb}{0.690196,0.690196,0.690196}%
\pgfsetstrokecolor{currentstroke}%
\pgfsetdash{}{0pt}%
\pgfpathmoveto{\pgfqpoint{0.505901in}{0.812167in}}%
\pgfpathlineto{\pgfqpoint{1.239333in}{1.100477in}}%
\pgfpathlineto{\pgfqpoint{2.241301in}{0.903038in}}%
\pgfusepath{stroke}%
\end{pgfscope}%
\begin{pgfscope}%
\pgfsetbuttcap%
\pgfsetroundjoin%
\pgfsetlinewidth{0.803000pt}%
\definecolor{currentstroke}{rgb}{0.690196,0.690196,0.690196}%
\pgfsetstrokecolor{currentstroke}%
\pgfsetdash{}{0pt}%
\pgfpathmoveto{\pgfqpoint{0.486714in}{1.884979in}}%
\pgfpathlineto{\pgfqpoint{1.236410in}{2.102808in}}%
\pgfpathlineto{\pgfqpoint{2.259645in}{1.953707in}}%
\pgfusepath{stroke}%
\end{pgfscope}%
\begin{pgfscope}%
\pgfsetrectcap%
\pgfsetroundjoin%
\pgfsetlinewidth{0.803000pt}%
\definecolor{currentstroke}{rgb}{0.000000,0.000000,0.000000}%
\pgfsetstrokecolor{currentstroke}%
\pgfsetdash{}{0pt}%
\pgfpathmoveto{\pgfqpoint{0.512205in}{0.814645in}}%
\pgfpathlineto{\pgfqpoint{0.493271in}{0.807202in}}%
\pgfusepath{stroke}%
\end{pgfscope}%
\begin{pgfscope}%
\definecolor{textcolor}{rgb}{0.000000,0.000000,0.000000}%
\pgfsetstrokecolor{textcolor}%
\pgfsetfillcolor{textcolor}%
\pgftext[x=0.284080in,y=0.800552in,,top]{\color{textcolor}\rmfamily\fontsize{8.000000}{9.600000}\selectfont \(\displaystyle -1\)}%
\end{pgfscope}%
\begin{pgfscope}%
\pgfsetrectcap%
\pgfsetroundjoin%
\pgfsetlinewidth{0.803000pt}%
\definecolor{currentstroke}{rgb}{0.000000,0.000000,0.000000}%
\pgfsetstrokecolor{currentstroke}%
\pgfsetdash{}{0pt}%
\pgfpathmoveto{\pgfqpoint{0.493167in}{1.886854in}}%
\pgfpathlineto{\pgfqpoint{0.473784in}{1.881222in}}%
\pgfusepath{stroke}%
\end{pgfscope}%
\begin{pgfscope}%
\definecolor{textcolor}{rgb}{0.000000,0.000000,0.000000}%
\pgfsetstrokecolor{textcolor}%
\pgfsetfillcolor{textcolor}%
\pgftext[x=0.259973in,y=1.876189in,,top]{\color{textcolor}\rmfamily\fontsize{8.000000}{9.600000}\selectfont \(\displaystyle 1\)}%
\end{pgfscope}%
\begin{pgfscope}%
\pgfpathrectangle{\pgfqpoint{0.120000in}{0.120000in}}{\pgfqpoint{2.460000in}{2.460000in}}%
\pgfusepath{clip}%
\pgfsetbuttcap%
\pgfsetroundjoin%
\pgfsetlinewidth{1.003750pt}%
\definecolor{currentstroke}{rgb}{0.286275,0.505882,0.549020}%
\pgfsetstrokecolor{currentstroke}%
\pgfsetdash{{3.700000pt}{1.600000pt}}{0.000000pt}%
\pgfpathmoveto{\pgfqpoint{1.524524in}{0.620167in}}%
\pgfpathlineto{\pgfqpoint{2.133142in}{0.897374in}}%
\pgfusepath{stroke}%
\end{pgfscope}%
\begin{pgfscope}%
\pgfpathrectangle{\pgfqpoint{0.120000in}{0.120000in}}{\pgfqpoint{2.460000in}{2.460000in}}%
\pgfusepath{clip}%
\pgfsetrectcap%
\pgfsetroundjoin%
\pgfsetlinewidth{1.003750pt}%
\definecolor{currentstroke}{rgb}{0.600000,0.109804,0.050980}%
\pgfsetstrokecolor{currentstroke}%
\pgfsetdash{}{0pt}%
\pgfpathmoveto{\pgfqpoint{1.524524in}{0.620167in}}%
\pgfpathlineto{\pgfqpoint{1.256694in}{1.071239in}}%
\pgfusepath{stroke}%
\end{pgfscope}%
\begin{pgfscope}%
\pgfpathrectangle{\pgfqpoint{0.120000in}{0.120000in}}{\pgfqpoint{2.460000in}{2.460000in}}%
\pgfusepath{clip}%
\pgfsetbuttcap%
\pgfsetroundjoin%
\pgfsetlinewidth{1.003750pt}%
\definecolor{currentstroke}{rgb}{0.286275,0.505882,0.549020}%
\pgfsetstrokecolor{currentstroke}%
\pgfsetdash{{3.700000pt}{1.600000pt}}{0.000000pt}%
\pgfpathmoveto{\pgfqpoint{1.524524in}{0.620167in}}%
\pgfpathlineto{\pgfqpoint{1.254103in}{2.080747in}}%
\pgfusepath{stroke}%
\end{pgfscope}%
\begin{pgfscope}%
\pgfpathrectangle{\pgfqpoint{0.120000in}{0.120000in}}{\pgfqpoint{2.460000in}{2.460000in}}%
\pgfusepath{clip}%
\pgfsetbuttcap%
\pgfsetroundjoin%
\pgfsetlinewidth{1.003750pt}%
\definecolor{currentstroke}{rgb}{0.286275,0.505882,0.549020}%
\pgfsetstrokecolor{currentstroke}%
\pgfsetdash{{3.700000pt}{1.600000pt}}{0.000000pt}%
\pgfpathmoveto{\pgfqpoint{2.133142in}{0.897374in}}%
\pgfpathlineto{\pgfqpoint{1.256694in}{1.071239in}}%
\pgfusepath{stroke}%
\end{pgfscope}%
\begin{pgfscope}%
\pgfpathrectangle{\pgfqpoint{0.120000in}{0.120000in}}{\pgfqpoint{2.460000in}{2.460000in}}%
\pgfusepath{clip}%
\pgfsetrectcap%
\pgfsetroundjoin%
\pgfsetlinewidth{1.003750pt}%
\definecolor{currentstroke}{rgb}{0.600000,0.109804,0.050980}%
\pgfsetstrokecolor{currentstroke}%
\pgfsetdash{}{0pt}%
\pgfpathmoveto{\pgfqpoint{2.133142in}{0.897374in}}%
\pgfpathlineto{\pgfqpoint{1.254103in}{2.080747in}}%
\pgfusepath{stroke}%
\end{pgfscope}%
\begin{pgfscope}%
\pgfpathrectangle{\pgfqpoint{0.120000in}{0.120000in}}{\pgfqpoint{2.460000in}{2.460000in}}%
\pgfusepath{clip}%
\pgfsetbuttcap%
\pgfsetroundjoin%
\pgfsetlinewidth{1.003750pt}%
\definecolor{currentstroke}{rgb}{0.286275,0.505882,0.549020}%
\pgfsetstrokecolor{currentstroke}%
\pgfsetdash{{3.700000pt}{1.600000pt}}{0.000000pt}%
\pgfpathmoveto{\pgfqpoint{1.256694in}{1.071239in}}%
\pgfpathlineto{\pgfqpoint{1.254103in}{2.080747in}}%
\pgfusepath{stroke}%
\end{pgfscope}%
\end{pgfpicture}%
\makeatother%
\endgroup%

%% file: figures/f_star.pgf
%% Creator: Matplotlib, PGF backend
%%
%% To include the figure in your LaTeX document, write
%%   \input{<filename>.pgf}
%%
%% Make sure the required packages are loaded in your preamble
%%   \usepackage{pgf}
%%
%% and, on pdftex
%%   \usepackage[utf8]{inputenc}\DeclareUnicodeCharacter{2212}{-}
%%
%% or, on luatex and xetex
%%   \usepackage{unicode-math}
%%
%% Figures using additional raster images can only be included by \input if
%% they are in the same directory as the main LaTeX file. For loading figures
%% from other directories you can use the `import` package
%%   \usepackage{import}
%%
%% and then include the figures with
%%   \import{<path to file>}{<filename>.pgf}
%%
%% Matplotlib used the following preamble
%%   \usepackage{fontspec}
%%   \setmainfont{DejaVuSerif.ttf}[Path=/usr/local/lib/python3.7/site-packages/matplotlib/mpl-data/fonts/ttf/]
%%   \setsansfont{DejaVuSans.ttf}[Path=/usr/local/lib/python3.7/site-packages/matplotlib/mpl-data/fonts/ttf/]
%%   \setmonofont{DejaVuSansMono.ttf}[Path=/usr/local/lib/python3.7/site-packages/matplotlib/mpl-data/fonts/ttf/]
%%
\begingroup%
\makeatletter%
\begin{pgfpicture}%
\pgfpathrectangle{\pgfpointorigin}{\pgfqpoint{2.700000in}{2.700000in}}%
\pgfusepath{use as bounding box, clip}%
\begin{pgfscope}%
\pgfsetbuttcap%
\pgfsetmiterjoin%
\definecolor{currentfill}{rgb}{1.000000,1.000000,1.000000}%
\pgfsetfillcolor{currentfill}%
\pgfsetlinewidth{0.000000pt}%
\definecolor{currentstroke}{rgb}{1.000000,1.000000,1.000000}%
\pgfsetstrokecolor{currentstroke}%
\pgfsetdash{}{0pt}%
\pgfpathmoveto{\pgfqpoint{0.000000in}{0.000000in}}%
\pgfpathlineto{\pgfqpoint{2.700000in}{0.000000in}}%
\pgfpathlineto{\pgfqpoint{2.700000in}{2.700000in}}%
\pgfpathlineto{\pgfqpoint{0.000000in}{2.700000in}}%
\pgfpathclose%
\pgfusepath{fill}%
\end{pgfscope}%
\begin{pgfscope}%
\pgfsetbuttcap%
\pgfsetmiterjoin%
\definecolor{currentfill}{rgb}{1.000000,1.000000,1.000000}%
\pgfsetfillcolor{currentfill}%
\pgfsetlinewidth{0.000000pt}%
\definecolor{currentstroke}{rgb}{0.000000,0.000000,0.000000}%
\pgfsetstrokecolor{currentstroke}%
\pgfsetstrokeopacity{0.000000}%
\pgfsetdash{}{0pt}%
\pgfpathmoveto{\pgfqpoint{0.120000in}{0.120000in}}%
\pgfpathlineto{\pgfqpoint{2.580000in}{0.120000in}}%
\pgfpathlineto{\pgfqpoint{2.580000in}{2.580000in}}%
\pgfpathlineto{\pgfqpoint{0.120000in}{2.580000in}}%
\pgfpathclose%
\pgfusepath{fill}%
\end{pgfscope}%
\begin{pgfscope}%
\pgfpathrectangle{\pgfqpoint{0.120000in}{0.120000in}}{\pgfqpoint{2.460000in}{2.460000in}}%
\pgfusepath{clip}%
\pgfsetbuttcap%
\pgfsetroundjoin%
\pgfsetlinewidth{1.003750pt}%
\definecolor{currentstroke}{rgb}{0.286275,0.505882,0.549020}%
\pgfsetstrokecolor{currentstroke}%
\pgfsetdash{{3.700000pt}{1.600000pt}}{0.000000pt}%
\pgfpathmoveto{\pgfqpoint{0.231818in}{1.350000in}}%
\pgfpathlineto{\pgfqpoint{2.468182in}{1.350000in}}%
\pgfusepath{stroke}%
\end{pgfscope}%
\begin{pgfscope}%
\pgfsetbuttcap%
\pgfsetroundjoin%
\definecolor{currentfill}{rgb}{0.000000,0.000000,0.000000}%
\pgfsetfillcolor{currentfill}%
\pgfsetlinewidth{0.803000pt}%
\definecolor{currentstroke}{rgb}{0.000000,0.000000,0.000000}%
\pgfsetstrokecolor{currentstroke}%
\pgfsetdash{}{0pt}%
\pgfsys@defobject{currentmarker}{\pgfqpoint{0.000000in}{-0.048611in}}{\pgfqpoint{0.000000in}{0.000000in}}{%
\pgfpathmoveto{\pgfqpoint{0.000000in}{0.000000in}}%
\pgfpathlineto{\pgfqpoint{0.000000in}{-0.048611in}}%
\pgfusepath{stroke,fill}%
}%
\begin{pgfscope}%
\pgfsys@transformshift{0.231818in}{1.350000in}%
\pgfsys@useobject{currentmarker}{}%
\end{pgfscope}%
\end{pgfscope}%
\begin{pgfscope}%
\definecolor{textcolor}{rgb}{0.000000,0.000000,0.000000}%
\pgfsetstrokecolor{textcolor}%
\pgfsetfillcolor{textcolor}%
\pgftext[x=0.231818in,y=1.252778in,,top]{\color{textcolor}\rmfamily\fontsize{8.000000}{9.600000}\selectfont \(\displaystyle -1\)}%
\end{pgfscope}%
\begin{pgfscope}%
\pgfsetbuttcap%
\pgfsetroundjoin%
\definecolor{currentfill}{rgb}{0.000000,0.000000,0.000000}%
\pgfsetfillcolor{currentfill}%
\pgfsetlinewidth{0.803000pt}%
\definecolor{currentstroke}{rgb}{0.000000,0.000000,0.000000}%
\pgfsetstrokecolor{currentstroke}%
\pgfsetdash{}{0pt}%
\pgfsys@defobject{currentmarker}{\pgfqpoint{0.000000in}{-0.048611in}}{\pgfqpoint{0.000000in}{0.000000in}}{%
\pgfpathmoveto{\pgfqpoint{0.000000in}{0.000000in}}%
\pgfpathlineto{\pgfqpoint{0.000000in}{-0.048611in}}%
\pgfusepath{stroke,fill}%
}%
\begin{pgfscope}%
\pgfsys@transformshift{2.468182in}{1.350000in}%
\pgfsys@useobject{currentmarker}{}%
\end{pgfscope}%
\end{pgfscope}%
\begin{pgfscope}%
\definecolor{textcolor}{rgb}{0.000000,0.000000,0.000000}%
\pgfsetstrokecolor{textcolor}%
\pgfsetfillcolor{textcolor}%
\pgftext[x=2.468182in,y=1.252778in,,top]{\color{textcolor}\rmfamily\fontsize{8.000000}{9.600000}\selectfont \(\displaystyle 1\)}%
\end{pgfscope}%
\begin{pgfscope}%
\pgfpathrectangle{\pgfqpoint{0.120000in}{0.120000in}}{\pgfqpoint{2.460000in}{2.460000in}}%
\pgfusepath{clip}%
\pgfsetbuttcap%
\pgfsetroundjoin%
\definecolor{currentfill}{rgb}{0.600000,0.109804,0.050980}%
\pgfsetfillcolor{currentfill}%
\pgfsetlinewidth{1.003750pt}%
\definecolor{currentstroke}{rgb}{0.600000,0.109804,0.050980}%
\pgfsetstrokecolor{currentstroke}%
\pgfsetdash{}{0pt}%
\pgfsys@defobject{currentmarker}{\pgfqpoint{-0.027778in}{-0.027778in}}{\pgfqpoint{0.027778in}{0.027778in}}{%
\pgfpathmoveto{\pgfqpoint{0.000000in}{-0.027778in}}%
\pgfpathcurveto{\pgfqpoint{0.007367in}{-0.027778in}}{\pgfqpoint{0.014433in}{-0.024851in}}{\pgfqpoint{0.019642in}{-0.019642in}}%
\pgfpathcurveto{\pgfqpoint{0.024851in}{-0.014433in}}{\pgfqpoint{0.027778in}{-0.007367in}}{\pgfqpoint{0.027778in}{0.000000in}}%
\pgfpathcurveto{\pgfqpoint{0.027778in}{0.007367in}}{\pgfqpoint{0.024851in}{0.014433in}}{\pgfqpoint{0.019642in}{0.019642in}}%
\pgfpathcurveto{\pgfqpoint{0.014433in}{0.024851in}}{\pgfqpoint{0.007367in}{0.027778in}}{\pgfqpoint{0.000000in}{0.027778in}}%
\pgfpathcurveto{\pgfqpoint{-0.007367in}{0.027778in}}{\pgfqpoint{-0.014433in}{0.024851in}}{\pgfqpoint{-0.019642in}{0.019642in}}%
\pgfpathcurveto{\pgfqpoint{-0.024851in}{0.014433in}}{\pgfqpoint{-0.027778in}{0.007367in}}{\pgfqpoint{-0.027778in}{0.000000in}}%
\pgfpathcurveto{\pgfqpoint{-0.027778in}{-0.007367in}}{\pgfqpoint{-0.024851in}{-0.014433in}}{\pgfqpoint{-0.019642in}{-0.019642in}}%
\pgfpathcurveto{\pgfqpoint{-0.014433in}{-0.024851in}}{\pgfqpoint{-0.007367in}{-0.027778in}}{\pgfqpoint{0.000000in}{-0.027778in}}%
\pgfpathclose%
\pgfusepath{stroke,fill}%
}%
\begin{pgfscope}%
\pgfsys@transformshift{0.231818in}{1.350000in}%
\pgfsys@useobject{currentmarker}{}%
\end{pgfscope}%
\begin{pgfscope}%
\pgfsys@transformshift{2.468182in}{1.350000in}%
\pgfsys@useobject{currentmarker}{}%
\end{pgfscope}%
\end{pgfscope}%
\end{pgfpicture}%
\makeatother%
\endgroup%

%% file: figures/phi.pgf
%% Creator: Matplotlib, PGF backend
%%
%% To include the figure in your LaTeX document, write
%%   \input{<filename>.pgf}
%%
%% Make sure the required packages are loaded in your preamble
%%   \usepackage{pgf}
%%
%% and, on pdftex
%%   \usepackage[utf8]{inputenc}\DeclareUnicodeCharacter{2212}{-}
%%
%% or, on luatex and xetex
%%   \usepackage{unicode-math}
%%
%% Figures using additional raster images can only be included by \input if
%% they are in the same directory as the main LaTeX file. For loading figures
%% from other directories you can use the `import` package
%%   \usepackage{import}
%%
%% and then include the figures with
%%   \import{<path to file>}{<filename>.pgf}
%%
%% Matplotlib used the following preamble
%%   \usepackage{fontspec}
%%   \setmainfont{DejaVuSerif.ttf}[Path=/usr/local/lib/python3.7/site-packages/matplotlib/mpl-data/fonts/ttf/]
%%   \setsansfont{DejaVuSans.ttf}[Path=/usr/local/lib/python3.7/site-packages/matplotlib/mpl-data/fonts/ttf/]
%%   \setmonofont{DejaVuSansMono.ttf}[Path=/usr/local/lib/python3.7/site-packages/matplotlib/mpl-data/fonts/ttf/]
%%
\begingroup%
\makeatletter%
\begin{pgfpicture}%
\pgfpathrectangle{\pgfpointorigin}{\pgfqpoint{2.700000in}{2.700000in}}%
\pgfusepath{use as bounding box, clip}%
\begin{pgfscope}%
\pgfsetbuttcap%
\pgfsetmiterjoin%
\definecolor{currentfill}{rgb}{1.000000,1.000000,1.000000}%
\pgfsetfillcolor{currentfill}%
\pgfsetlinewidth{0.000000pt}%
\definecolor{currentstroke}{rgb}{1.000000,1.000000,1.000000}%
\pgfsetstrokecolor{currentstroke}%
\pgfsetdash{}{0pt}%
\pgfpathmoveto{\pgfqpoint{0.000000in}{0.000000in}}%
\pgfpathlineto{\pgfqpoint{2.700000in}{0.000000in}}%
\pgfpathlineto{\pgfqpoint{2.700000in}{2.700000in}}%
\pgfpathlineto{\pgfqpoint{0.000000in}{2.700000in}}%
\pgfpathclose%
\pgfusepath{fill}%
\end{pgfscope}%
\begin{pgfscope}%
\pgfsetbuttcap%
\pgfsetmiterjoin%
\definecolor{currentfill}{rgb}{1.000000,1.000000,1.000000}%
\pgfsetfillcolor{currentfill}%
\pgfsetlinewidth{0.000000pt}%
\definecolor{currentstroke}{rgb}{0.000000,0.000000,0.000000}%
\pgfsetstrokecolor{currentstroke}%
\pgfsetstrokeopacity{0.000000}%
\pgfsetdash{}{0pt}%
\pgfpathmoveto{\pgfqpoint{0.427646in}{0.426296in}}%
\pgfpathlineto{\pgfqpoint{2.506224in}{0.426296in}}%
\pgfpathlineto{\pgfqpoint{2.506224in}{2.580000in}}%
\pgfpathlineto{\pgfqpoint{0.427646in}{2.580000in}}%
\pgfpathclose%
\pgfusepath{fill}%
\end{pgfscope}%
\begin{pgfscope}%
\pgfsetbuttcap%
\pgfsetroundjoin%
\definecolor{currentfill}{rgb}{0.000000,0.000000,0.000000}%
\pgfsetfillcolor{currentfill}%
\pgfsetlinewidth{0.803000pt}%
\definecolor{currentstroke}{rgb}{0.000000,0.000000,0.000000}%
\pgfsetstrokecolor{currentstroke}%
\pgfsetdash{}{0pt}%
\pgfsys@defobject{currentmarker}{\pgfqpoint{0.000000in}{-0.048611in}}{\pgfqpoint{0.000000in}{0.000000in}}{%
\pgfpathmoveto{\pgfqpoint{0.000000in}{0.000000in}}%
\pgfpathlineto{\pgfqpoint{0.000000in}{-0.048611in}}%
\pgfusepath{stroke,fill}%
}%
\begin{pgfscope}%
\pgfsys@transformshift{0.458214in}{0.426296in}%
\pgfsys@useobject{currentmarker}{}%
\end{pgfscope}%
\end{pgfscope}%
\begin{pgfscope}%
\definecolor{textcolor}{rgb}{0.000000,0.000000,0.000000}%
\pgfsetstrokecolor{textcolor}%
\pgfsetfillcolor{textcolor}%
\pgftext[x=0.458214in,y=0.329074in,,top]{\color{textcolor}\rmfamily\fontsize{8.000000}{9.600000}\selectfont \(\displaystyle -0.33\)}%
\end{pgfscope}%
\begin{pgfscope}%
\pgfsetbuttcap%
\pgfsetroundjoin%
\definecolor{currentfill}{rgb}{0.000000,0.000000,0.000000}%
\pgfsetfillcolor{currentfill}%
\pgfsetlinewidth{0.803000pt}%
\definecolor{currentstroke}{rgb}{0.000000,0.000000,0.000000}%
\pgfsetstrokecolor{currentstroke}%
\pgfsetdash{}{0pt}%
\pgfsys@defobject{currentmarker}{\pgfqpoint{0.000000in}{-0.048611in}}{\pgfqpoint{0.000000in}{0.000000in}}{%
\pgfpathmoveto{\pgfqpoint{0.000000in}{0.000000in}}%
\pgfpathlineto{\pgfqpoint{0.000000in}{-0.048611in}}%
\pgfusepath{stroke,fill}%
}%
\begin{pgfscope}%
\pgfsys@transformshift{2.475657in}{0.426296in}%
\pgfsys@useobject{currentmarker}{}%
\end{pgfscope}%
\end{pgfscope}%
\begin{pgfscope}%
\definecolor{textcolor}{rgb}{0.000000,0.000000,0.000000}%
\pgfsetstrokecolor{textcolor}%
\pgfsetfillcolor{textcolor}%
\pgftext[x=2.475657in,y=0.329074in,,top]{\color{textcolor}\rmfamily\fontsize{8.000000}{9.600000}\selectfont \(\displaystyle 0.33\)}%
\end{pgfscope}%
\begin{pgfscope}%
\definecolor{textcolor}{rgb}{0.000000,0.000000,0.000000}%
\pgfsetstrokecolor{textcolor}%
\pgfsetfillcolor{textcolor}%
\pgftext[x=1.466935in,y=0.210926in,,top]{\color{textcolor}\rmfamily\fontsize{8.000000}{9.600000}\selectfont \(\displaystyle \widetilde{x}\)}%
\end{pgfscope}%
\begin{pgfscope}%
\pgfsetbuttcap%
\pgfsetroundjoin%
\definecolor{currentfill}{rgb}{0.000000,0.000000,0.000000}%
\pgfsetfillcolor{currentfill}%
\pgfsetlinewidth{0.803000pt}%
\definecolor{currentstroke}{rgb}{0.000000,0.000000,0.000000}%
\pgfsetstrokecolor{currentstroke}%
\pgfsetdash{}{0pt}%
\pgfsys@defobject{currentmarker}{\pgfqpoint{-0.048611in}{0.000000in}}{\pgfqpoint{0.000000in}{0.000000in}}{%
\pgfpathmoveto{\pgfqpoint{0.000000in}{0.000000in}}%
\pgfpathlineto{\pgfqpoint{-0.048611in}{0.000000in}}%
\pgfusepath{stroke,fill}%
}%
\begin{pgfscope}%
\pgfsys@transformshift{0.427646in}{0.426296in}%
\pgfsys@useobject{currentmarker}{}%
\end{pgfscope}%
\end{pgfscope}%
\begin{pgfscope}%
\definecolor{textcolor}{rgb}{0.000000,0.000000,0.000000}%
\pgfsetstrokecolor{textcolor}%
\pgfsetfillcolor{textcolor}%
\pgftext[x=0.179573in, y=0.384087in, left, base]{\color{textcolor}\rmfamily\fontsize{8.000000}{9.600000}\selectfont \(\displaystyle 0.0\)}%
\end{pgfscope}%
\begin{pgfscope}%
\pgfsetbuttcap%
\pgfsetroundjoin%
\definecolor{currentfill}{rgb}{0.000000,0.000000,0.000000}%
\pgfsetfillcolor{currentfill}%
\pgfsetlinewidth{0.803000pt}%
\definecolor{currentstroke}{rgb}{0.000000,0.000000,0.000000}%
\pgfsetstrokecolor{currentstroke}%
\pgfsetdash{}{0pt}%
\pgfsys@defobject{currentmarker}{\pgfqpoint{-0.048611in}{0.000000in}}{\pgfqpoint{0.000000in}{0.000000in}}{%
\pgfpathmoveto{\pgfqpoint{0.000000in}{0.000000in}}%
\pgfpathlineto{\pgfqpoint{-0.048611in}{0.000000in}}%
\pgfusepath{stroke,fill}%
}%
\begin{pgfscope}%
\pgfsys@transformshift{0.427646in}{2.126589in}%
\pgfsys@useobject{currentmarker}{}%
\end{pgfscope}%
\end{pgfscope}%
\begin{pgfscope}%
\definecolor{textcolor}{rgb}{0.000000,0.000000,0.000000}%
\pgfsetstrokecolor{textcolor}%
\pgfsetfillcolor{textcolor}%
\pgftext[x=0.179573in, y=2.084379in, left, base]{\color{textcolor}\rmfamily\fontsize{8.000000}{9.600000}\selectfont \(\displaystyle 0.5\)}%
\end{pgfscope}%
\begin{pgfscope}%
\definecolor{textcolor}{rgb}{0.000000,0.000000,0.000000}%
\pgfsetstrokecolor{textcolor}%
\pgfsetfillcolor{textcolor}%
\pgftext[x=0.219789in,y=1.503148in,,bottom,rotate=90.000000]{\color{textcolor}\rmfamily\fontsize{8.000000}{9.600000}\selectfont \(\displaystyle \widetilde{y}\)}%
\end{pgfscope}%
\begin{pgfscope}%
\pgfpathrectangle{\pgfqpoint{0.427646in}{0.426296in}}{\pgfqpoint{2.078578in}{2.153704in}}%
\pgfusepath{clip}%
\pgfsetrectcap%
\pgfsetroundjoin%
\pgfsetlinewidth{1.003750pt}%
\definecolor{currentstroke}{rgb}{0.000000,0.000000,0.000000}%
\pgfsetstrokecolor{currentstroke}%
\pgfsetdash{}{0pt}%
\pgfpathmoveto{\pgfqpoint{2.475657in}{0.426296in}}%
\pgfpathlineto{\pgfqpoint{2.474526in}{0.479424in}}%
\pgfpathlineto{\pgfqpoint{2.471135in}{0.532433in}}%
\pgfpathlineto{\pgfqpoint{2.465492in}{0.585204in}}%
\pgfpathlineto{\pgfqpoint{2.457610in}{0.637619in}}%
\pgfpathlineto{\pgfqpoint{2.447506in}{0.689559in}}%
\pgfpathlineto{\pgfqpoint{2.435204in}{0.740909in}}%
\pgfpathlineto{\pgfqpoint{2.420729in}{0.791554in}}%
\pgfpathlineto{\pgfqpoint{2.404116in}{0.841379in}}%
\pgfpathlineto{\pgfqpoint{2.385401in}{0.890274in}}%
\pgfpathlineto{\pgfqpoint{2.364626in}{0.938128in}}%
\pgfpathlineto{\pgfqpoint{2.341838in}{0.984834in}}%
\pgfpathlineto{\pgfqpoint{2.317088in}{1.030288in}}%
\pgfpathlineto{\pgfqpoint{2.290431in}{1.074387in}}%
\pgfpathlineto{\pgfqpoint{2.261928in}{1.117032in}}%
\pgfpathlineto{\pgfqpoint{2.231642in}{1.158129in}}%
\pgfpathlineto{\pgfqpoint{2.199640in}{1.197584in}}%
\pgfpathlineto{\pgfqpoint{2.165996in}{1.235310in}}%
\pgfpathlineto{\pgfqpoint{2.130784in}{1.271221in}}%
\pgfpathlineto{\pgfqpoint{2.094083in}{1.305238in}}%
\pgfpathlineto{\pgfqpoint{2.055975in}{1.337283in}}%
\pgfpathlineto{\pgfqpoint{2.016547in}{1.367286in}}%
\pgfpathlineto{\pgfqpoint{1.975886in}{1.395178in}}%
\pgfpathlineto{\pgfqpoint{1.934084in}{1.420897in}}%
\pgfpathlineto{\pgfqpoint{1.891234in}{1.444386in}}%
\pgfpathlineto{\pgfqpoint{1.847433in}{1.465592in}}%
\pgfpathlineto{\pgfqpoint{1.802778in}{1.484467in}}%
\pgfpathlineto{\pgfqpoint{1.757370in}{1.500968in}}%
\pgfpathlineto{\pgfqpoint{1.711311in}{1.515060in}}%
\pgfpathlineto{\pgfqpoint{1.664703in}{1.526710in}}%
\pgfpathlineto{\pgfqpoint{1.617652in}{1.535892in}}%
\pgfpathlineto{\pgfqpoint{1.570264in}{1.542586in}}%
\pgfpathlineto{\pgfqpoint{1.522643in}{1.546777in}}%
\pgfpathlineto{\pgfqpoint{1.474898in}{1.548454in}}%
\pgfpathlineto{\pgfqpoint{1.427134in}{1.547615in}}%
\pgfpathlineto{\pgfqpoint{1.379460in}{1.544262in}}%
\pgfpathlineto{\pgfqpoint{1.331982in}{1.538401in}}%
\pgfpathlineto{\pgfqpoint{1.284807in}{1.530046in}}%
\pgfpathlineto{\pgfqpoint{1.238040in}{1.519216in}}%
\pgfpathlineto{\pgfqpoint{1.191787in}{1.505935in}}%
\pgfpathlineto{\pgfqpoint{1.146150in}{1.490233in}}%
\pgfpathlineto{\pgfqpoint{1.101233in}{1.472144in}}%
\pgfpathlineto{\pgfqpoint{1.057136in}{1.451711in}}%
\pgfpathlineto{\pgfqpoint{1.013958in}{1.428977in}}%
\pgfpathlineto{\pgfqpoint{0.971796in}{1.403995in}}%
\pgfpathlineto{\pgfqpoint{0.930745in}{1.376821in}}%
\pgfpathlineto{\pgfqpoint{0.890896in}{1.347515in}}%
\pgfpathlineto{\pgfqpoint{0.852338in}{1.316142in}}%
\pgfpathlineto{\pgfqpoint{0.815159in}{1.282775in}}%
\pgfpathlineto{\pgfqpoint{0.779442in}{1.247486in}}%
\pgfpathlineto{\pgfqpoint{0.745266in}{1.210356in}}%
\pgfpathlineto{\pgfqpoint{0.712709in}{1.171468in}}%
\pgfpathlineto{\pgfqpoint{0.681844in}{1.130908in}}%
\pgfpathlineto{\pgfqpoint{0.652739in}{1.088768in}}%
\pgfpathlineto{\pgfqpoint{0.625460in}{1.045143in}}%
\pgfpathlineto{\pgfqpoint{0.600068in}{1.000129in}}%
\pgfpathlineto{\pgfqpoint{0.576620in}{0.953829in}}%
\pgfpathlineto{\pgfqpoint{0.555168in}{0.906346in}}%
\pgfpathlineto{\pgfqpoint{0.535762in}{0.857786in}}%
\pgfpathlineto{\pgfqpoint{0.518443in}{0.808259in}}%
\pgfpathlineto{\pgfqpoint{0.503252in}{0.757875in}}%
\pgfpathlineto{\pgfqpoint{0.490222in}{0.706747in}}%
\pgfpathlineto{\pgfqpoint{0.479383in}{0.654990in}}%
\pgfpathlineto{\pgfqpoint{0.470758in}{0.602721in}}%
\pgfpathlineto{\pgfqpoint{0.464367in}{0.550056in}}%
\pgfpathlineto{\pgfqpoint{0.460224in}{0.497113in}}%
\pgfpathlineto{\pgfqpoint{0.458340in}{0.444012in}}%
\pgfpathlineto{\pgfqpoint{0.458214in}{0.426296in}}%
\pgfpathlineto{\pgfqpoint{0.458214in}{0.426296in}}%
\pgfusepath{stroke}%
\end{pgfscope}%
\begin{pgfscope}%
\pgfsetrectcap%
\pgfsetmiterjoin%
\pgfsetlinewidth{0.803000pt}%
\definecolor{currentstroke}{rgb}{0.000000,0.000000,0.000000}%
\pgfsetstrokecolor{currentstroke}%
\pgfsetdash{}{0pt}%
\pgfpathmoveto{\pgfqpoint{0.427646in}{0.426296in}}%
\pgfpathlineto{\pgfqpoint{0.427646in}{2.580000in}}%
\pgfusepath{stroke}%
\end{pgfscope}%
\begin{pgfscope}%
\pgfsetrectcap%
\pgfsetmiterjoin%
\pgfsetlinewidth{0.803000pt}%
\definecolor{currentstroke}{rgb}{0.000000,0.000000,0.000000}%
\pgfsetstrokecolor{currentstroke}%
\pgfsetdash{}{0pt}%
\pgfpathmoveto{\pgfqpoint{0.427646in}{0.426296in}}%
\pgfpathlineto{\pgfqpoint{2.506224in}{0.426296in}}%
\pgfusepath{stroke}%
\end{pgfscope}%
\end{pgfpicture}%
\makeatother%
\endgroup%

%% file: figures/delta_3_under_Phi.pgf
%% Creator: Matplotlib, PGF backend
%%
%% To include the figure in your LaTeX document, write
%%   \input{<filename>.pgf}
%%
%% Make sure the required packages are loaded in your preamble
%%   \usepackage{pgf}
%%
%% and, on pdftex
%%   \usepackage[utf8]{inputenc}\DeclareUnicodeCharacter{2212}{-}
%%
%% or, on luatex and xetex
%%   \usepackage{unicode-math}
%%
%% Figures using additional raster images can only be included by \input if
%% they are in the same directory as the main LaTeX file. For loading figures
%% from other directories you can use the `import` package
%%   \usepackage{import}
%%
%% and then include the figures with
%%   \import{<path to file>}{<filename>.pgf}
%%
%% Matplotlib used the following preamble
%%   \usepackage{fontspec}
%%   \setmainfont{DejaVuSerif.ttf}[Path=/usr/local/lib/python3.7/site-packages/matplotlib/mpl-data/fonts/ttf/]
%%   \setsansfont{DejaVuSans.ttf}[Path=/usr/local/lib/python3.7/site-packages/matplotlib/mpl-data/fonts/ttf/]
%%   \setmonofont{DejaVuSansMono.ttf}[Path=/usr/local/lib/python3.7/site-packages/matplotlib/mpl-data/fonts/ttf/]
%%
\begingroup%
\makeatletter%
\begin{pgfpicture}%
\pgfpathrectangle{\pgfpointorigin}{\pgfqpoint{2.700000in}{2.700000in}}%
\pgfusepath{use as bounding box, clip}%
\begin{pgfscope}%
\pgfsetbuttcap%
\pgfsetmiterjoin%
\definecolor{currentfill}{rgb}{1.000000,1.000000,1.000000}%
\pgfsetfillcolor{currentfill}%
\pgfsetlinewidth{0.000000pt}%
\definecolor{currentstroke}{rgb}{1.000000,1.000000,1.000000}%
\pgfsetstrokecolor{currentstroke}%
\pgfsetdash{}{0pt}%
\pgfpathmoveto{\pgfqpoint{0.000000in}{0.000000in}}%
\pgfpathlineto{\pgfqpoint{2.700000in}{0.000000in}}%
\pgfpathlineto{\pgfqpoint{2.700000in}{2.700000in}}%
\pgfpathlineto{\pgfqpoint{0.000000in}{2.700000in}}%
\pgfpathclose%
\pgfusepath{fill}%
\end{pgfscope}%
\begin{pgfscope}%
\pgfsetbuttcap%
\pgfsetmiterjoin%
\definecolor{currentfill}{rgb}{1.000000,1.000000,1.000000}%
\pgfsetfillcolor{currentfill}%
\pgfsetlinewidth{0.000000pt}%
\definecolor{currentstroke}{rgb}{0.000000,0.000000,0.000000}%
\pgfsetstrokecolor{currentstroke}%
\pgfsetstrokeopacity{0.000000}%
\pgfsetdash{}{0pt}%
\pgfpathmoveto{\pgfqpoint{0.427646in}{0.426296in}}%
\pgfpathlineto{\pgfqpoint{2.560955in}{0.426296in}}%
\pgfpathlineto{\pgfqpoint{2.560955in}{2.580000in}}%
\pgfpathlineto{\pgfqpoint{0.427646in}{2.580000in}}%
\pgfpathclose%
\pgfusepath{fill}%
\end{pgfscope}%
\begin{pgfscope}%
\pgfsetbuttcap%
\pgfsetroundjoin%
\definecolor{currentfill}{rgb}{0.000000,0.000000,0.000000}%
\pgfsetfillcolor{currentfill}%
\pgfsetlinewidth{0.803000pt}%
\definecolor{currentstroke}{rgb}{0.000000,0.000000,0.000000}%
\pgfsetstrokecolor{currentstroke}%
\pgfsetdash{}{0pt}%
\pgfsys@defobject{currentmarker}{\pgfqpoint{0.000000in}{-0.048611in}}{\pgfqpoint{0.000000in}{0.000000in}}{%
\pgfpathmoveto{\pgfqpoint{0.000000in}{0.000000in}}%
\pgfpathlineto{\pgfqpoint{0.000000in}{-0.048611in}}%
\pgfusepath{stroke,fill}%
}%
\begin{pgfscope}%
\pgfsys@transformshift{0.438207in}{0.426296in}%
\pgfsys@useobject{currentmarker}{}%
\end{pgfscope}%
\end{pgfscope}%
\begin{pgfscope}%
\definecolor{textcolor}{rgb}{0.000000,0.000000,0.000000}%
\pgfsetstrokecolor{textcolor}%
\pgfsetfillcolor{textcolor}%
\pgftext[x=0.438207in,y=0.329074in,,top]{\color{textcolor}\rmfamily\fontsize{8.000000}{9.600000}\selectfont \(\displaystyle -1\)}%
\end{pgfscope}%
\begin{pgfscope}%
\pgfsetbuttcap%
\pgfsetroundjoin%
\definecolor{currentfill}{rgb}{0.000000,0.000000,0.000000}%
\pgfsetfillcolor{currentfill}%
\pgfsetlinewidth{0.803000pt}%
\definecolor{currentstroke}{rgb}{0.000000,0.000000,0.000000}%
\pgfsetstrokecolor{currentstroke}%
\pgfsetdash{}{0pt}%
\pgfsys@defobject{currentmarker}{\pgfqpoint{0.000000in}{-0.048611in}}{\pgfqpoint{0.000000in}{0.000000in}}{%
\pgfpathmoveto{\pgfqpoint{0.000000in}{0.000000in}}%
\pgfpathlineto{\pgfqpoint{0.000000in}{-0.048611in}}%
\pgfusepath{stroke,fill}%
}%
\begin{pgfscope}%
\pgfsys@transformshift{2.550394in}{0.426296in}%
\pgfsys@useobject{currentmarker}{}%
\end{pgfscope}%
\end{pgfscope}%
\begin{pgfscope}%
\definecolor{textcolor}{rgb}{0.000000,0.000000,0.000000}%
\pgfsetstrokecolor{textcolor}%
\pgfsetfillcolor{textcolor}%
\pgftext[x=2.550394in,y=0.329074in,,top]{\color{textcolor}\rmfamily\fontsize{8.000000}{9.600000}\selectfont \(\displaystyle 1\)}%
\end{pgfscope}%
\begin{pgfscope}%
\definecolor{textcolor}{rgb}{0.000000,0.000000,0.000000}%
\pgfsetstrokecolor{textcolor}%
\pgfsetfillcolor{textcolor}%
\pgftext[x=1.494301in,y=0.210926in,,top]{\color{textcolor}\rmfamily\fontsize{8.000000}{9.600000}\selectfont \(\displaystyle \widetilde{x}\)}%
\end{pgfscope}%
\begin{pgfscope}%
\pgfsetbuttcap%
\pgfsetroundjoin%
\definecolor{currentfill}{rgb}{0.000000,0.000000,0.000000}%
\pgfsetfillcolor{currentfill}%
\pgfsetlinewidth{0.803000pt}%
\definecolor{currentstroke}{rgb}{0.000000,0.000000,0.000000}%
\pgfsetstrokecolor{currentstroke}%
\pgfsetdash{}{0pt}%
\pgfsys@defobject{currentmarker}{\pgfqpoint{-0.048611in}{0.000000in}}{\pgfqpoint{0.000000in}{0.000000in}}{%
\pgfpathmoveto{\pgfqpoint{0.000000in}{0.000000in}}%
\pgfpathlineto{\pgfqpoint{-0.048611in}{0.000000in}}%
\pgfusepath{stroke,fill}%
}%
\begin{pgfscope}%
\pgfsys@transformshift{0.427646in}{0.426296in}%
\pgfsys@useobject{currentmarker}{}%
\end{pgfscope}%
\end{pgfscope}%
\begin{pgfscope}%
\definecolor{textcolor}{rgb}{0.000000,0.000000,0.000000}%
\pgfsetstrokecolor{textcolor}%
\pgfsetfillcolor{textcolor}%
\pgftext[x=0.179573in, y=0.384087in, left, base]{\color{textcolor}\rmfamily\fontsize{8.000000}{9.600000}\selectfont \(\displaystyle 0.0\)}%
\end{pgfscope}%
\begin{pgfscope}%
\pgfsetbuttcap%
\pgfsetroundjoin%
\definecolor{currentfill}{rgb}{0.000000,0.000000,0.000000}%
\pgfsetfillcolor{currentfill}%
\pgfsetlinewidth{0.803000pt}%
\definecolor{currentstroke}{rgb}{0.000000,0.000000,0.000000}%
\pgfsetstrokecolor{currentstroke}%
\pgfsetdash{}{0pt}%
\pgfsys@defobject{currentmarker}{\pgfqpoint{-0.048611in}{0.000000in}}{\pgfqpoint{0.000000in}{0.000000in}}{%
\pgfpathmoveto{\pgfqpoint{0.000000in}{0.000000in}}%
\pgfpathlineto{\pgfqpoint{-0.048611in}{0.000000in}}%
\pgfusepath{stroke,fill}%
}%
\begin{pgfscope}%
\pgfsys@transformshift{0.427646in}{2.126589in}%
\pgfsys@useobject{currentmarker}{}%
\end{pgfscope}%
\end{pgfscope}%
\begin{pgfscope}%
\definecolor{textcolor}{rgb}{0.000000,0.000000,0.000000}%
\pgfsetstrokecolor{textcolor}%
\pgfsetfillcolor{textcolor}%
\pgftext[x=0.179573in, y=2.084379in, left, base]{\color{textcolor}\rmfamily\fontsize{8.000000}{9.600000}\selectfont \(\displaystyle 1.5\)}%
\end{pgfscope}%
\begin{pgfscope}%
\definecolor{textcolor}{rgb}{0.000000,0.000000,0.000000}%
\pgfsetstrokecolor{textcolor}%
\pgfsetfillcolor{textcolor}%
\pgftext[x=0.214316in,y=1.503148in,,bottom,rotate=90.000000]{\color{textcolor}\rmfamily\fontsize{8.000000}{9.600000}\selectfont \(\displaystyle \widetilde{y}\)}%
\end{pgfscope}%
\begin{pgfscope}%
\pgfpathrectangle{\pgfqpoint{0.427646in}{0.426296in}}{\pgfqpoint{2.133308in}{2.153704in}}%
\pgfusepath{clip}%
\pgfsetbuttcap%
\pgfsetroundjoin%
\pgfsetlinewidth{1.003750pt}%
\definecolor{currentstroke}{rgb}{0.286275,0.505882,0.549020}%
\pgfsetstrokecolor{currentstroke}%
\pgfsetdash{{3.700000pt}{1.600000pt}}{0.000000pt}%
\pgfpathmoveto{\pgfqpoint{2.550394in}{0.426296in}}%
\pgfpathlineto{\pgfqpoint{2.549210in}{0.479961in}}%
\pgfpathlineto{\pgfqpoint{2.545660in}{0.533505in}}%
\pgfpathlineto{\pgfqpoint{2.539752in}{0.586809in}}%
\pgfpathlineto{\pgfqpoint{2.531500in}{0.639753in}}%
\pgfpathlineto{\pgfqpoint{2.520921in}{0.692218in}}%
\pgfpathlineto{\pgfqpoint{2.508041in}{0.744087in}}%
\pgfpathlineto{\pgfqpoint{2.492887in}{0.795243in}}%
\pgfpathlineto{\pgfqpoint{2.475493in}{0.845572in}}%
\pgfpathlineto{\pgfqpoint{2.455899in}{0.894960in}}%
\pgfpathlineto{\pgfqpoint{2.434149in}{0.943298in}}%
\pgfpathlineto{\pgfqpoint{2.410290in}{0.990476in}}%
\pgfpathlineto{\pgfqpoint{2.384378in}{1.036388in}}%
\pgfpathlineto{\pgfqpoint{2.356470in}{1.080933in}}%
\pgfpathlineto{\pgfqpoint{2.326628in}{1.124009in}}%
\pgfpathlineto{\pgfqpoint{2.294919in}{1.165521in}}%
\pgfpathlineto{\pgfqpoint{2.261415in}{1.205375in}}%
\pgfpathlineto{\pgfqpoint{2.226190in}{1.243482in}}%
\pgfpathlineto{\pgfqpoint{2.189325in}{1.279756in}}%
\pgfpathlineto{\pgfqpoint{2.150900in}{1.314116in}}%
\pgfpathlineto{\pgfqpoint{2.111003in}{1.346485in}}%
\pgfpathlineto{\pgfqpoint{2.069723in}{1.376791in}}%
\pgfpathlineto{\pgfqpoint{2.027153in}{1.404965in}}%
\pgfpathlineto{\pgfqpoint{1.983387in}{1.430944in}}%
\pgfpathlineto{\pgfqpoint{1.938525in}{1.454670in}}%
\pgfpathlineto{\pgfqpoint{1.892667in}{1.476090in}}%
\pgfpathlineto{\pgfqpoint{1.845915in}{1.495155in}}%
\pgfpathlineto{\pgfqpoint{1.798375in}{1.511824in}}%
\pgfpathlineto{\pgfqpoint{1.750152in}{1.526058in}}%
\pgfpathlineto{\pgfqpoint{1.701356in}{1.537825in}}%
\pgfpathlineto{\pgfqpoint{1.652096in}{1.547100in}}%
\pgfpathlineto{\pgfqpoint{1.602481in}{1.553862in}}%
\pgfpathlineto{\pgfqpoint{1.552625in}{1.558095in}}%
\pgfpathlineto{\pgfqpoint{1.502637in}{1.559789in}}%
\pgfpathlineto{\pgfqpoint{1.452630in}{1.558942in}}%
\pgfpathlineto{\pgfqpoint{1.402717in}{1.555554in}}%
\pgfpathlineto{\pgfqpoint{1.353010in}{1.549634in}}%
\pgfpathlineto{\pgfqpoint{1.303619in}{1.541195in}}%
\pgfpathlineto{\pgfqpoint{1.254656in}{1.530256in}}%
\pgfpathlineto{\pgfqpoint{1.206230in}{1.516841in}}%
\pgfpathlineto{\pgfqpoint{1.158451in}{1.500980in}}%
\pgfpathlineto{\pgfqpoint{1.111424in}{1.482709in}}%
\pgfpathlineto{\pgfqpoint{1.065257in}{1.462068in}}%
\pgfpathlineto{\pgfqpoint{1.020051in}{1.439105in}}%
\pgfpathlineto{\pgfqpoint{0.975909in}{1.413871in}}%
\pgfpathlineto{\pgfqpoint{0.932929in}{1.386422in}}%
\pgfpathlineto{\pgfqpoint{0.891209in}{1.356820in}}%
\pgfpathlineto{\pgfqpoint{0.850841in}{1.325131in}}%
\pgfpathlineto{\pgfqpoint{0.811916in}{1.291426in}}%
\pgfpathlineto{\pgfqpoint{0.774521in}{1.255781in}}%
\pgfpathlineto{\pgfqpoint{0.738741in}{1.218276in}}%
\pgfpathlineto{\pgfqpoint{0.704655in}{1.178994in}}%
\pgfpathlineto{\pgfqpoint{0.672339in}{1.138025in}}%
\pgfpathlineto{\pgfqpoint{0.641868in}{1.095460in}}%
\pgfpathlineto{\pgfqpoint{0.613307in}{1.051394in}}%
\pgfpathlineto{\pgfqpoint{0.586723in}{1.005926in}}%
\pgfpathlineto{\pgfqpoint{0.562174in}{0.959158in}}%
\pgfpathlineto{\pgfqpoint{0.539715in}{0.911195in}}%
\pgfpathlineto{\pgfqpoint{0.519397in}{0.862145in}}%
\pgfpathlineto{\pgfqpoint{0.501266in}{0.812117in}}%
\pgfpathlineto{\pgfqpoint{0.485361in}{0.761224in}}%
\pgfpathlineto{\pgfqpoint{0.471719in}{0.709580in}}%
\pgfpathlineto{\pgfqpoint{0.460370in}{0.657300in}}%
\pgfpathlineto{\pgfqpoint{0.451340in}{0.604503in}}%
\pgfpathlineto{\pgfqpoint{0.444649in}{0.551306in}}%
\pgfpathlineto{\pgfqpoint{0.440312in}{0.497829in}}%
\pgfpathlineto{\pgfqpoint{0.438339in}{0.444191in}}%
\pgfpathlineto{\pgfqpoint{0.438207in}{0.426296in}}%
\pgfpathlineto{\pgfqpoint{0.438207in}{0.426296in}}%
\pgfusepath{stroke}%
\end{pgfscope}%
\begin{pgfscope}%
\pgfpathrectangle{\pgfqpoint{0.427646in}{0.426296in}}{\pgfqpoint{2.133308in}{2.153704in}}%
\pgfusepath{clip}%
\pgfsetrectcap%
\pgfsetroundjoin%
\pgfsetlinewidth{1.003750pt}%
\definecolor{currentstroke}{rgb}{0.600000,0.109804,0.050980}%
\pgfsetstrokecolor{currentstroke}%
\pgfsetdash{}{0pt}%
\pgfpathmoveto{\pgfqpoint{2.550394in}{0.426296in}}%
\pgfpathlineto{\pgfqpoint{2.548991in}{0.473984in}}%
\pgfpathlineto{\pgfqpoint{2.544786in}{0.521482in}}%
\pgfpathlineto{\pgfqpoint{2.537798in}{0.568601in}}%
\pgfpathlineto{\pgfqpoint{2.528053in}{0.615152in}}%
\pgfpathlineto{\pgfqpoint{2.515590in}{0.660950in}}%
\pgfpathlineto{\pgfqpoint{2.500460in}{0.705813in}}%
\pgfpathlineto{\pgfqpoint{2.482722in}{0.749562in}}%
\pgfpathlineto{\pgfqpoint{2.462447in}{0.792022in}}%
\pgfpathlineto{\pgfqpoint{2.439717in}{0.833024in}}%
\pgfpathlineto{\pgfqpoint{2.414621in}{0.872405in}}%
\pgfpathlineto{\pgfqpoint{2.387259in}{0.910008in}}%
\pgfpathlineto{\pgfqpoint{2.357741in}{0.945682in}}%
\pgfpathlineto{\pgfqpoint{2.326185in}{0.979286in}}%
\pgfpathlineto{\pgfqpoint{2.292716in}{1.010686in}}%
\pgfpathlineto{\pgfqpoint{2.257467in}{1.039756in}}%
\pgfpathlineto{\pgfqpoint{2.220579in}{1.066380in}}%
\pgfpathlineto{\pgfqpoint{2.182200in}{1.090453in}}%
\pgfpathlineto{\pgfqpoint{2.152528in}{1.106775in}}%
\pgfpathlineto{\pgfqpoint{2.122170in}{1.121572in}}%
\pgfpathlineto{\pgfqpoint{2.091192in}{1.134808in}}%
\pgfpathlineto{\pgfqpoint{2.059666in}{1.146456in}}%
\pgfpathlineto{\pgfqpoint{2.027662in}{1.156489in}}%
\pgfpathlineto{\pgfqpoint{1.995250in}{1.164885in}}%
\pgfpathlineto{\pgfqpoint{1.962505in}{1.171624in}}%
\pgfpathlineto{\pgfqpoint{1.929499in}{1.176691in}}%
\pgfpathlineto{\pgfqpoint{1.896307in}{1.180076in}}%
\pgfpathlineto{\pgfqpoint{1.863003in}{1.181770in}}%
\pgfpathlineto{\pgfqpoint{1.829661in}{1.181770in}}%
\pgfpathlineto{\pgfqpoint{1.796357in}{1.180076in}}%
\pgfpathlineto{\pgfqpoint{1.763164in}{1.176691in}}%
\pgfpathlineto{\pgfqpoint{1.730158in}{1.171624in}}%
\pgfpathlineto{\pgfqpoint{1.697413in}{1.164885in}}%
\pgfpathlineto{\pgfqpoint{1.665002in}{1.156489in}}%
\pgfpathlineto{\pgfqpoint{1.632997in}{1.146456in}}%
\pgfpathlineto{\pgfqpoint{1.601471in}{1.134808in}}%
\pgfpathlineto{\pgfqpoint{1.570494in}{1.121572in}}%
\pgfpathlineto{\pgfqpoint{1.540136in}{1.106775in}}%
\pgfpathlineto{\pgfqpoint{1.510464in}{1.090453in}}%
\pgfpathlineto{\pgfqpoint{1.472084in}{1.066380in}}%
\pgfpathlineto{\pgfqpoint{1.435197in}{1.039756in}}%
\pgfpathlineto{\pgfqpoint{1.399948in}{1.010686in}}%
\pgfpathlineto{\pgfqpoint{1.366479in}{0.979286in}}%
\pgfpathlineto{\pgfqpoint{1.334922in}{0.945682in}}%
\pgfpathlineto{\pgfqpoint{1.305404in}{0.910008in}}%
\pgfpathlineto{\pgfqpoint{1.278043in}{0.872405in}}%
\pgfpathlineto{\pgfqpoint{1.252947in}{0.833024in}}%
\pgfpathlineto{\pgfqpoint{1.230216in}{0.792022in}}%
\pgfpathlineto{\pgfqpoint{1.209941in}{0.749562in}}%
\pgfpathlineto{\pgfqpoint{1.192203in}{0.705813in}}%
\pgfpathlineto{\pgfqpoint{1.177073in}{0.660950in}}%
\pgfpathlineto{\pgfqpoint{1.164611in}{0.615152in}}%
\pgfpathlineto{\pgfqpoint{1.154866in}{0.568601in}}%
\pgfpathlineto{\pgfqpoint{1.147877in}{0.521482in}}%
\pgfpathlineto{\pgfqpoint{1.143673in}{0.473984in}}%
\pgfpathlineto{\pgfqpoint{1.142270in}{0.426296in}}%
\pgfpathlineto{\pgfqpoint{1.142270in}{0.426296in}}%
\pgfusepath{stroke}%
\end{pgfscope}%
\begin{pgfscope}%
\pgfpathrectangle{\pgfqpoint{0.427646in}{0.426296in}}{\pgfqpoint{2.133308in}{2.153704in}}%
\pgfusepath{clip}%
\pgfsetbuttcap%
\pgfsetroundjoin%
\pgfsetlinewidth{1.003750pt}%
\definecolor{currentstroke}{rgb}{0.286275,0.505882,0.549020}%
\pgfsetstrokecolor{currentstroke}%
\pgfsetdash{{3.700000pt}{1.600000pt}}{0.000000pt}%
\pgfpathmoveto{\pgfqpoint{2.550394in}{0.426296in}}%
\pgfpathlineto{\pgfqpoint{2.549298in}{0.456090in}}%
\pgfpathlineto{\pgfqpoint{2.546016in}{0.485699in}}%
\pgfpathlineto{\pgfqpoint{2.540570in}{0.514937in}}%
\pgfpathlineto{\pgfqpoint{2.532992in}{0.543623in}}%
\pgfpathlineto{\pgfqpoint{2.523331in}{0.571579in}}%
\pgfpathlineto{\pgfqpoint{2.511646in}{0.598630in}}%
\pgfpathlineto{\pgfqpoint{2.498009in}{0.624608in}}%
\pgfpathlineto{\pgfqpoint{2.482507in}{0.649351in}}%
\pgfpathlineto{\pgfqpoint{2.465236in}{0.672705in}}%
\pgfpathlineto{\pgfqpoint{2.446302in}{0.694524in}}%
\pgfpathlineto{\pgfqpoint{2.425824in}{0.714673in}}%
\pgfpathlineto{\pgfqpoint{2.403930in}{0.733026in}}%
\pgfpathlineto{\pgfqpoint{2.380756in}{0.749469in}}%
\pgfpathlineto{\pgfqpoint{2.356446in}{0.763899in}}%
\pgfpathlineto{\pgfqpoint{2.331152in}{0.776228in}}%
\pgfpathlineto{\pgfqpoint{2.305030in}{0.786376in}}%
\pgfpathlineto{\pgfqpoint{2.278244in}{0.794283in}}%
\pgfpathlineto{\pgfqpoint{2.250961in}{0.799898in}}%
\pgfpathlineto{\pgfqpoint{2.223350in}{0.803186in}}%
\pgfpathlineto{\pgfqpoint{2.195584in}{0.804127in}}%
\pgfpathlineto{\pgfqpoint{2.167835in}{0.802716in}}%
\pgfpathlineto{\pgfqpoint{2.140276in}{0.798960in}}%
\pgfpathlineto{\pgfqpoint{2.113079in}{0.792884in}}%
\pgfpathlineto{\pgfqpoint{2.086413in}{0.784524in}}%
\pgfpathlineto{\pgfqpoint{2.060444in}{0.773934in}}%
\pgfpathlineto{\pgfqpoint{2.035334in}{0.761179in}}%
\pgfpathlineto{\pgfqpoint{2.011239in}{0.746338in}}%
\pgfpathlineto{\pgfqpoint{1.988310in}{0.729505in}}%
\pgfpathlineto{\pgfqpoint{1.966688in}{0.710783in}}%
\pgfpathlineto{\pgfqpoint{1.946509in}{0.690290in}}%
\pgfpathlineto{\pgfqpoint{1.927899in}{0.668152in}}%
\pgfpathlineto{\pgfqpoint{1.910973in}{0.644509in}}%
\pgfpathlineto{\pgfqpoint{1.895837in}{0.619506in}}%
\pgfpathlineto{\pgfqpoint{1.882585in}{0.593301in}}%
\pgfpathlineto{\pgfqpoint{1.871299in}{0.566055in}}%
\pgfpathlineto{\pgfqpoint{1.862050in}{0.537939in}}%
\pgfpathlineto{\pgfqpoint{1.854895in}{0.509128in}}%
\pgfpathlineto{\pgfqpoint{1.849879in}{0.479801in}}%
\pgfpathlineto{\pgfqpoint{1.847033in}{0.450140in}}%
\pgfpathlineto{\pgfqpoint{1.846332in}{0.426296in}}%
\pgfpathlineto{\pgfqpoint{1.846332in}{0.426296in}}%
\pgfusepath{stroke}%
\end{pgfscope}%
\begin{pgfscope}%
\pgfpathrectangle{\pgfqpoint{0.427646in}{0.426296in}}{\pgfqpoint{2.133308in}{2.153704in}}%
\pgfusepath{clip}%
\pgfsetrectcap%
\pgfsetroundjoin%
\pgfsetlinewidth{1.003750pt}%
\definecolor{currentstroke}{rgb}{0.600000,0.109804,0.050980}%
\pgfsetstrokecolor{currentstroke}%
\pgfsetdash{}{0pt}%
\pgfpathmoveto{\pgfqpoint{1.846332in}{0.426296in}}%
\pgfpathlineto{\pgfqpoint{1.844928in}{0.473984in}}%
\pgfpathlineto{\pgfqpoint{1.840724in}{0.521482in}}%
\pgfpathlineto{\pgfqpoint{1.833736in}{0.568601in}}%
\pgfpathlineto{\pgfqpoint{1.823991in}{0.615152in}}%
\pgfpathlineto{\pgfqpoint{1.811528in}{0.660950in}}%
\pgfpathlineto{\pgfqpoint{1.796398in}{0.705813in}}%
\pgfpathlineto{\pgfqpoint{1.778660in}{0.749562in}}%
\pgfpathlineto{\pgfqpoint{1.758385in}{0.792022in}}%
\pgfpathlineto{\pgfqpoint{1.735655in}{0.833024in}}%
\pgfpathlineto{\pgfqpoint{1.710558in}{0.872405in}}%
\pgfpathlineto{\pgfqpoint{1.683197in}{0.910008in}}%
\pgfpathlineto{\pgfqpoint{1.653679in}{0.945682in}}%
\pgfpathlineto{\pgfqpoint{1.622123in}{0.979286in}}%
\pgfpathlineto{\pgfqpoint{1.588653in}{1.010686in}}%
\pgfpathlineto{\pgfqpoint{1.553405in}{1.039756in}}%
\pgfpathlineto{\pgfqpoint{1.516517in}{1.066380in}}%
\pgfpathlineto{\pgfqpoint{1.478138in}{1.090453in}}%
\pgfpathlineto{\pgfqpoint{1.448466in}{1.106775in}}%
\pgfpathlineto{\pgfqpoint{1.418107in}{1.121572in}}%
\pgfpathlineto{\pgfqpoint{1.387130in}{1.134808in}}%
\pgfpathlineto{\pgfqpoint{1.355604in}{1.146456in}}%
\pgfpathlineto{\pgfqpoint{1.323599in}{1.156489in}}%
\pgfpathlineto{\pgfqpoint{1.291188in}{1.164885in}}%
\pgfpathlineto{\pgfqpoint{1.258443in}{1.171624in}}%
\pgfpathlineto{\pgfqpoint{1.225437in}{1.176691in}}%
\pgfpathlineto{\pgfqpoint{1.192245in}{1.180076in}}%
\pgfpathlineto{\pgfqpoint{1.158940in}{1.181770in}}%
\pgfpathlineto{\pgfqpoint{1.125599in}{1.181770in}}%
\pgfpathlineto{\pgfqpoint{1.092294in}{1.180076in}}%
\pgfpathlineto{\pgfqpoint{1.059102in}{1.176691in}}%
\pgfpathlineto{\pgfqpoint{1.026096in}{1.171624in}}%
\pgfpathlineto{\pgfqpoint{0.993351in}{1.164885in}}%
\pgfpathlineto{\pgfqpoint{0.960940in}{1.156489in}}%
\pgfpathlineto{\pgfqpoint{0.928935in}{1.146456in}}%
\pgfpathlineto{\pgfqpoint{0.897409in}{1.134808in}}%
\pgfpathlineto{\pgfqpoint{0.866432in}{1.121572in}}%
\pgfpathlineto{\pgfqpoint{0.836073in}{1.106775in}}%
\pgfpathlineto{\pgfqpoint{0.806402in}{1.090453in}}%
\pgfpathlineto{\pgfqpoint{0.768022in}{1.066380in}}%
\pgfpathlineto{\pgfqpoint{0.731135in}{1.039756in}}%
\pgfpathlineto{\pgfqpoint{0.695886in}{1.010686in}}%
\pgfpathlineto{\pgfqpoint{0.662417in}{0.979286in}}%
\pgfpathlineto{\pgfqpoint{0.630860in}{0.945682in}}%
\pgfpathlineto{\pgfqpoint{0.601342in}{0.910008in}}%
\pgfpathlineto{\pgfqpoint{0.573981in}{0.872405in}}%
\pgfpathlineto{\pgfqpoint{0.548885in}{0.833024in}}%
\pgfpathlineto{\pgfqpoint{0.526154in}{0.792022in}}%
\pgfpathlineto{\pgfqpoint{0.505879in}{0.749562in}}%
\pgfpathlineto{\pgfqpoint{0.488141in}{0.705813in}}%
\pgfpathlineto{\pgfqpoint{0.473011in}{0.660950in}}%
\pgfpathlineto{\pgfqpoint{0.460549in}{0.615152in}}%
\pgfpathlineto{\pgfqpoint{0.450804in}{0.568601in}}%
\pgfpathlineto{\pgfqpoint{0.443815in}{0.521482in}}%
\pgfpathlineto{\pgfqpoint{0.439611in}{0.473984in}}%
\pgfpathlineto{\pgfqpoint{0.438207in}{0.426296in}}%
\pgfpathlineto{\pgfqpoint{0.438207in}{0.426296in}}%
\pgfusepath{stroke}%
\end{pgfscope}%
\begin{pgfscope}%
\pgfpathrectangle{\pgfqpoint{0.427646in}{0.426296in}}{\pgfqpoint{2.133308in}{2.153704in}}%
\pgfusepath{clip}%
\pgfsetbuttcap%
\pgfsetroundjoin%
\pgfsetlinewidth{1.003750pt}%
\definecolor{currentstroke}{rgb}{0.286275,0.505882,0.549020}%
\pgfsetstrokecolor{currentstroke}%
\pgfsetdash{{3.700000pt}{1.600000pt}}{0.000000pt}%
\pgfpathmoveto{\pgfqpoint{1.142270in}{0.426296in}}%
\pgfpathlineto{\pgfqpoint{1.141173in}{0.456090in}}%
\pgfpathlineto{\pgfqpoint{1.137892in}{0.485699in}}%
\pgfpathlineto{\pgfqpoint{1.132445in}{0.514937in}}%
\pgfpathlineto{\pgfqpoint{1.124868in}{0.543623in}}%
\pgfpathlineto{\pgfqpoint{1.115206in}{0.571579in}}%
\pgfpathlineto{\pgfqpoint{1.103521in}{0.598630in}}%
\pgfpathlineto{\pgfqpoint{1.089885in}{0.624608in}}%
\pgfpathlineto{\pgfqpoint{1.074383in}{0.649351in}}%
\pgfpathlineto{\pgfqpoint{1.057111in}{0.672705in}}%
\pgfpathlineto{\pgfqpoint{1.038178in}{0.694524in}}%
\pgfpathlineto{\pgfqpoint{1.017700in}{0.714673in}}%
\pgfpathlineto{\pgfqpoint{0.995806in}{0.733026in}}%
\pgfpathlineto{\pgfqpoint{0.972632in}{0.749469in}}%
\pgfpathlineto{\pgfqpoint{0.948322in}{0.763899in}}%
\pgfpathlineto{\pgfqpoint{0.923027in}{0.776228in}}%
\pgfpathlineto{\pgfqpoint{0.896906in}{0.786376in}}%
\pgfpathlineto{\pgfqpoint{0.870120in}{0.794283in}}%
\pgfpathlineto{\pgfqpoint{0.842837in}{0.799898in}}%
\pgfpathlineto{\pgfqpoint{0.815226in}{0.803186in}}%
\pgfpathlineto{\pgfqpoint{0.787460in}{0.804127in}}%
\pgfpathlineto{\pgfqpoint{0.759711in}{0.802716in}}%
\pgfpathlineto{\pgfqpoint{0.732152in}{0.798960in}}%
\pgfpathlineto{\pgfqpoint{0.704955in}{0.792884in}}%
\pgfpathlineto{\pgfqpoint{0.678289in}{0.784524in}}%
\pgfpathlineto{\pgfqpoint{0.652320in}{0.773934in}}%
\pgfpathlineto{\pgfqpoint{0.627210in}{0.761179in}}%
\pgfpathlineto{\pgfqpoint{0.603115in}{0.746338in}}%
\pgfpathlineto{\pgfqpoint{0.580185in}{0.729505in}}%
\pgfpathlineto{\pgfqpoint{0.558564in}{0.710783in}}%
\pgfpathlineto{\pgfqpoint{0.538385in}{0.690290in}}%
\pgfpathlineto{\pgfqpoint{0.519775in}{0.668152in}}%
\pgfpathlineto{\pgfqpoint{0.502849in}{0.644509in}}%
\pgfpathlineto{\pgfqpoint{0.487713in}{0.619506in}}%
\pgfpathlineto{\pgfqpoint{0.474460in}{0.593301in}}%
\pgfpathlineto{\pgfqpoint{0.463174in}{0.566055in}}%
\pgfpathlineto{\pgfqpoint{0.453925in}{0.537939in}}%
\pgfpathlineto{\pgfqpoint{0.446771in}{0.509128in}}%
\pgfpathlineto{\pgfqpoint{0.441755in}{0.479801in}}%
\pgfpathlineto{\pgfqpoint{0.438909in}{0.450140in}}%
\pgfpathlineto{\pgfqpoint{0.438207in}{0.426296in}}%
\pgfpathlineto{\pgfqpoint{0.438207in}{0.426296in}}%
\pgfusepath{stroke}%
\end{pgfscope}%
\begin{pgfscope}%
\pgfpathrectangle{\pgfqpoint{0.427646in}{0.426296in}}{\pgfqpoint{2.133308in}{2.153704in}}%
\pgfusepath{clip}%
\pgfsetbuttcap%
\pgfsetroundjoin%
\pgfsetlinewidth{1.003750pt}%
\definecolor{currentstroke}{rgb}{0.286275,0.505882,0.549020}%
\pgfsetstrokecolor{currentstroke}%
\pgfsetdash{{3.700000pt}{1.600000pt}}{0.000000pt}%
\pgfpathmoveto{\pgfqpoint{1.846332in}{0.426296in}}%
\pgfpathlineto{\pgfqpoint{1.845236in}{0.456090in}}%
\pgfpathlineto{\pgfqpoint{1.841954in}{0.485699in}}%
\pgfpathlineto{\pgfqpoint{1.836508in}{0.514937in}}%
\pgfpathlineto{\pgfqpoint{1.828930in}{0.543623in}}%
\pgfpathlineto{\pgfqpoint{1.819269in}{0.571579in}}%
\pgfpathlineto{\pgfqpoint{1.807583in}{0.598630in}}%
\pgfpathlineto{\pgfqpoint{1.793947in}{0.624608in}}%
\pgfpathlineto{\pgfqpoint{1.778445in}{0.649351in}}%
\pgfpathlineto{\pgfqpoint{1.761173in}{0.672705in}}%
\pgfpathlineto{\pgfqpoint{1.742240in}{0.694524in}}%
\pgfpathlineto{\pgfqpoint{1.721762in}{0.714673in}}%
\pgfpathlineto{\pgfqpoint{1.699868in}{0.733026in}}%
\pgfpathlineto{\pgfqpoint{1.676694in}{0.749469in}}%
\pgfpathlineto{\pgfqpoint{1.652384in}{0.763899in}}%
\pgfpathlineto{\pgfqpoint{1.627089in}{0.776228in}}%
\pgfpathlineto{\pgfqpoint{1.600968in}{0.786376in}}%
\pgfpathlineto{\pgfqpoint{1.574182in}{0.794283in}}%
\pgfpathlineto{\pgfqpoint{1.546899in}{0.799898in}}%
\pgfpathlineto{\pgfqpoint{1.519288in}{0.803186in}}%
\pgfpathlineto{\pgfqpoint{1.491522in}{0.804127in}}%
\pgfpathlineto{\pgfqpoint{1.463773in}{0.802716in}}%
\pgfpathlineto{\pgfqpoint{1.436214in}{0.798960in}}%
\pgfpathlineto{\pgfqpoint{1.409017in}{0.792884in}}%
\pgfpathlineto{\pgfqpoint{1.382351in}{0.784524in}}%
\pgfpathlineto{\pgfqpoint{1.356382in}{0.773934in}}%
\pgfpathlineto{\pgfqpoint{1.331272in}{0.761179in}}%
\pgfpathlineto{\pgfqpoint{1.307177in}{0.746338in}}%
\pgfpathlineto{\pgfqpoint{1.284247in}{0.729505in}}%
\pgfpathlineto{\pgfqpoint{1.262626in}{0.710783in}}%
\pgfpathlineto{\pgfqpoint{1.242447in}{0.690290in}}%
\pgfpathlineto{\pgfqpoint{1.223837in}{0.668152in}}%
\pgfpathlineto{\pgfqpoint{1.206911in}{0.644509in}}%
\pgfpathlineto{\pgfqpoint{1.191775in}{0.619506in}}%
\pgfpathlineto{\pgfqpoint{1.178522in}{0.593301in}}%
\pgfpathlineto{\pgfqpoint{1.167237in}{0.566055in}}%
\pgfpathlineto{\pgfqpoint{1.157987in}{0.537939in}}%
\pgfpathlineto{\pgfqpoint{1.150833in}{0.509128in}}%
\pgfpathlineto{\pgfqpoint{1.145817in}{0.479801in}}%
\pgfpathlineto{\pgfqpoint{1.142971in}{0.450140in}}%
\pgfpathlineto{\pgfqpoint{1.142270in}{0.426296in}}%
\pgfpathlineto{\pgfqpoint{1.142270in}{0.426296in}}%
\pgfusepath{stroke}%
\end{pgfscope}%
\begin{pgfscope}%
\pgfsetrectcap%
\pgfsetmiterjoin%
\pgfsetlinewidth{0.803000pt}%
\definecolor{currentstroke}{rgb}{0.000000,0.000000,0.000000}%
\pgfsetstrokecolor{currentstroke}%
\pgfsetdash{}{0pt}%
\pgfpathmoveto{\pgfqpoint{0.427646in}{0.426296in}}%
\pgfpathlineto{\pgfqpoint{0.427646in}{2.580000in}}%
\pgfusepath{stroke}%
\end{pgfscope}%
\begin{pgfscope}%
\pgfsetrectcap%
\pgfsetmiterjoin%
\pgfsetlinewidth{0.803000pt}%
\definecolor{currentstroke}{rgb}{0.000000,0.000000,0.000000}%
\pgfsetstrokecolor{currentstroke}%
\pgfsetdash{}{0pt}%
\pgfpathmoveto{\pgfqpoint{0.427646in}{0.426296in}}%
\pgfpathlineto{\pgfqpoint{2.560955in}{0.426296in}}%
\pgfusepath{stroke}%
\end{pgfscope}%
\end{pgfpicture}%
\makeatother%
\endgroup%

%% file: figures/nu_cartoon_2d.tikz
\begin{tikzpicture}[x=50pt,y=50pt]

\begin{scope}[pattern color=edred, color=edred]
\draw[fill,ultra thick,pattern=south east lines,fill opacity=0.6]
(-1,-1) -- (-1, 1) -- (1, 1) -- (1, -1) -- cycle;
\node at (0,0) [font=\Huge,color=black] {$I_2$};

\draw[color=black, very thick, -Stealth] (1.2, 0) .. controls (1.6, 0.2) and (2, 0.2) .. (2.4, 0);
\node at (1.8, 0.3) [font=\LARGE,color=black] {$\nu_2$};

\tikzset{shift={(2.6,0,0)}}

\draw[fill,ultra thick,pattern=south east lines,fill opacity=0.6]
(-1,-1) -- (1, 1) -- (1, -1) -- cycle;
\node at (0.4,-0.4) [font=\Huge,color=black] {$\Delta_2$};

\end{scope}

\end{tikzpicture}

%% file: figures/nu_cartoon_3d.tikz
\begin{tikzpicture}[z={(90:45pt)},x={(0:50pt)},y={(50:25pt)},
                    point/.style={circle,draw,fill,inner sep=0.75mm}]

\begin{scope}[canvas is xy plane at z=1,pattern color=edred,color=edred]
\draw[fill,ultra thick,pattern=south east lines,fill opacity=0.6]
(-1,-1) -- (-1, 1) -- (1, 1) -- (1, -1) -- cycle;
\node at (0,0) [font=\LARGE,color=black] {$I_3^+$};
\node at (0.35,0) [font=\normalsize,color=black] {$\cong I_2$};
\node[point] (xnode) at (0.5, 0.5) {};
\node (xlabel) at (0.7, 0.5) [font=\Large,color=black] {$\mathbf{x}$};
\end{scope}

\begin{scope}[canvas is xy plane at z=-1, fill=edblue, color=edblue]
\draw[fill,ultra thick,fill opacity=.2]
(-1,-1) -- (-1, 1) -- (1, 1) -- (1, -1) -- cycle;
\node at (0,0) [font=\LARGE,color=black] {$I_3^-$};
\node at (0.35,0) [font=\normalsize,color=black] {$\cong I_2$};
\node[point] (minxnode) at (-0.5, -0.5) {};
\node (xlabel) at (-0.75, -0.5) [font=\Large,color=black] {$\smn\mathbf{x}$};
\end{scope}

\begin{scope}[canvas is xz plane at y=-1, color=black]
\draw[loosely dashed] (-1,-1) -- (-1, 1) ;
\draw[loosely dashed] (1,-1) -- (1, 1) ;
\end{scope}

\begin{scope}[canvas is xz plane at y=1, color=black]
\draw[loosely dashed] (-1,-1) -- (-1, 1) ;
\draw[loosely dashed] (1,-1) -- (1, 1) ;
\end{scope}

\draw[color=black, very thick, -Stealth] (1.5, 0, 0) .. controls (2, 0, 0.3) and (3, 0, 0.3) .. (3.5, 0, 0);
\node at (2.5, 0, 0.4) [font=\LARGE,color=black] {$\nu_3$};

\tikzset{shift={(4,0,0)}}

\begin{scope}[canvas is xy plane at z=-1, fill=edblue, color=edblue]
\draw[fill,ultra thick,fill opacity=.2]
(-1,-1) -- (1, 1) -- (1, -1) -- cycle;
\end{scope}

\begin{scope}[canvas is yz plane at x=1,pattern color=edred,color=edred]
\draw[fill,ultra thick,pattern=south east lines,fill opacity=0.6]
(-1,-1) -- (1.06, 1) -- (1.06, -1);
\draw[ultra thick, densely dashed]
(1,-1) -- (-1, -1);
\end{scope}

\node at (1.4, 1, 0.5) [font=\LARGE,color=black] {$\Delta_3^{(0)}$};
\node at (1.95, 1, 0.5) [font=\normalsize,color=black] {$\cong \Delta_2$};
\node at (-0.97, -0.9, -0.6) [font=\LARGE,color=black] {$\Delta_3^{(3)}$};
\node at (-0.4, -0.9, -0.6) [font=\normalsize,color=black] {$\cong \Delta_2$};

\node[point,color=edred] at (1, 0.1, -0.5) {};
\node at (1.7, 0.1, -0.5) [font=\large,color=black] {$\nu_3(\mathbf{x})$};
\draw[color=black, thick] (1.05, 0.2, -0.5) .. controls (1.1, 0.3, -0.5) and (1.2, 0.3, -0.5) .. (1.43, 0.1, -0.5);

\node[point,color=edblue] at (0.1, -0.35, -1) {};
\node at (1.7, 0.1, -1) [font=\large,color=black] {$\nu_3(\smn\mathbf{x})$};
\draw[color=black, thick] (0.2, -0.35, -1) .. controls (0.2, 0, -1) and (1, 0, -1) .. (1.36, 0.1, -1);

\draw[color=black, very thick, -Stealth] (0.93, 0.1, -0.5) .. controls (0.7, 0.1, -0.5) and (0.26, -0.5, -0.3) .. (0.16, -0.5, -0.85);
\node at (0.3, 0, -0.46) [font=\large,color=black] {$\alpha$};

\draw[loosely dashed] (-0.94, -0.9, -0.9) -- (-1.04,-1, -1);
\draw[loosely dashed] (-0.59,-0.55, -0.55) -- (1, 1, 1.04);

\end{tikzpicture}

%% file: figures/vertical_cartoon.tikz
\begin{tikzpicture}[z={(90:45pt)},x={(0:50pt)},y={(50:25pt)},
                    point/.style={circle,draw,fill,inner sep=0.75mm}]

\begin{scope}[canvas is xy plane at z=1,color=black]
\draw[ultra thick]
(-1,-1) -- (-1, 1) -- (1, 1) -- (1, -1) -- cycle;
\node[point,color=edred] at (0.7, -1) {};
\node[point,color=edblue] at (-0.7, 1) {};
\end{scope}

\begin{scope}[canvas is xy plane at z=-1,color=black]
\draw[ultra thick]
(-1,-1) -- (-1, 1) -- (1, 1) -- (1, -1) -- cycle;
\node[point,color=edred] at (0.7, -1) {};
\node[point,color=edblue] at (-0.7, 1) {};
\end{scope}

\begin{scope}[canvas is xz plane at y=-1,color=edred]
\draw[thick, dashed]
(0.7, 1) -- (0.7, -1);
\node[point] at (0.7, -0.45) {};
\end{scope}

\begin{scope}[canvas is xz plane at y=1,color=edblue]
\draw[thick, dashed]
(-0.7, 1) -- (-0.7, -1);
\node[point] at (-0.7, 0.45) {};
\end{scope}

\draw[loosely dashed] (1, 1, -1) -- (1, 1, 1);
\draw[loosely dashed] (1, -1, -0.65) -- (1, -1, 1);
\draw[loosely dashed] (1, -1, -0.85) -- (1, -1, -1);
\draw[loosely dashed] (-1, -1, -1) -- (-1, -1, 1);
\draw[loosely dashed] (-1, 1, -1) -- (-1, 1, 1);

\node at (0.92, -1, 1.2) [font=\Large,color=edred] {$\mathbf{x}^+$};
\node at (0.92, -1, -0.8) [font=\Large,color=edred] {$\mathbf{x}^-$};
\node at (0.85, -1, -0.3) [font=\Large,color=edred] {$\mathbf{x}$};

\node at (-0.45, 1, 1.2) [font=\Large,color=edblue] {$\mathbf{y}^+$};
\node at (-0.45, 1, -0.8) [font=\Large,color=edblue] {$\mathbf{y}^-$};
\node at (-0.35, 1, 0.6) [font=\Large,color=edblue] {$\mathbf{y}\scriptstyle{\textcolor{black}{=}\textcolor{edred}{\smn\mathbf{x}}}$};

\node at (0, 0, 0) [font=\huge,color=black] {$I_n$};

\end{tikzpicture}

%% file: figures/x_bar_cartoon.tikz
\begin{tikzpicture}[x=50pt,y=50pt,point/.style={circle,draw,fill,inner sep=0.75mm}]

\draw[ultra thick]
(-1,-1) -- (-1, 1) -- (1, 1) -- (1, -1) -- cycle;

\node[point,color=edred] at (0.7, -1) {};
\node[point,color=edblue] at (-0.7, 1) {};

\node at (0.8, -0.8) [font=\Large,color=edred] {$\overline{\mathbf{x}}$};
\node at (-0.4, 1.2) [font=\Large,color=edblue] {$\overline{\mathbf{y}}\scriptstyle{\textcolor{black}{=}\textcolor{edred}{\smn\overline{\mathbf{x}}}}$};

\node at (0,0) [font=\huge,color=black] {$I_{n-1}$};

\end{tikzpicture}

%% file: 04c_additional.tex
\subsection{Some Visual Intuition}
\label{sec:contours}

\begin{figure}
    \centering
    \input{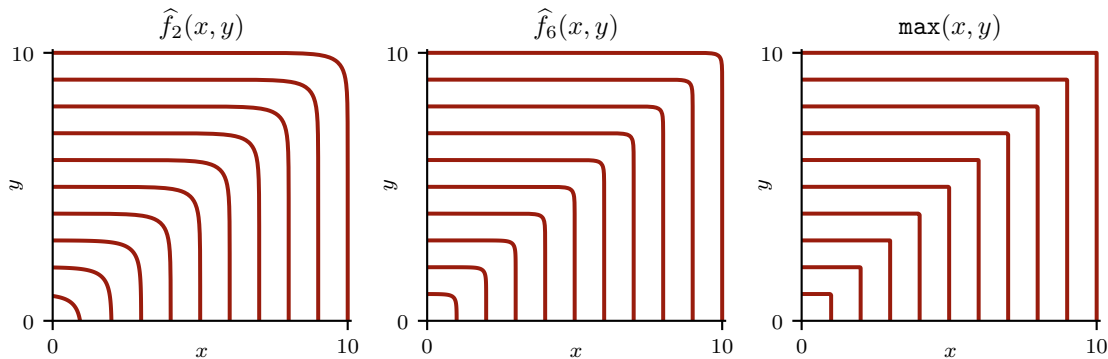}
    \caption[Contours of sum-decomposition approximations to \texttt{max} compared to the true \texttt{max} function.]{Contours of sum-decomposition approximations to \texttt{max} compared to the true \texttt{max} function. The sum-decompositions $\widehat{f}_2$ and $\widehat{f}_6$ are defined by \Cref{eq:max_1d_approx}, with $a=2, 6$ respectively. As $a$ increases, the corners become sharper, but never perfectly sharp.}
    \label{fig:max_contours}
\end{figure}

We have shown mathematically that there exist functions which cannot be approximately sum-decomposed via a low-dimensional latent space. In contrast to the heavy mathematical content above, we now provide some less formal intuition as to why this particular function is so hard to approximate. To that end, we will gloss over some details for the sake of readability.

We consider the contours of the encoding map $\Phi = \sum \phi$. This provides a useful visual perspective on sum-decomposition, helping to develop some intuition for the behaviour of functions represented in this form. Contours of $\Phi$ must also be contours of the full sum-decomposition $\rho \circ \Phi$: if $\Phi(X) = \Phi(Y)$ then $\rho \circ \Phi(X) = \rho \circ \Phi(Y)$.

We plot the contours of \texttt{max} in \Cref{fig:max_contours}, as well as the contours of two terms from an approximate sum-decomposition of \texttt{max} via $\mathbb{R}$. As can be seen here, the contours of \texttt{max} have sharp corners on the line $y=x$, whereas the contours of the approximations have a soft corner here. It is easily seen that a sum-decomposition via $\mathbb{R}$ cannot have the same sharp contours as \texttt{max} -- we can push the contours as close to the corner as desired, providing ever-better approximation, but we cannot make the contours sharp like the true function. This picture of pushing the sum-decomposition contours closer to the true contours gives some intuition for why arbitrarily close approximation can be possible even when exact representation is not.

% \begin{figure}
%     \centering
%     \includegraphics{figures/max3d.pdf}
%     \caption{Surfaces of constant value for a within-$1.1$ sum-decomposition of \texttt{max}, and for \texttt{max} itself. The sum-decomposition $\widehat{f}_1$ is defined by \Cref{eq:max_1d_approx}.}
%     \label{fig:max_3d}
% \end{figure}

\begin{figure}
    \centering
    \input{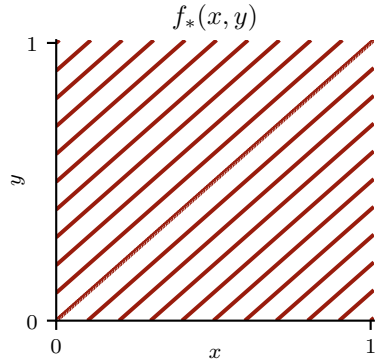}
    \caption{Contours of the function $f_*$ as defined by \Cref{eq:defn_bad_target}.}
    \label{fig:f_star_contours}
\end{figure}

This contour perspective can also give some intuition for the case we have addressed in this section, where even approximation is not possible. As long as the sum-decomposition $\widehat{f}$ is differentiable and goes via $\mathbb{R}$, any contour of $\widehat{f}$ which meets the line $y=x$ must do so at right angles, as is the case for the sum-decompositions plotted in \Cref{fig:max_contours}.\footnote{Neural networks can, of course, represent non-differentiable functions, but if we need to construct a $\phi$ which is non-differentiable in infinitely many places, things get more complicated.} \texttt{max} does not violate this condition so badly that we cannot approximate it -- its contours do at least cross $y=x$. But as \Cref{fig:f_star_contours} shows, the function $f_*$ from our proof of \Cref{thm:main_approximation_theorem} violates this condition very badly, having a contour along $y=x$, perpendicular to the contours of any sum-decomposition. Again, what we show in \Cref{fig:max_contours} and \Cref{fig:f_star_contours} does not prove anything, as we have glossed over many details -- such as the ability of neural networks to represent non-differentiable functions. But perhaps this provides some intuition for why this function should be so difficult to model with sum-decomposition.

%% file: 05_other_methods.tex
\section{Other Methods for Achieving Permutation Invariance}

\label{sec:other_methods}

In this section we discuss methods for achieving permutation invariance beyond $k$-ary Janossy pooling.
In particular, we cover two paradigms mentioned in passing in \Cref{sec:related_work}, which are also discussed in \cite{Murphy2018}: methods which achieve exact permutation invariance through sorting, and methods which do not achieve exact permutation invariance, but only approximate this property.
We also provide a brief discussion of the relationship between learning permutation-invariant functions for sets and learning on graphs.

\subsection{Sorting}
\label{sec:sorting}

When a set is represented as an ordered structure, it is generally assumed that the order can be arbitrary. That is, two representations $\mathbf{x}$ and $\mathbf{y}$ of the same set, where the elements are ordered differently, are both regarded as valid representations of the set. However, one could also say that each set has only one valid representation, corresponding to some canonical ordering of the elements. If it is ensured that only this valid representation is seen by the model by first sorting the elements, then permutation invariance is guaranteed.
This is, in most cases, computationally cheaper than permuting and pooling. The only additional computation necessary is the sorting of the elements into their canonical representation, which is $O(n\log n)$. This method, visualised in \Cref{fig:sorting}, does have some drawbacks.
\begin{figure}
\includegraphics[width=\textwidth]{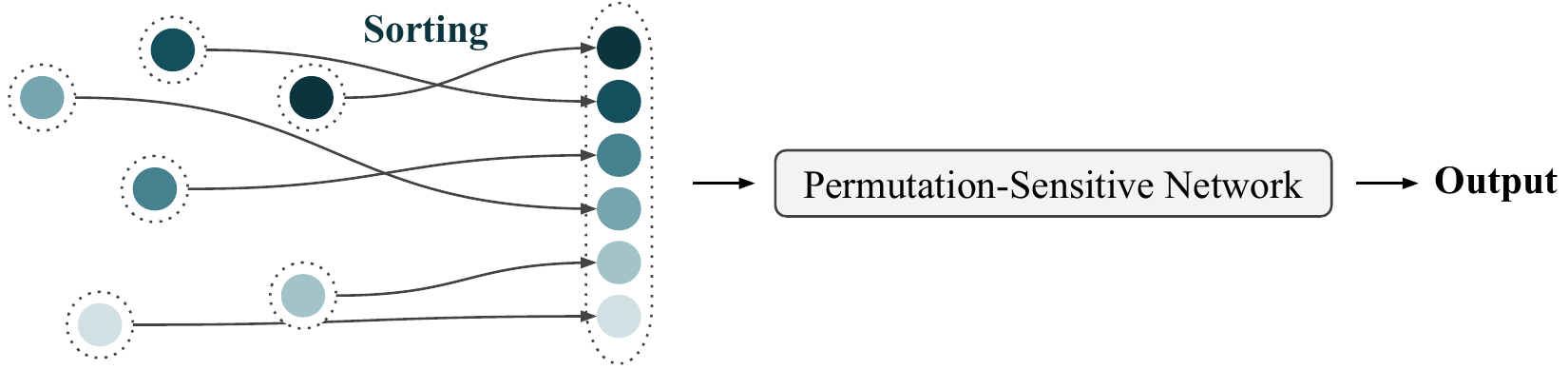}
\caption{Achieving permutation invariance via sorting. Finding or learning a canonical way of sorting the elements of a set before feeding it into a permutation-sensitive network yields a permutation-invariant output -- if the sorting does not depend on the initial order of the elements.}
\label{fig:sorting}
\end{figure}

First, there is an inherent ambiguity of how to pick the canonical ordering. 
This choice can be made manually, but the choice of ordering may in fact affect the performance of the model.
This is because the ordering may or may not be meaningful in the context of the task.
To illustrate this, consider two ways of sorting the integers.
First, sorting according to the usual order relation on the integers ($1 < 2 < 3$), and second, sorting alphabetically by the representation of each integer as an English word (``one'' $<$ ``three'' $<$ ``two'').
If we want to compute the \texttt{max} of a set of integers, the first choice of sort operation renders the task trivial, while the second choice of sort operation does not.

To avoid poorly-specified orderings, the ordering can instead be learned -- or more straightforwardly, a function can be learned giving a score to each element, which is then used to sort the elements by their scores.
This raises the issue of how to learn such a scoring function. More specifically, how are the gradients for this scoring function obtained?
The ranking of elements according to their sorting score is a piecewise constant function,\footnote{That is, the function taking a list of elements to a list of ranks. For alphabetical sorting, for example, we have $(\text{bat}, \text{cat}, \text{ant}) \mapsto (2,3,1)$.} meaning that the gradients are zero almost everywhere and undefined at the remaining locations. This would not be a problem if there were labels for the perfect ranking during training time -- in that case the model could just predict the ranking and gradients could be obtained by comparing the model's predictions to the ground truth. But, in general, there are no labels for the perfect ranking. A good ranking is whatever allows the decoder (the permutation-sensitive network) to perform well. This makes getting proper gradients significantly more difficult.

Backpropagating through piece-wise constant functions, however, is an established task in deep learning. The straight-through estimator, for example, is a viable tool to apply here \citep{BengioStraightThrough,UnderstandingStraightThrough}. Recently, a cheap differentiable sorting operation with a computational cost of $O(n \log n)$ was also proposed in \cite{blondel2020}. To the best of our knowledge, this has not been applied to standard set-based deep learning tasks, but this area could certainly provide interesting applications of this method.

\cite{NonLocal2020} propose yet another approach: after ranking the inputs according to a learned score function, a 1D convolution is applied to the ranked values. However, in order to get gradients for the score function, \cite{NonLocal2020} multiply each value with its score before feeding it into the convolutional layer. Interestingly, while the gradients do backpropagate into the score function, they do not come from the sorting. The gradients come from the scores being used as features later on.

\subsection{Approximate Permutation Invariance}
\label{sec:approx}

One motivation for $k$-ary Janossy pooling is its low computational cost relative to full $M$-ary Janossy pooling.
Recall that the high computational cost of $M$-ary Janossy pooling is due to the fact that $S_M$ consists of $M!$ permutations, each of which must be computed.
\cite{Murphy2018} propose an alternative method of reducing this computational cost: instead of pooling over all permutations in $S_M$, a fixed number $p<M!$ of permutations are randomly sampled. 
The pooling is then performed only over those random samples, yielding approximate permutation invariance.
\cite{Murphy2018} show that this method can provide good empirical results even when setting $p=1$. It is also possible to set, e.g., $p=1$ during training and $p>1$ at test time when fidelity of our predictions is more important, akin to ensemble methods. This could even be a potentially useful way to obtain uncertainty estimates for the model's predictions.

\cite{Pabbaraju2020} approach approximate permutation invariance in a different way. Inspired by Generative Adversarial Networks \citep{GANs}, they propose to use an adversary that tries to find the permutation of the input data which maximises the loss of the main task. This encourages the model to become close to permutation-invariant to minimise the impact of the adversary.

\subsection{From Sets to Graphs}
\label{sec:graphs}

\begin{figure}
\includegraphics[width=\textwidth]{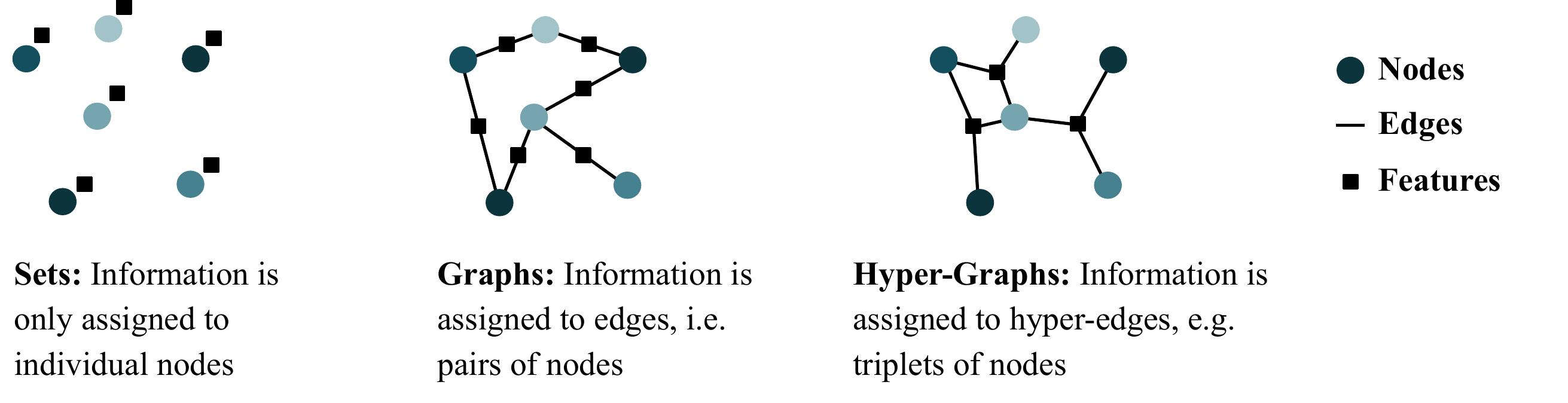}
\caption[From sets to graphs and hyper-graphs.]{From sets to graphs and hyper-graphs. Sets can be extended to graphs by introducing the notion of edges. A permutation of the input now corresponds to a permutation of the nodes together with the respective edges.}
\label{fig:graphs}
\end{figure}

Even though this paper is focussed on sets, not graphs, there is an obvious link between these two topics, which we briefly discuss.
% The above discussion of neural networks applied to sets paves the way for drawing a connection to graphs, as visualised in \Cref{fig:graphs}: 
Graphs typically consist of nodes (which form sets) and edges, where each edge connects two nodes. This can be extended to hyper-graphs which include information attached to triplets of nodes \citep{HyperGCN,LambdaNet}. From this perspective, it becomes clear that a set can be seen as a graph without edges (\Cref{fig:graphs}).

Whereas a set of $M$ nodes (without edge information) can be represented as a vector of length $M$, a fully-connected bi-directional graph has $M(M-1)$ edges, which can be captured in an $M \times M$ adjacency matrix. Analogously, hyper-graph data can be represented with higher order adjency tensors. The concept of permutations can now easily be extended from sets to graphs: in addition to nodes, the adjacency matrices capturing the edge information are transformed by the same permutation. This raises the question whether permutation invariant architectures from the sets literature can be applied to graphs as well. Interestingly, there is no canonical way to apply Deep Sets to graphs as it is not clear how to process the edge information. With self-attention algorithms \citep{Vaswani2017}, on the other hand, it is straightforward to include edge information \citep{velikovi2017graph, fuchs2020}.

Notably, \cite{Maron2019} study the set of independent linear, permutation-invariant or equivariant functions (i.e. matrices) transforming vectorised versions of the adjacency tensors. In fact, there is an orthogonal basis of such linear, permutation-equivariant functions. The number of basis vectors is connected to the \textit{Bell number} and is, remarkably, independent of the number of nodes in the (hyper-)graph. As an example, the edge information of a bi-directional graph can be written as an $M \times M$ adjacency matrix or, after vectorising this matrix, as a vector of length $M^2$. This vector can now be linearly transformed by multiplying it with a matrix of size $M^2 \times M^2$. \cite{Maron2019} show that there is an orthogonal basis of 15 different $M^2 \times M^2$ matrices which transform the (vectorised) adjacency matrix in an equivariant manner. The number 15 is the fourth Bell number, and is independent of the size $M$ of the graph.

This result shows a contrast between linear and nonlinear permutation-invariant functions. The complexity of the space of linear functions does not change as the number of inputs increases. By contrast, at least from the perspective of Deep Sets, the ``complexity'' of the nonlinear function space as measured by the necessary latent dimension increases with the number of inputs, as is shown in \Cref{sec:universal_representation}. Our main result in \Cref{sec:universal_approximation} shows that this is the case even if we consider not the whole function space, but also any dense subspace.

% If the (hyper-)graph data is captured in adjency tensors, one can find a set of independent linear, permutation-invariant or equivariant functions (i.e. matrices) transforming vectorised versions of these tensors. \cite{Maron2019} show that there is an orthogonal basis of such linear, permutation equivariant functions. The number of basis vectors is connected to the \textit{Bell number} and is, remarkably, independent of the number of nodes in the (hyper-)graph.

% \fabian{TODO} define what quadratic means, clarify whether and in what sense this is actually linear; the scaling of Maron's approach seems to be quadratic

%% file: 06_function_representation.tex
\section{Universality Beyond Deep Sets}

We have seen that Deep Sets is theoretically capable of representing all permutation-invariant functions if the latent space is at least as large as the cardinality of the sets being processed. 
For large input set sizes, this can be prohibitively slow and memory-consuming. 
Moreover, even for small set sizes, the theoretical universality does not guarantee that this is the best choice in practice. 
For example, from relational reasoning experiments \citep{fuchs2019endtoend}, it is known that Deep Sets is not always best at learning about interactions between input elements.
Especially when dependencies between elements are important for the task, other architectures, such as self-attention, are often preferred. 
This raises the question of how this empirical evidence of superior performance on some tasks relates to general theoretical limitations of these models.
In the following, we provide a summary of the sufficiency and necessity criteria for universal function representation for different categories of set-based learning approaches.

\label{sec:universal_representation_other}

\subsection{Janossy Pooling}

We consider $k$-ary Janossy pooling as depicted in Figure \ref{fig:permuting} and described in Section \ref{sec:related_work}. We assume a universal function approximator processing the permuted input subsets, a sum over the processed inputs, and a final component which is again a universal function approximator acting on the sum.

\begin{definition}
Let $M, N, k \in \mathbb{N}$, $\mathbf{x} \in \mathbb{R}^M$. Let $f:\mathbb{R}^M \to \mathbb{R}$. Write $T_k(\mathbf{x})$ for the set of all $k$-tuples of coordinates of $\mathbf{x}$. We say that $f$ has a \emph{continuous $k$-ary Janossy representation via $\mathbb{R}^N$} if there exist continuous functions $\phi: \mathbb{R}^k \to \mathbb{R}^N$ and $\rho: \mathbb{R}^N \to \mathbb{R}$ such that:

\begin{equation*}
f(\mathbf{x}) = \rho \Big( \sum_{\mathbf{t} \in T_k(\mathbf{x})} \phi(\mathbf{t}) \Big) 
\end{equation*}
\end{definition}

\subsubsection{Case $k=1$}

This case is Deep Sets. In Sections \ref{sec:universal_representation} and \ref{sec:universal_approximation}, we provided an in-depth analysis of this case. In a nutshell, Janossy pooling for $k=1$ can approximate all permutation-invariant functions if the latent space $N$ is at least as large as the number of inputs $M$. Moreover, we show that a smaller latent space ($N<M$) does not suffice in general.

\subsubsection{Case $1<k<M$}

This inherits sufficiency from the case $k=1$: the network processing inputs of cardinality $k$ can ignore all inputs except the first. This recovers Deep Sets, though each element is seen more than once, so the model must divide its output by the number of inputs having the same first element. In Janossy pooling, this number is $(M-1)\cdot(M-2)\cdot \ldots \cdot(M-2)(M-k) = (M-1)!/(M-k-1)!$. This gives the following statement:

\begin{theorem}
Let $f: \mathbb{R}^{M} \to \mathbb{R}$ be continuous and permutation-invariant. Then $f$ has a continuous $k$-ary Janossy representation via $\mathbb{R}^M$ for any choice of $k$.
\end{theorem}

\subsubsection{Case $k=M$}
\label{sec:KisM}
In this case, the first part of the neural network sees the entire input, and can therefore directly approximate the desired function. The model must then divide by a factor which takes into account the sum over the permuted terms, namely $M!$. This means that there is no requirement anymore on the size of the latent space, and we can make the following statement:

\begin{theorem}
Let $f: \mathbb{R}^{M} \to \mathbb{R}$ be continuous and permutation-invariant. Then $f$ has a continuous $M$-ary Janossy representation via $\mathbb{R}$.
\end{theorem}

\begin{figure}
	\centering
	\begin{subfigure}[h]{0.47\textwidth}
		\centering
		\includegraphics[width=0.99\textwidth]{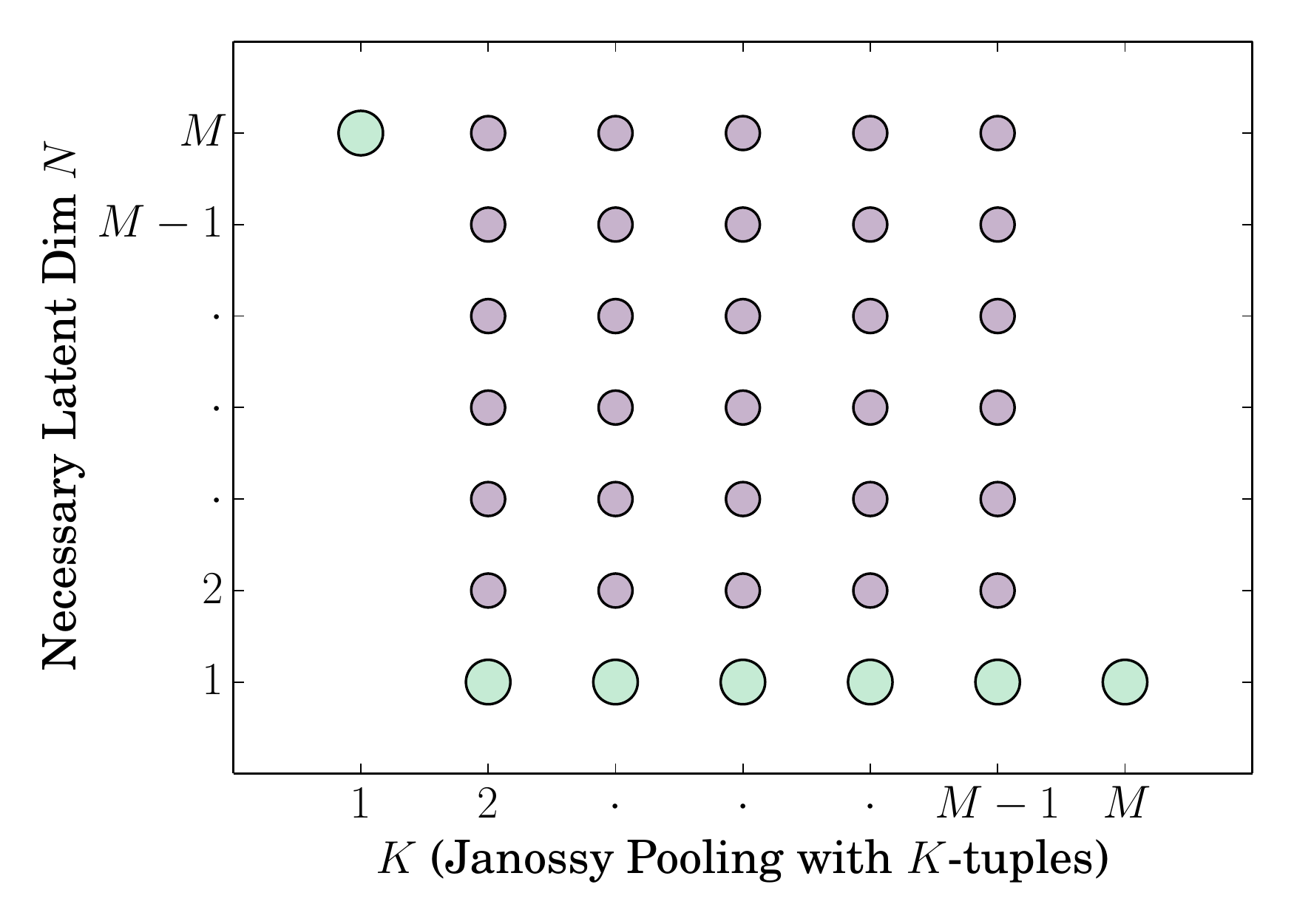}
		\caption{Minimum Necessity}
		\label{fig:influence_local_instability_all}
	\end{subfigure} 
	\begin{subfigure}[h]{0.47\textwidth}
		\centering
		\includegraphics[width=0.99\textwidth]{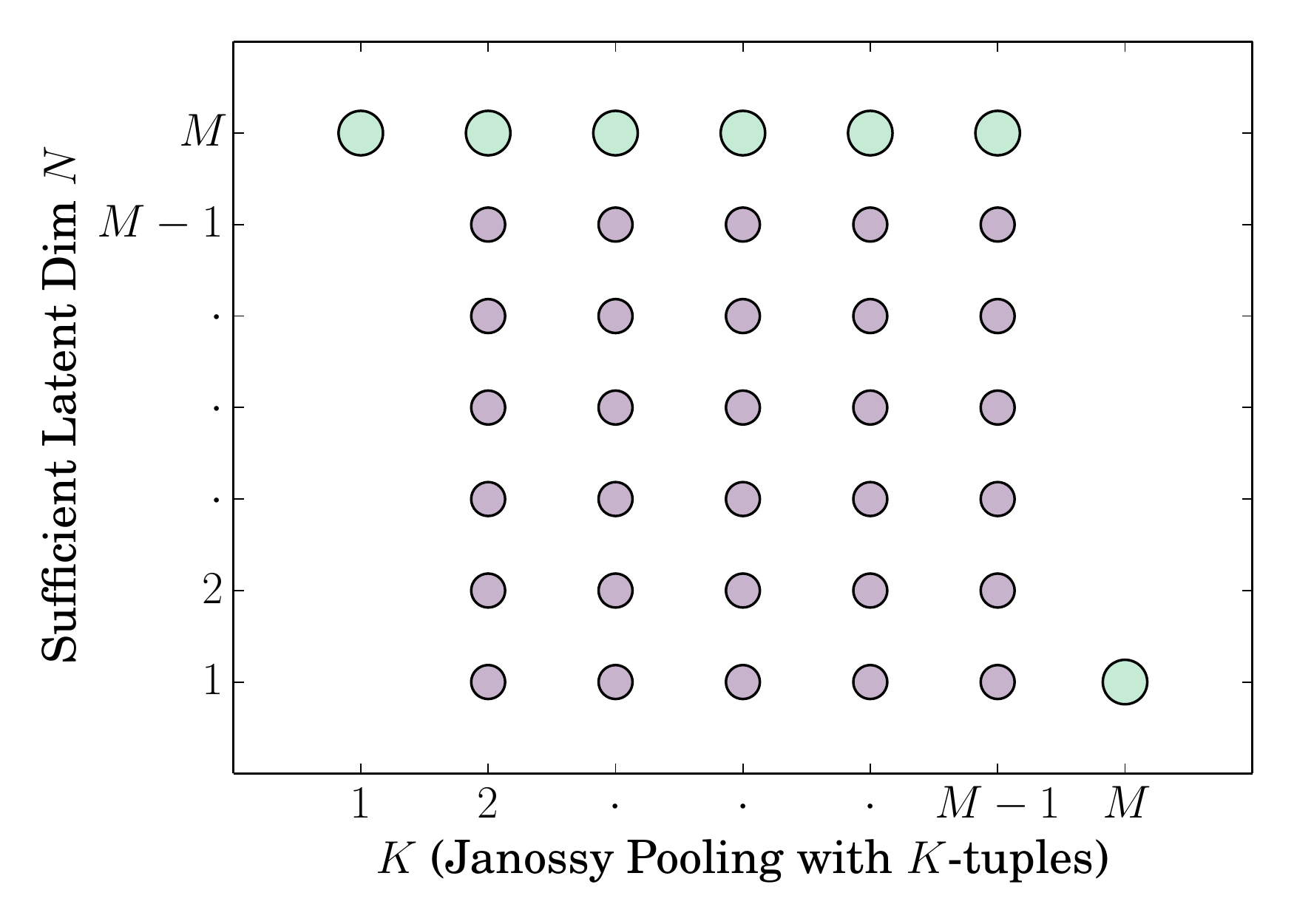}
		\caption{Maximum Sufficiency}
		\label{fig:influence_local_instability_easy}
	\end{subfigure} 
	\caption[Necessary and sufficient conditions for the size of the latent dimension for Janossy pooling.]{Necessary and sufficient conditions for the size of the latent dimension for Janossy pooling. The large, green circles show the minimum and maximum boundaries as according to our proofs. The smaller, purple circles show possible true necessary and sufficient conditions for universal function representation.}
	\label{fig:K_suff_necc}
\end{figure}

\subsubsection{Open Questions}
For the extreme cases of $k$-ary Janossy pooling, $k=1$ and $k=M$, we know a necessary and sufficient condition on the latent space dimension required for universal approximation. 
For $1<k<M$, however, there is a gap between the largest known necessary condition and the lowest known sufficient condition. Figure \ref{fig:K_suff_necc} visualises the conditions we derived and what the true, strongest statements might be. 
As an example: for $k=2$, we know that a latent dimension of $N=M$ is sufficient and a latent dimension of $N=1$ is necessary in order to be able to represent all permutation-invariant functions. 
However, it is possible that $N=2$ is sufficient, or that $N=M-1$ is necessary.
This is an interesting open question and we leave it for future work to further investigate this.

The proofs of our results on Deep Sets rely on the input elements being 1-dimensional. It is known that Deep Sets is universal for approximation of set functions with higher-dimensional inputs, and indeed sufficient conditions on the latent space dimension are known for both invariant \citep{han2019universal} and equivariant \citep{segol2020universal, maron2019provably} set functions. It is unknown, however, whether these known sufficiency bounds are also necessary -- indeed, no non-trivial necessary bound is known, other than the 1-dimensional case presented in \Cref{sec:universal_approximation}.

\subsection{Sorting}
In Section \ref{sec:sorting}, we described a category of methods which achieve permutation invariance through sorting the inputs. We assume that a (potentially learned) scoring function is used to sort the inputs. A list of sorted inputs is then fed to a universal function approximator (e.g. a sufficiently large neural network). As the order of the sorted input is invariant with respect to permutations of the input, the whole architecture is automatically permutation-invariant. Since the second part of the model is a universal approximator, this architecture choice is therefore a universal function approximator for permutation-invariant functions and we arrive at the following conclusion:

\begin{theorem}
Let $f: \mathbb{R}^{M} \to \mathbb{R}$ be continuous and permutation-invariant. Then for any sorted version of the inputs, $f$ has a continuous representation which acts on these sorted inputs.
\end{theorem}
Caution must be exercised in the case where the sorting does not operate directly on the inputs, but instead on features derived from the inputs. Universality still holds if the features are a smooth bijective mapping from the inputs.

\subsection{Approximate Permutation Invariance}

In Section \ref{sec:approx}, we mentioned two approaches to approximate permutation invariance.
These models are in some sense \emph{too} expressive -- they may, depending on the exact implementation details, be capable of approximating the permutation-invariant target function, but they may also be capable of approximating functions which are not permutation-invariant. 

In the approach of \cite{Pabbaraju2020}, permutation invariance is encouraged by the training process, but the model class is not constrained to be permutation-invariant -- if a permutation-sensitive function approximator is used to build the model, then the trained model may retain the ability to distinguish between different permutations of the same set. Whether the model is capable of approximating the target function depends entirely on the model class chosen -- \cite{Pabbaraju2020} choose an LSTM, but other models could be trained with a similar adversarial method.

In the case of approximate Janossy pooling it is similarly clear that, if a universal approximator is used for the function whose outputs are pooled (i.e. for the function $f$ in the notation of Equation \ref{eq:janossy}), then the overall model inherits this universality.
Although permutation invariance is not guaranteed, the model is guaranteed to approximate permutation invariance in the following sense.
Write $p$ for the number of permutations sampled, and write $f$ for the permutation-sensitive function whose outputs are pooled.
For a given set $\mathbf{x}$, there are $M!$ possible permutations $\pi$, and therefore $M!$ possible outputs $f\big(\pi(\mathbf{x})\big)$.
When $p=1$, we get an output sampled uniformly from these possible outputs, and we can view this output as a random variable with variance $\sigma^2$.
For $p>1$, the output is the mean of a random sample of size $p$ from the set of outputs of $f$.
The variance of this random variable is

\begin{equation*}
    \frac{M! - p}{M! - 1} \cdot \frac{\sigma^2}{p}.
\end{equation*}

This goes monotonically to $0$ as $p \to M!$.

%% file: 07_conclusion.tex
\section{Conclusion}
\label{sec:conclusions}

This work provides a theoretical characterisation of the representation and approximation of permutation-invariant functions. 
To this end, we focus on the Deep Sets architecture from \citet{Zaheer2017} and derive sufficient and necessary conditions for continuous function representation. 
This is achieved by noting the importance of considering continuous transformations. 
We find that Deep Sets is only able to represent all continuous functions if the latent space is at least as large as the number of inputs. 
Further, we relate these insights back to the broader paradigm of Janossy pooling, as introduced by \citet{Murphy2018}.

Continuing our focus on continuous transformations, we adopt a topological perspective to answer the question of whether requiring universal function \textit{representation} is an overly strong criterion, and whether universal function \textit{approximation} would be possible even with smaller latent spaces. 
We turn again to the Deep Sets architecture and answer this question in the negative. 
We show that functions exist which can only be approximated very poorly (i.e. the worst case error is no better than a constant baseline) when making the latent space any smaller than the number of inputs.

We hope that the analytical arguments in this work inspire future theoretical treatments of the subject while also providing practical guidance for machine learning practitioners. 
Regarding future work, it would be interesting to extend the theoretical analysis beyond scalar-valued set elements, to vector-valued set elements, as commonly encountered in practical applications.

%% file: 08_appendices.tex
\appendix

\section{Mathematical Remarks}

\subsection{Infinite Sums}
\label{sec:infinite_sums}

Throughout this paper we consider expressions of the form

\begin{equation}
\label{eq:setsum}
\Phi(X) = \Sigma_{x \in X} \phi(x)
\end{equation}

where $X$ is an arbitrary set. The meaning of this expression is clear when $X$ is finite, but when $X$ is infinite, we must be precise about what we mean.

\subsubsection{Countable Sums}

We usually denote countable sums as e.g. $\Sigma_{i=1}^\infty x_i$. Note that there is an ordering of the $x_i$ here, whereas there is no ordering in \Cref{eq:setsum}. The reason that we consider sums is for their permutation invariance in the finite case, but note that in the infinite case, permutation invariance of sums does not necessarily hold! For instance, the alternating harmonic series $\Sigma_{i=1}^\infty \frac{(-1)^i}{i}$ can be made to converge to any real number simply by reordering the terms of the sum. For expressions like \eqref{eq:setsum} to make sense, we must require that the sums in question are indeed permutation invariant. This property is known as \emph{absolute convergence}, and it is equivalent to the property that the sum of absolute values of the series converges. We therefore require everywhere that $\Sigma_{x \in X} |\phi(x)|$ is convergent. For any $X$ where this is not the case, we will set $\Phi(X) = \infty$.

\subsubsection{Uncountable Sums}

It is well known that a sum over an uncountable set of elements only converges if all but countably many elements are 0. Allowing sums over uncountable sets is therefore of little interest, since it essentially reduces to the countable case.

\subsection{Continuity of Functions on Sets}
\label{sec:cont_set_fun_remark}

We are interested in functions on subsets of $\mathbb{R}$, i.e. elements of $2^\mathbb{R}$, and the notion of continuity on $2^\mathbb{R}$ is not straightforward. 
As a convenient shorthand, we discuss ``continuous'' functions $f$ on $2^\mathbb{R}$, by which we mean that the function $f_M$ induced by $f$ on $\mathbb{R}^M$ by $f_N(x_1, ..., x_M) = f(\{x_1, ..., x_M\})$ is continuous for every $M \in \mathbb{N}$.

\subsection{Remark on \Cref{ori_countable_theorem}}
\label{sec:multisets}

The proof for \Cref{ori_countable_theorem} from \citet{Zaheer2017} can be extended to dealing with multi sets, i.e. sets with repeated elements. To that end, we replace the mapping to natural numbers $c(X): \mathbb{R}^M \to \mathbb{N}$ with a mapping to prime numbers $p(X): \mathbb{R}^M \to \mathbb{P}$. We then choose $\phi(x_m) = -\log p(x_m)$. Therefore,
\begin{equation}
\Phi(X) = \sum_{m=1}^{M} \phi (x_m) = \log \prod_{m=1}^{M} \frac{1}{p(x_m)}
\end{equation}
which takes a unique value for each distinct $X$ therefore extending the validity of the proof to multi-sets. This choice of $\phi$ diverges with infinite set size.

In fact, it is straightforward to show that there is no function $\phi$ for which $\Phi$ provides a unique mapping for arbitrary multi-sets while at same time guaranteeing convergence for infinitely large sets. Assume a function $\phi$ and an arbitrary point $x$ such that $\phi(x) = a \neq 0$. Then, the multiset comprising infinitely many identical members $x$ would give: 

\begin{equation}
\Phi(X) = \sum_{i=1}^{\infty} \phi (x_m) = \sum_{i=1}^{\infty} a = \pm\infty
\end{equation}

\begin{figure}
    \centering
    \begin{tikzpicture}
    
    \newlength\ultwo
    \setlength{\ultwo}{1.25em}
           
    \draw[pattern=south east lines,pattern color=edred,draw=edred,very thick, rounded corners=1.25\ultwo,fill opacity=0.5] 
            (0, 7.5\ultwo) -- (7\ultwo, 4\ultwo) -- (7\ultwo, 1.5\ultwo) -- (-7\ultwo, 1.5\ultwo) -- (-7\ultwo, 4\ultwo) -- cycle;
            
    \draw[fill=edblue,fill opacity=0.15,color=edblue,very thick, rounded corners=1\ultwo] 
            (-3\ultwo, -7\ultwo) -- (-7\ultwo, -1.5\ultwo) -- (-7\ultwo, 0.5\ultwo) -- 
            (7\ultwo, 0.5\ultwo) -- (7\ultwo, -1.5\ultwo) -- (3\ultwo, -7\ultwo) -- cycle;
            
    \graph[nodes={draw,very thick,fill=gray!20!white,text=black,circle,inner sep=0.2\ultwo},
           edges={very thick},
           radius=5\ultwo] {
           subgraph K_n [clockwise, V={$[3]_6$,$[5]_6$,$[6]_6$,$[4]_6$,$[2]_6$,$[0]_6$,$[1]_6$}]
           };
            
    \node[color=edred] at (6\ultwo, 6\ultwo) {\Large $X_6^{+1}$};
    \node[color=edblue!80!black] at (6\ultwo, -6\ultwo) {\Large $X_6^{\smn 1}$};
    
    \end{tikzpicture}
    \caption[The $6$-dimensional simplex $\Delta_6$ visualised as the complete graph $K_7$.]{The $6$-dimensional simplex $\Delta_6$, and the low-dimensional faces $X_6^{+1}$ and $X_6^{\smn 1}$, visualised using the complete graph $K_7$. Here, $[m]_6$ denotes the vertex of $\Delta_6$ whose coordinates in $\mathbb{R}^6$ consist of $m$ ones followed by $6-m$ negative ones.}
    \label{fig:simplex_graph_vis}
\end{figure}
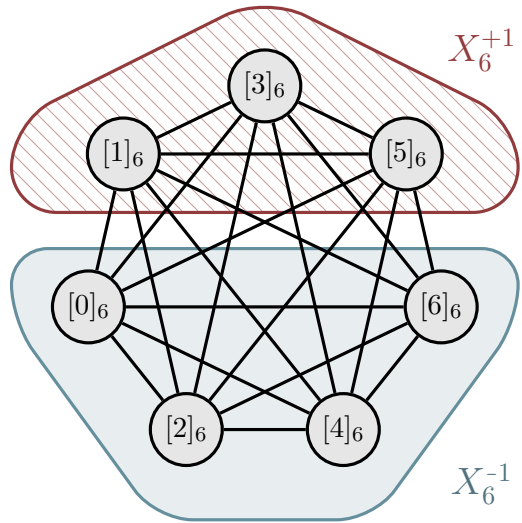

\subsection{Visualising High-dimensional Simplices\label{sec:visualising_simplices}}

In the proof of \Cref{thm:main_approximation_theorem}, we make heavy use of the $N$-dimensional simplex $\Delta_N$, its $N-1$-dimensional faces $\Delta_N^{(i)}$, and the two opposite faces $X_M^{+1}$ and $X_M^{\smn 1}$ of the $M$-dimensional simplex $\Delta_M$. To mentally keep track of all these different objects, it is helpful to have a way of visualising them. Although these simplices are high-dimensional objects, they can be visualised easily as graphs, as shown in \Cref{fig:simplex_graph_vis}.

An $N$-dimensional simplex can be defined as the convex hull of $N+1$ points in general position.\footnote{That is, no three points are collinear, no four points are coplanar, etc.} We can represent this as the complete graph $K_{N+1}$, where just as any pair of vertices of the simplex are joined by an edge, so are any pair of vertices of the graph. An $N-1$-dimensional face of the simplex is just the convex hull of $N$ of the $N+1$ vertices, and this can be viewed as an induced $K_N$ subgraph of $K_{N+1}$. Lower-dimensional faces can be viewed similarly. 

Under our definition of $\Delta_N$, the vertices of the simplex are the points whose coordinates are a (possibly empty) sequence of ones followed by a sequence of negative ones. Any vertex can be written in shorthand as $[m]_N$, indicating the point in dimension $N$ consisting of $m$ ones followed by all negative ones. For example, $[2]_3 = (1, 1, \smn1)$, the vertex of $\Delta_3$ beginning with two ones.

\Cref{fig:simplex_graph_vis} illustrates this way of visualising $\Delta_N$, showing the 6-dimensional simplex $\Delta_6$ represented as a complete graph with 7 vertices. With the vertices arranged appropriately, the faces $\smash{X_6^{+1}}$ and $\smash{X_6^{\smn1}}$ correspond to the indicated subgraphs, and it is visually clear that it is sensible to describe them as ``opposite'' faces of the simplex. The faces $\smash{\Delta_6^{(0)}}$ and $\smash{\Delta_6^{(6)}}$ are obtained by removing $[0]_6$ and $[6]_6$ respectively. 

\subsection{Extending \Cref{thm:main_approximation_theorem} to Higher Dimensions}

As discussed in the main text, an interesting line of future work lies in extending \Cref{thm:main_approximation_theorem} to sets of higher-dimensional elements.
We briefly make note of a small obstacle to extending the arguments of this work to the higher-dimensional case.

Our proof of \Cref{thm:main_approximation_theorem} makes use of the fact that an arbitrary continuous function $f$ on $\mathbb{R}^M$ can be used to construct a continuous permutation-invariant function by composing with $\texttt{sort}$.
This simplifies the process of defining a hard-to-approximate continuous permutation-invariant target function $f_*$, as in \Cref{eq:defn_bad_target}.
For $d$-dimensional inputs with $d>1$, however, this simplification is not possible because there is no continuous sorting function on $\mathbb{R}^d$.

In fact we can say something slightly stronger than this.
The important property of \texttt{sort} for the purpose above is that it maps all permutations of a given set to the same canonical permutation. 
That is, it has the following two properties:

\begin{enumerate}
    \item It is permutation-invariant.
    \item For any set $\mathbf{x}$, there is a permutation $\pi$ such that $\texttt{sort}\big(\pi(\mathbf{x})\big) = \pi(\mathbf{x})$.
\end{enumerate}

We refer to any function satisfying these properties as a \emph{quasi-sorting function}, since this definition includes many functions which cannot be defined by sorting according to any order relation.
For sets of $d$-dimensional elements, there is not only no continuous sorting function, there is also no continuous quasi-sorting function.
To see this, we give an informal argument which can easily be made precise.

\begin{figure}
    \centering
    \begin{subfigure}[b]{0.3\textwidth}
        \resizebox{\textwidth}{!}{\input{figures/rotate1.tikz}}
        \caption{}
        \label{fig:two_points_rotating_a}
    \end{subfigure}
    \quad
    \begin{subfigure}[b]{0.3\textwidth}
        \resizebox{\textwidth}{!}{\input{figures/rotate2.tikz}}
        \caption{}
        \label{fig:two_points_rotating_b}
    \end{subfigure}
    \quad
    \begin{subfigure}[b]{0.3\textwidth}
        \resizebox{\textwidth}{!}{\input{figures/rotate3.tikz}}
        \caption{}
        \label{fig:two_points_rotating_c}
    \end{subfigure}
    
    \medskip
    
    \begin{subfigure}[b]{0.3\textwidth}
        \resizebox{\textwidth}{!}{\input{figures/rotate4.tikz}}
        \caption{}
    \end{subfigure}
    \quad
    \begin{subfigure}[b]{0.3\textwidth}
        \resizebox{\textwidth}{!}{\input{figures/rotate5.tikz}}
        \caption{}
        \label{fig:two_points_rotating_e}
    \end{subfigure}
    
    \caption[Illustration of a quasi-sorting function $g$ on $\mathbb{R}^{2\times 2}$.]{Rotating a pair of points around their midpoint, and seeing how the canonical ordering changes. Figures \ref{fig:two_points_rotating_a} and \ref{fig:two_points_rotating_e} represent the same set, so the canonical ordering $g(\mathbf{x})$ must be the same in both cases. This is visually represented by the colours of the points, where the solid blue point occupies the first row of $g(\mathbf{x})$ and the striped red point occupies the second row of $g(\mathbf{x})$. In order for the colours in \ref{fig:two_points_rotating_a} and \ref{fig:two_points_rotating_e} to match, point ``$1$'' must switch from blue to red at some point during the 180-degree rotation from \ref{fig:two_points_rotating_a} to \ref{fig:two_points_rotating_e}. This switching of colours represents swapping the rows of $g(\mathbf{x})$, which is a discontinuity.}
    \label{fig:two_points_rotating}
\end{figure}
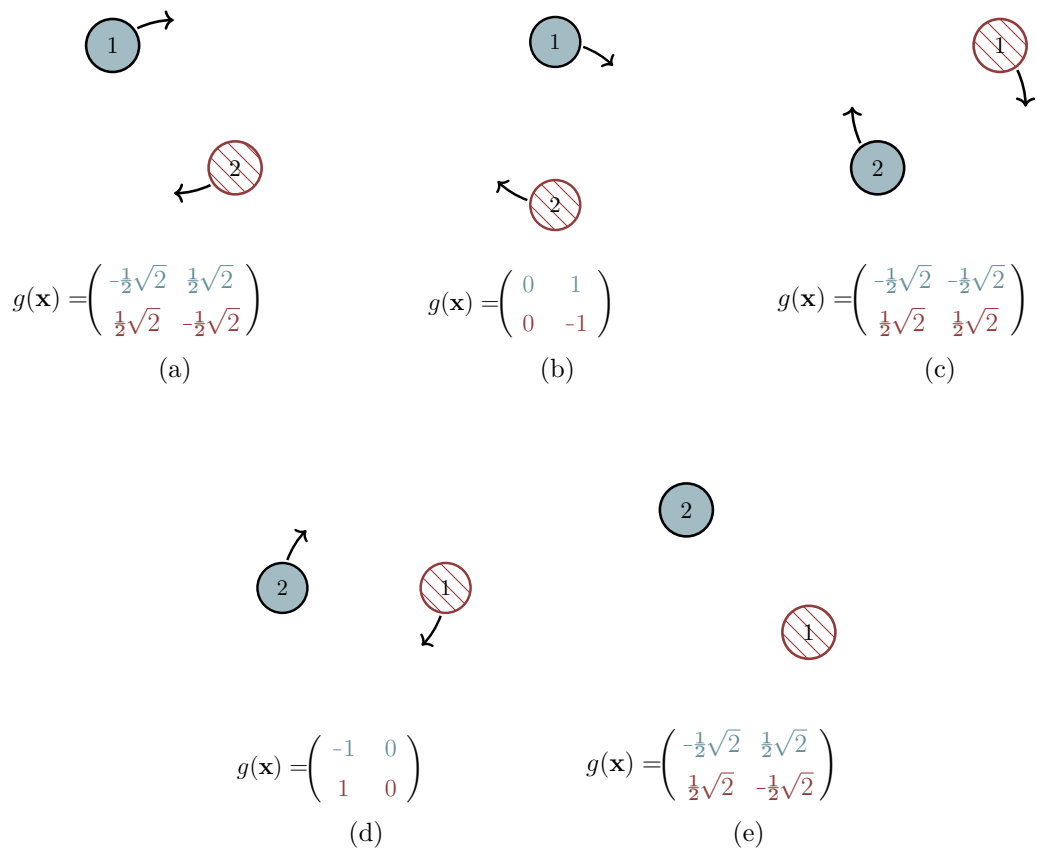

Consider the simplest case, a set of two 2-dimensional elements, and let $g$ be a quasi-sorting function.
We can visualise a set of two 2-dimensional elements as being a pair of points in the plane.
As is the case throughout this paper, we are working with an ordered representation of this set, in this case as an element $\mathbf{x}$ of $\mathbb{R}^{2\times 2}$. 
That is, we have a $2\times 2$ matrix $\mathbf{x}$, where each row gives the coordinates of one point.

Now rotate this pair of points around their midpoint, in a continuous motion, through an angle of 180 degrees.
This motion is illustrated in \Cref{fig:two_points_rotating}, with the point in the first row of $\mathbf{x}$ labelled ``$1$'' and the point in the second row labelled ``$2$''.
The points are coloured according to which row they belong to in the canonical permutation $g(\mathbf{x})$, with solid blue indicating the first row and striped red indicating the second row.

After a 180 degree rotation, the set is the same as it was at the beginning.
The canonical permutation, i.e. the colouring, is also the same as at the beginning.
This means that at some point during the rotation, the colour of each point has switched -- in \Cref{fig:two_points_rotating}, this happens between \ref{fig:two_points_rotating_b} and \ref{fig:two_points_rotating_c}.
This switching of colours corresponds to swapping the rows of the output matrix.
Since the rows of the matrix are distinct (because the points representing each row are a constant distance apart), this swapping is discontinuous, and therefore $g$ is not continuous.

\section{Proofs and Additional Results\label{sec:proofs_appendix}}

\begin{restatable}{theorem}{qdiscontinuous}
\label{thm:q_discontinuous}
There exist functions $f : 2^\mathbb{Q} \to \mathbb{R}$ such that, whenever $(\rho, \phi)$ is a sum-decomposition of $f$ via $\mathbb{R}$, $\phi$ is discontinuous at every point $q \in \mathbb{Q}$.
\end{restatable}

\begin{restatable}{theorem}{uncdiscontinuous}
\label{thm:uncountable_discontinuous}
Let $f : \mathbb{R}^{\mathcal{F}} \to \mathbb{R}$. Then $f$ is sum-decomposable via $\mathbb{R}$.
\end{restatable}

\begin{restatable}{theorem}{uncfinite}
\label{thm:uncountable_only_finite_subsets}
If $\mathfrak{U}$ is uncountable, then there exist functions $f : 2^\mathfrak{U} \to \mathbb{R}$ which are not sum-decomposable. This holds even if the sum-decomposition $(\rho, \phi)$ is allowed to be discontinuous.
\end{restatable}

\subsection{Proof of \Cref{thm:q_discontinuous}}

\qdiscontinuous*

\begin{proof}
Consider $f(X) = \texttt{sup}(X)$, the least upper bound of $X$. Write $\Phi(X) = \Sigma_{x \in X} \phi(x)$. So we have:

\begin{displaymath}
\texttt{sup}(X) = \rho(\Phi(X))
\end{displaymath}

First note that $\phi(q) \neq 0$ for any $q \in \mathbb{Q}$. If we had $\phi(q) = 0$, then we would have, for every $X \subset \mathbb{Q}$:

\begin{displaymath}
\Phi(X) = \Phi(X) + \phi(q) = \Phi(X \cup \{q\})
\end{displaymath}

But then, for instance, we would have: 

\begin{displaymath}
q = \texttt{sup}(\{q-1, q\}) = \texttt{sup}(\{q-1\}) = q-1
\end{displaymath}

This is a contradiction, so $\phi(q) \neq 0$.

Next, note that $\Phi(X)$ must be finite for every upper-bounded $X \subset \mathbb{Q}$ (since \texttt{sup} is undefined for unbounded $X$, we do not consider such sets, and may allow $\Phi$ to diverge). Even if we allowed the domain of $\rho$ to be $\mathbb{R} \cup \{\infty\}$, suppose $\Phi(X) = \infty$ for some upper-bounded set $X$. Then:

\begin{eqnarray*}
\texttt{sup}(X) & = & \rho(\Phi(X)) \\
& = & \rho(\infty) \\
& = & \rho(\infty + \phi(\texttt{sup(X) + 1})) \\
& = & \rho(\Phi(X \cup \{\texttt{sup}(X) + 1\})) \\
& = & \texttt{sup}(X \cup \{\texttt{sup}(X) + 1\}) \\
& = & \texttt{sup}(X) + 1
\end{eqnarray*}

This is a contradiction, so $\Phi(X) < \infty$ for any upper-bounded set $X$.

Now from the above it is immediate that, for any upper-bounded set $X$, only finitely many $x \in X$ can have $\phi(x) > \frac{1}{n}$. Otherwise we can find an infinite upper-bounded set $Y \subset X$ with $\phi(y) > \frac{1}{n}$ for every $y \in Y$, and $\Phi(Y) = \infty$.

Finally, let $q \in \mathbb{Q}$. We have already shown that $\phi(q) \neq 0$, and we will now construct a sequence $q_n$ with:

\begin{enumerate}
    \item $q_n \to q$
    \item $\phi(q_n) \to 0$
\end{enumerate}

If $\phi$ were continuous at $q$, we would have $\phi(q_n) \to \phi(q)$, so the above two points together will give us that $\phi$ is discontinuous at $q$. 

So now, for each $n \in \mathbb{N}$, consider the set $B_n$ of points which lie within $\frac{1}{n}$ of $q$. Since only finitely many points $p \in B_n$ have $\phi(p) > \frac{1}{n}$, and $B_n$ is infinite, there must be a point $q_n \in B_n$ with $\phi(q_n) < \frac{1}{n}$. The sequence of such $q_n$ clearly satisfies both points above, and so $\phi$ is discontinuous everywhere.
\end{proof}

\subsection{Proof of \Cref{thm:uncountable_discontinuous}}

\uncdiscontinuous*

\begin{proof}
Define $\Phi : \mathbb{R}^{\mathcal{F}} \to \mathbb{R}$ by $\Phi(X) = \Sigma_{x \in X}\phi(x)$. If we can demonstrate that there exists some $\phi$ such that $\Phi$ is injective, then we can simply choose $\rho = f \circ \Phi^{-1}$ and the result is proved.

Say that a set $X \subset \mathbb{R}$ is \emph{finite-sum-distinct} (f.s.d.) if, for any finite subsets $A, B \subset X$, $\Sigma_{a \in A}a \neq \Sigma_{b \in B}b$. Now, if we can show that there is a finite-sum-distinct set $D$ with the same cardinality as $\mathbb{R}$ (we denote $|\mathbb{R}|$ by $\mathfrak{c}$), then we can simply choose $\phi$ to be a bijection from $\mathbb{R}$ to $D$. Then, by finite-sum-distinctness, $\Phi$ will be injective, and the result is proved.

Now recall the statement of Zorn's Lemma: suppose $\mathcal{P}$ is a partially ordered set (or \emph{poset}) in which every totally ordered subset has an upper bound. Then $\mathcal{P}$ has a maximal element.

The set of all f.s.d. subsets of $\mathbb{R}$ (which we will denote $\mathcal{D}$) forms a poset ordered by inclusion. Supposing that $\mathcal{D}$ satisfies the conditions of Zorn's Lemma, it must have a maximal element, i.e. there is a f.s.d. set $D_\text{max}$ such that any set $E$ with $D_\text{max} \subsetneq E$ is not f.s.d. We claim that $D_\text{max}$ has cardinality $\mathfrak{c}$. 

To see this, let $D$ be a f.s.d. set with infinite cardinality $\kappa < \mathfrak{c}$ (any maximal $D$ clearly cannot be finite). We will show that $D \neq D_\text{max}$. Define the \emph{forbidden elements} with respect to $D$ to be those elements $x$ of $\mathbb{R}$ such that $D \cup \{x\}$ is not f.s.d. We denote this set of forbidden elements $F_D$. Now note that, if $D$ is maximal, then $D \cup F_D = \mathbb{R}$. In particular, this implies that $|F_D|=\mathfrak{c}$. But now consider the elements of $F_D$. By definition of $F_D$, we have that $x \in F_D$ if and only if $\exists c_1, ..., c_m, d_1, ..., d_n \in D$ such that $c_1 + ... + c_m + x = d_1 + ... + d_n$. So we can write $x$ as a sum of finitely many elements of $D$, minus a sum of finitely many other elements of $D$. So there is a surjection from pairs of finite sets of $D$ to elements of $F_D$. i.e.:

\begin{displaymath}
|F_D| \leq |D^{\mathcal{F}} \times D^{\mathcal{F}}|
\end{displaymath}

But since $D$ is infinite:

\begin{displaymath}
|D^{\mathcal{F}} \times D^{\mathcal{F}}| = |D| = \kappa < \mathfrak{c}
\end{displaymath}

So $|F_D| < \mathfrak{c}$, and therefore $|D|$ is not maximal. This demonstrates that $D_\text{max}$ must have cardinality $\mathfrak{c}$.

To complete the proof, it remains to show that $\mathcal{D}$ satisfies the conditions of Zorn's Lemma, i.e. that every totally ordered subset (or \emph{chain}) $\mathcal{C}$ of $\mathcal{D}$ has an upper bound. So consider: 

\begin{displaymath}
C_\text{ub} = \bigcup \mathcal{C} = \bigcup_{C\in\mathcal{C}} C
\end{displaymath}

We claim that $C_\text{ub}$ is an upper bound for $\mathcal{C}$. It is clear that $C \subset C_\text{ub}$ for every $C \in \mathcal{C}$, so it remains to be shown that $C_\text{ub} \in \mathcal{D}$, i.e. that $C_\text{ub}$ is f.s.d.

We proceed by contradiction. Suppose that $C_\text{ub}$ is not f.s.d. Then: 

\begin{equation}
\label{eq:notfsd}
\exists c_1, ..., c_m, d_1, ..., d_n \in C_\text{ub} : \Sigma_i c_i = \Sigma_j d_j
\end{equation}

But now by construction of $C_\text{ub}$ there must be sets $C_1, ..., C_m, D_1, ..., D_m \in \mathcal{C}$ with $c_i \in C_i, d_j \in D_j$. Let $\mathcal{B} = \{C_i\}_{i=1}^m \cup \{D_j\}_{j=1}^n$. $\mathcal{B}$ is totally ordered by inclusion and all sets contained in it are f.s.d., since it is a subset of $\mathcal{C}$. Since $\mathcal{B}$ is finite it has a maximal element $B_\text{max}$. By maximality, we have $c_i, d_j \in B_\text{max}$ for all $c_i, d_j$. But then by \eqref{eq:notfsd}, $B_\text{max}$ is not f.s.d., which is a contradiction. So we have that $C_\text{ub}$ is f.s.d. 

In summary: 

\begin{enumerate}
\item $\mathcal{D}$ satisfies the conditions of Zorn's Lemma.
\item Therefore there exists a maximal f.s.d. set, $D_\text{max}$.
\item We have shown that any such set must have cardinality $\mathfrak{c}$.
\item Given an f.s.d. set $D_\text{max}$ with cardinality $\mathfrak{c}$, we can choose $\phi$ to be a bijection between $\mathbb{R}$ and $D_\text{max}$.
\item Given such a $\phi$, we have that $\Phi(X)=\Sigma_{x \in X} \phi(x)$ is injective on $R^{\mathcal{F}}$.
\item Given injective $\Phi$, choose $\rho = f \circ \Phi^{-1}$.
\item This choice gives us $f(X) = \rho(\Sigma_{x \in X}\phi(x))$ by construction.
\end{enumerate}

This completes the proof.
\end{proof}

\subsection{Proof of \Cref{thm:uncountable_only_finite_subsets}}
\uncfinite*

\begin{proof}
Consider $f(X) = \texttt{sup}(X)$.

As discussed above, a sum over uncountably many elements can converge only if countably many elements are non-zero. But as in the proof of \Cref{thm:q_discontinuous}, $\phi(x) \neq 0$ for any $x$. So it is immediate that sum-decomposition is not possible for functions operating on uncountable subsets of $\mathfrak{U}$.

Even restricting to countable subsets is not enough. As in the proof of \Cref{thm:q_discontinuous}, we must have that for each $n \in \mathbb{N}$, $\phi(x) > \frac{1}{n}$ for only finitely many $x$. But then if this is the case, let $\mathfrak{U}_n$ be the set of all $x \in \mathfrak{U}$ with $\phi(x) > \frac{1}{n}$. Since $\phi(x) \neq 0$, we know that $\mathfrak{U} = \bigcup \mathfrak{U}_n$. But this is a countable union of finite sets, which is impossible because $\mathfrak{U}$ is uncountable.

\end{proof}

\subsection{Proof of \Cref{cor:ori_uncountable_theorem}}

\oriunc*

\begin{proof}
The reverse implication is clear. We already know from \citet{Zaheer2017} that the function $\Phi: \Delta_M \to \mathbb{R}^{M+1}$ defined as follows is a homeomorphism onto its image:

\begin{gather*}
\Phi_q(X) = \sum_{m=1}^{M} \phi_q(x_m), \quad q = 0,\dots,M \\
\phi_q(x) = x^q, \quad q = 0,\dots,M
\end{gather*}

Now define $\widetilde{\Phi} \to \mathbb{R}^M$ by

\begin{gather*}
\widetilde{\Phi}_q(X) = \sum_{m=1}^{M} \widetilde{\phi}_q(x_m), \quad q = 1,\dots,M \\
\widetilde{\phi}_q(x) = x^q, \quad q = 1,\dots,M.
\end{gather*}

Note that $\Phi_0(X)=M$ for all $X$, so $\text{Im}(\Phi) = \{M\} \times \text{Im}(\widetilde{\Phi})$. Since $\{M\}$ is a singleton, these two images are homeomorphic, with a homeomorphism given by:

\begin{gather*}
\gamma : \text{Im}(\widetilde{\Phi}) \to \text{Im}(\Phi) \\
\gamma(x_1,\dots,x_M) = (M, x_1, \dots, x_M)
\end{gather*}

Now by definition, $\widetilde{\Phi} = \gamma^{-1} \circ \Phi$. Since this is a composition of homeomorphisms, $\widetilde{\Phi}$ is also a homeomorphism. Therefore $(f \circ \widetilde{\Phi}^{-1}, \widetilde{\phi})$ is a continuous sum-decomposition of $f$ via $\mathbb{R}^M$. 
\end{proof}

\subsection{Proof of \Cref{thm:arbitrary_set_sizes}}

\arbitrary*

\begin{proof}
We use the adapted sum-of-power mapping $\widetilde{\Phi}$ from above, denoted in this section by $\Phi$
\begin{gather*}
\Phi_q(X) = \sum_{m=1}^{M} \phi_q(x_m), \quad q = 1,\dots,M \\
\phi_q(x_m) = (x_m)^q, \quad q = 1,\dots,M
\end{gather*}
which is shown above to be injective (up to reordering input sets). We separate $\Phi_q(X)$ into two terms:
\begin{equation}
    \Phi_q(X) = \sum_{m=1}^{M'} \phi_q(x_m) + \sum_{m=M'+1}^{M} \phi_q(x_m)
\end{equation}
For an input set $X$ with $M'=M-P$ elements, with $P \geq 0$, we say that the set contains $M'$ ``actual elements'' as well as $P$ ``empty" elements which are not part of the input set.
Those $P$ ``empty elements'' can be regarded as place fillers when the size of the input set is smaller than $M$, i.e. when $M' < M$.

We fill in these $P$ elements with a constant value $k \notin [0, 1]$, preserving the injectivity of $\Phi_q(X)$ for input sets $X$ of arbitrary size $M'$:
\begin{equation}
\label{eq:split_e}
    \Phi_q(X) = \sum_{m=1}^{M'} \phi_q(x_m) + \sum_{m=M'+1}^{M} \phi_q(k)
\end{equation}

\Cref{eq:split_e} is no longer strictly speaking a sum-decomposition.
This can be fixed by re-arranging the expression:
\begin{equation}
\label{eq:split_e_rearranged}
    \begin{split}
    \Phi_q(X) & = \sum_{m=1}^{M'} \phi_q(x_{m}) + \sum_{m=M'+1}^{M} \phi_q(k) \\
           & = \sum_{m=1}^{M'} \phi_q(x_{m}) + \sum_{m=1}^{M} \phi_q(k) - \sum_{m=1}^{M'} \phi_q(k) \\
           & = \sum_{m=1}^{M'} \left[ \phi_q(x_m) - \phi_q(k) \right] + \sum_{m=1}^{M} \phi_q(k)
    \end{split}
\end{equation}

The last term in \Cref{eq:split_e_rearranged} is a constant value which only depends on the choice of $k$ and is independent of $X$ and $M'$.
Hence, we can replace $\phi_q(x)$ by $\widehat{\phi_q}(x)=\phi_q(x)-\phi_q(k)$.
This leads to a new sum-of-power mapping $\widehat{\Phi}_q(X)$ with
\begin{equation}
\label{eq:new_e}
    \begin{split}
        \widehat{\Phi}_q(X) & = \sum_{m=1}^{M'} \widehat{\phi}_q(x_m)  \\
        & = \Phi_q(X) - M\cdot\phi_q(k).
    \end{split}
\end{equation}

$\widehat{\Phi}$ is injective since $\Phi$ is injective and the last term in the above sum is constant. $\widehat{\Phi}$ is also in the form of a sum-decomposition.

For each $m < M$, we can follow the reasoning used in the rest of the proof of \Cref{ori_uncountable_theorem} to note that $\widehat{\Phi}$ is a homeomorphism when restricted to sets of size $m$ -- we denote these restricted functions by $\widehat{\Phi}_m$.
Now each $\widehat{\Phi}_m^{-1}$ is a continuous function into $\mathbb{R}^m$. We can associate with each a continuous function $\widehat{\Phi}_{m,M}^{-1}$ which maps into $\mathbb{R}^M$, with the $M-m$ tailing dimensions filled with the value $k$.

Now the domains of the $\widehat{\Phi}_{m,M}^{-1}$ are compact, since the domain of $\widehat{\Phi}_{m,M}^{-1}$ is just the image of the compact set $[0, 1]$ under the continuous function $\widehat{\Phi}_m$. The domains of the $\widehat{\Phi}_{m,M}^{-1}$ are also disjoint, since $k \notin [0, 1]$. We can therefore find a function $\widehat{\Phi}_{C}^{-1}$ which is continuous on $\mathbb{R}^M$ and agrees with each $\widehat{\Phi}_{m,M}^{-1}$ on its domain.

To complete the proof, let $\mathcal{Y}$ be a connected compact set with $k\in\mathcal{Y}, [0, 1]\subset\mathcal{Y}$. Let $\widehat{f}$ be a function on subsets of $\mathcal{Y}$ of size exactly $M$ satisfying

\begin{gather*}
    \widehat{f}(X) = f(X); \quad X \subset [0,1] \\
    \widehat{f}(X) = f(X \cap [0, 1]); \quad X \subset [0, 1] \cup \{k\}.
\end{gather*}

We can choose $\widehat{f}$ to be continuous because $f$ is continuous.
Then $(\widehat{f}\circ\widehat{\Phi}_{C}^{-1}, \widehat{\phi})$ is a continuous sum-decomposition of $f$.

\end{proof}

\subsection{Max-Decomposition}
\label{app:max-decomp}

Analogously to sum-decomposition, we define the notion of \emph{max-decomposition}. A function $f$ is max-decomposable if there are functions $\rho$ and $\phi$ such that

\begin{displaymath}
f(\mathbf{x}) = \rho \Bigl( \texttt{max}_i \big( \phi(x_i) \big) \Bigr).
\end{displaymath}

where the \texttt{max} is taken over each dimension independently in the latent space. Our definitions of decomposability via $Z$ and continuous decomposability also extend to the notion of max-decomposition.

We now state and prove a theorem which is closely related to \Cref{thm:max_not_decomposable}, but which establishes limitations on max-decomposition, rather than sum-decomposition.

\begin{theorem}
\label{thm:sum_not_decomposable}
Let $M > N \in \mathbb{N}$. Then there exist permutation invariant continuous functions $f : \mathbb{R}^M \to \mathbb{R}$ which are not max-decomposable via $\mathbb{R}^N$.
\end{theorem}

Note that this theorem rules out any max-decomposition, whether continuous or discontinuous. We specifically demonstrate that summation is not max-decomposable.

\begin{proof}
Consider $f(\mathbf{x}) = \sum_{i=1}^M x_m$. Let $\phi : \mathbb{R} \to \mathbb{R}^N$, and let $\mathbf{x} \in \mathbb{R}^M$ such that $x_i \neq x_j$ when $i \neq j$.

For $n=1,\dots,N$, let $\mu(n) \in \{1,\dots,M\}$ such that

\begin{displaymath}
\texttt{max}_i \big( \phi(x_i)_n \big) = \phi(x_{\mu(n)})_n.
\end{displaymath}

That is, $\phi(x_{\mu(q)})$ attains the maximal value in the $q$-th dimension of the latent space among all $\phi(x_i)$. Now since $N < M$, there is some $m \in \{1,\dots,M\}$ such that $\mu(n) \neq m$ for any $n \in \{1,\dots,N\}$. So now consider $\widetilde{\mathbf{x}}$ defined by:

\begin{gather}
\widetilde{x}_i = x_i ; i \neq m \\
\widetilde{x}_m = x_{\mu(1)}
\end{gather}

Then:

\begin{displaymath}
\texttt{max}_i \big( \phi(x_i) \big) = \texttt{max}_i \big( \phi(\widetilde{x}_i) \big)
\end{displaymath}

But since we chose $\mathbf{x}$ such that all $x_i$ were distinct, we have $\sum_{i=1}^M x_i \neq \sum_{i=1}^M \widetilde{x}_i$ by the definition of $\widetilde{\mathbf{x}}$. This shows that $\phi$ cannot form part of a max-decomposition for $f$. But $\phi$ was arbitrary, so no max-decomposition exists.

\end{proof}

\subsection{Function Approximation}
\label{sec:approximation_appendix}

\subsubsection{A stronger statement of \Cref{thm:main_approximation_theorem}}

Given a function $f : [\smn1, 1]^M \to \mathbb{R}$, write $V_f$ for the \emph{variation} of $f$:

\begin{equation*}
    V_f = \underset{\mathbf{x}}{\texttt{max}} \big( f(\mathbf{x}) \big) - \underset{\mathbf{x}}{\texttt{min}} \big( f(\mathbf{x}) \big)
\end{equation*}

\edef\mythmcount{\value{theorem}}
\setcounterref{theorem}{thm:main_approximation_theorem}
\addtocounter{theorem}{-1}
\begin{theorem}[Strong Form]
Let $M,N \in \mathbb{N}$ with $M>N$. There exists a continuous, non-constant, permutation-invariant function $f: [\smn1, 1]^M \to \mathbb{R}$ such that:

\begin{enumerate}
    \item $f$ is affine on $\Delta_M$.
    \item For any continuous sum-decomposition $\rho \circ \Phi$ via $\mathbb{R}^N$, there is some $\mathbf{x} \in [\smn1,1]^M$ with $|\rho \big( \Phi(\mathbf{x}) \big) - f(\mathbf{x})| \geq \frac{V_f}{2}$.
\end{enumerate}

\end{theorem}
\setcounter{theorem}{\mythmcount}

\subsubsection{Completing the proof of Lemma \ref{lem:nu_n}}

\edef\mythmcount{\value{theorem}}
\setcounterref{theorem}{lem:nu_n}
\addtocounter{theorem}{-1}
\lemmanu*
\setcounter{theorem}{\mythmcount}
\begin{proof}

We proceed by induction. We state our induction hypotheses as follows, in the order in which we prove the corresponding conclusions:

\begin{enumerate}[leftmargin=7\parindent,label=\textbf{Hypothesis \arabic*},ref=\arabic*]
\item $\nu_{n-1}$ is continuous. \label{hyp:continuity}
\item Let $\mathbf{x} \in I_{n-1}$ and $1\leq j<n-1$. Then $\nu_{n-1}(\mathbf{x})_j \geq \nu_{n-1}(\smn\mathbf{x})_{j+1}$. \label{hyp:technical}
\item The codomain of $\nu_{n-1}$ is $\Delta_{n-1}$. \label{hyp:codomain}
\item Let $\mathbf{x} \in \partial I_{n-1}$. Then $\Gamma_{n-1} \big( \nu_{n-1}( \smn\mathbf{x} ) \big) = -\Gamma_{n-1} \big( \nu_{n-1}( \mathbf{x} ) \big)$. \label{hyp:main}
\end{enumerate}

Hypotheses \labelcref{hyp:continuity,hyp:codomain,hyp:main} come from the statement of the lemma. Hypothesis \ref{hyp:technical} is a technical condition which is required for the proof. We will refer to these statements as ``hypotheses'' when referring to the $n-1$ case, and ``conclusions'' when referring to the $n$ case.

\subsubsection*{Base case $n=1$}

Let $\nu_1(\mathbf{x}) := \mathbf{x}$. Conclusions \labelcref{hyp:continuity,hyp:codomain} are trivial. Conclusion \ref{hyp:technical} is also trivial -- there is no $j$ with $1 \leq j < 1$, and so there are no inequalities to satisfy. $\partial I_1 = \{ \smn1, 1 \}$, so conclusion \ref{hyp:main} reduces to checking that $\Gamma_1(\smn1) = -\Gamma_1(1)$. This is true by definition of $\Gamma$.

\subsubsection*{Inductive step}

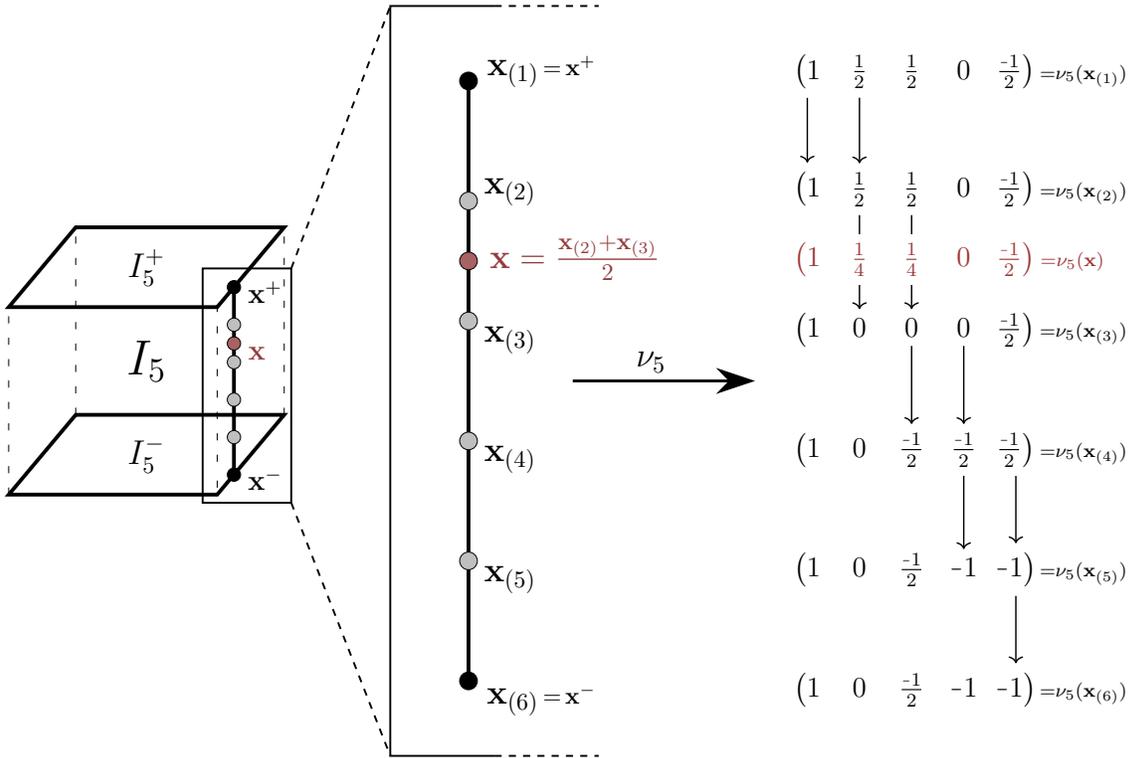
\begin{figure}
    \centering
    \resizebox{\textwidth}{!}{\input{figures/interpolation_cartoon_2.tikz}}
    \captionsetup{parindent=17pt}
    \caption[lof entry]{An example of how the interpolation scheme works for $n=5$, with $\mathbf{x}$ chosen to lie on a surface vertical. The specific values of $\nu_5(\mathbf{x}_{(1)})$ and $\nu_5(\mathbf{x}_{(6)})$ shown here are purely illustrative, but the remaining values are computed according to our interpolation scheme. \Cref{eq:interpolation_scheme} directly implies that the value of $\nu_n$ at successive endpoints can differ only along the indicated vertical arrows -- for instance, $\nu_n(\mathbf{x}_{(1)})$ and $\nu_n(\mathbf{x}_{(2)})$ can differ only in the first and second places. One perspective on this is to view the function values $\nu_n(\mathbf{x}_{(i)})_j$ as the entries of a matrix $\mathcal{N}_{ij}$ -- \Cref{eq:interpolation_scheme} fixes the above- and below-diagonal entries of $\mathcal{N}$ to match $\mathbf{y}^+$ and $\mathbf{y}^-$ respectively, so values can only change across the diagonal. Crossing the diagonal corresponds to moving along the vertical arrows shown above.
    
    Part 4 of the proof of Lemma \ref{lem:nu_n} demonstrates that, along a surface vertical, the values of $\nu_n$ either do not change from endpoint to endpoint, as seen here between $\mathbf{x}_{(1)}$ and $\mathbf{x}_{(2)}$, or they change in equal pairs, as seen here between $\mathbf{x}_{(2)}$ and $\mathbf{x}_{(3)}$. In either case, the value of $\Gamma_n \circ \nu_n$ is not affected -- trivially in the former case, and in the latter case because equal pairs cancel to $0$ in the alternating sum which defines $\Gamma_n$ (Equation \ref{eq:def_of_gamma}). Importantly, $\Gamma_n \circ \nu_n$ is constant not only at consecutive endpoints, but also when linearly interpolating between the endpoints. In the above, for instance, $\Gamma_5\big(\nu_5(\mathbf{x})\big)=\Gamma_5\big(\nu_5(\mathbf{x}_{(2)})\big)=\Gamma_5\big(\nu_5(\mathbf{x}_{(3)})\big)$, because the second and third coordinates are equal and cancel to 0 under $\Gamma_n$. To show that this $\Gamma$-equivalence property always holds, we rely on $\nu_n$ being sufficiently ``well-behaved'' on $I_n^+$ and $I_n^-$. In particular we will see that \Cref{eq:y_minus_geq_y_plus}, which relies on hypothesis \ref{hyp:technical}, is crucial.
    
    Here we have illustrated the interpolation scheme along a surface vertical. The same interpolation scheme applies across all of $I_n$, with the same behaviour of changing only along the vertical arrows. However, the $\Gamma$-equivalence property discussed above only holds on surface verticals, and does not hold on the rest of $I_n$.}
    \label{fig:interpolation_cartoon}
\end{figure}

\noindent\textbf{Conclusion \ref{hyp:continuity} -- Defining the continuous function $\nu_n$}

\noindent We define $\nu_n^+ : I_n^+ \to I_n$, $\nu_n^- : I_n^- \to I_n$ as follows.

\begin{align*}
\nu_n^+(\mathbf{x}) &:= \big( 1, \nu_{n-1}(\overline{\mathbf{x}}) \big) \\
\nu_n^-(\mathbf{x}) &:= \big( \nu_{n-1}(\smn\overline{\mathbf{x}}), \smn1 \big) = \alpha\big( \nu_n^+(\smn\mathbf{x}) \big) \\
\end{align*}

These definitions follow the partial definition of $\nu_n$ given in \Cref{sec:approximation_proof}. $\nu_n^+$ and $\nu_n^-$ are both continuous by continuity of $\nu_{n-1}$ (hypothesis \ref{hyp:continuity}). We will define $\nu_n$ to agree with $\nu_n^+$ and $\nu_n^-$ on their domains, and extend $\nu_n$ to the rest of $I_n$ by interpolating between $\nu_n^+$ and $\nu_n^-$. As discussed in \Cref{sec:approximation_proof}, any point $\mathbf{x} \in I_n$ lies on a line running parallel to the $n$-th coordinate axis, and this line meets $I_n^+, I_n^-$ at $\mathbf{x}^+, \mathbf{x}^-$ respectively. We interpolate along these ``vertical'' lines to define $\nu_n$ on the rest of $I_n$. Simple linear interpolation does not define a function with the right properties, so we instead define the following piecewise linear interpolation scheme.

Divide the straight line segment from $\mathbf{x}^+$ to $\mathbf{x}^-$ into $n$ sections of equal length, with endpoints $\mathbf{x}_{(1)}, \ldots, \mathbf{x}_{(n+1)}$ (so $\mathbf{x}_{(1)} = \mathbf{x}^+$ and $\mathbf{x}_{(n+1)} = \mathbf{x}^-$). We fix the value of $\nu_n$ at each endpoint and linearly interpolate between them. For compactness, we write $\mathbf{y}^+ = \nu_n^+(\mathbf{x}^+)$ and $\mathbf{y}^- = \nu_n^-(\mathbf{x}^-)$, and make the following definition:

\begin{equation*}
m_j := \texttt{median}(y^+_j, y^-_{j-1}, y^-_j)
\end{equation*}

Then we fix $\nu_n$ at the endpoints as follows:

\begin{equation}
\label{eq:interpolation_scheme}
\nu_n(\mathbf{x}_{(i)})_j := 
    \begin{cases}
    y^+_j & i < j \\
    m_j & i = j \\
    y^-_j & i > j
    \end{cases}
\end{equation}

The special case $i = j = 1$ depends on $\nu_n^-(\mathbf{x}^-)_0$, which is undefined -- for consistency with $\nu_n^+$, we must set $\nu_n(\mathbf{x}_{(1)})_1 = y^+_1$. \Cref{fig:interpolation_cartoon} gives an example of how this interpolation scheme works for $n=5$.

This interpolation scheme defines a continuous function $\nu_n$ on $I_n$, by continuity of $\nu_n^+$, $\nu_n^-$, \texttt{median}, and linear interpolation. Conclusion \ref{hyp:continuity} is thus established.

In the rest of the proof, we will often use the more verbose notation when referring to $\mathbf{y}^\pm$, writing $\nu_n^\pm(\mathbf{x}^\pm)$ to make the dependence on $\mathbf{x}$ explicit. We will use the compact notation $\mathbf{y}^\pm$ where this dependence is less important.

\vspace{1em}
\noindent\textbf{Conclusion \ref{hyp:technical} -- A useful inequality}

\noindent First we show that for every $j$, we have

\begin{equation}
\label{eq:nuplus_geq_numinus}
\begin{aligned}
y_j^+ &\geq y_j^-.
\end{aligned}
\end{equation}

By definition:

\begin{align*}
y^+_j = \nu_n(\mathbf{x}^+)_j &= 
    \begin{cases}
    1 & j = 1 \\
    \nu_{n-1}(\overline{\mathbf{x}})_{j-1} & j > 1
    \end{cases}
\\
y^-_j = \nu_n(\mathbf{x}^-)_j &= 
    \begin{cases}
    \nu_{n-1}(\smn\overline{\mathbf{x}})_j & j < n \\
    \smn1 & j = n
    \end{cases}
\end{align*}

\Cref{eq:nuplus_geq_numinus} is therefore trivial for $j=1$ and $j=n$, since hypothesis \ref{hyp:codomain} implies that $\nu_{n-1}$ is bounded between $\smn1$ and $1$. For $1 < j < n$, \Cref{eq:nuplus_geq_numinus} is exactly hypothesis \ref{hyp:technical}. Note also from the above definitions that 

\begin{equation}
\label{eq:nu_xplus_xminus_equal}
\nu_n\big(\mathbf{x}^-\big)_j = \nu_n\big( (\smn\mathbf{x})^+ \big)_{j+1} \quad\quad 1 \leq j < n.
\end{equation}

\Cref{eq:nuplus_geq_numinus} implies, by definition of $m_j$, that $y^+_j \geq m_j \geq y_j^-$ for all $j$. Therefore by \Cref{eq:interpolation_scheme}:

\begin{equation}
\label{eq:descending_interpolation}
\nu_n(\mathbf{x}_{(i)})_j \geq \nu_n(\mathbf{x}_{(i+1)})_j \quad\quad 1 \leq i \leq n, \text{ all }j
\end{equation}

This implies that for all $\mathbf{x}$ and for all $j$, we have

\begin{equation*}
\nu_n(\mathbf{x}^+)_j \geq \nu_n(\mathbf{x})_j \geq \nu_n(\mathbf{x}^-)_j.
\end{equation*}

Therefore for $1 \leq j < n$, by \Cref{eq:nu_xplus_xminus_equal}, we have

\begin{align*}
\nu_n(\mathbf{x})_j &\geq \nu_n(\mathbf{x}^-)_j \\
&= \nu_n((\smn\mathbf{x})^+)_{j+1} \\
&\geq \nu_n(\smn\mathbf{x})_{j+1}.
\end{align*}

This final inequality is exactly conclusion \ref{hyp:technical}.

\clearpage
\noindent\textbf{Conclusion \ref{hyp:codomain} -- The codomain of $\nu_n$ is $\Delta_n$}

\noindent We must show that $\nu_n(\mathbf{x})_j$ is a descending sequence in $j$. On $I_n^+$ and $I_n^-$, this follows immediately by definition and by hypothesis \ref{hyp:codomain}. Away from $I_n^+$ and $I_n^-$, we must consider the interpolation scheme. We want $\nu_n(\mathbf{x})_j \geq \nu_n(\mathbf{x})_{j+1}$ for all $j$. We only need to check this inequality at the endpoints $\mathbf{x}_{(i)}$ of the interpolation scheme -- if it holds at the endpoints, then the function value at any other point is a convex combination of points from $\Delta_n$, and $\Delta_n$ is convex.

The relevant inequality is therefore as follows:

\begin{equation}
\label{eq:descending_in_j}
\nu_n(\mathbf{x}_{(i)})_j \geq \nu_n(\mathbf{x}_{(i)})_{j+1} \quad \quad \text{all } i,  1 \leq j < n
\end{equation}

To prove this, we will make use of the following inequality, which holds for $1 < j < n$:

\begin{equation}
\label{eq:y_minus_geq_y_plus}
y^-_{j-1} \geq y^+_{j+1}
\end{equation}

Expanding definitions, $y^-_{j-1} = \nu_{n-1}(\smn\overline{\mathbf{x}})_{j-1}$ and $y^+_{j+1} = \nu_{n-1}(\overline{\mathbf{x}})_j$. \Cref{eq:y_minus_geq_y_plus} therefore follows directly from hypothesis \ref{hyp:technical}.

To prove \Cref{eq:descending_in_j}, there are five cases to check. For case 4, we make use of the fact that $a \geq \texttt{median}(a, b, c)$ if and only if $a \geq b$ or $a \geq c$. For case 5, we make use of the fact that $\texttt{median}(a, b, c) \geq \texttt{min}(a, b)$.

\begin{enumerate}[leftmargin=3\parindent]
\item $i < j$. Expanding definitions, we need

\begin{equation*}
y^+_j \geq y^+_{j+1}.
\end{equation*} 

This is the $I_n^+$ case, which as already noted is true by definition and induction.

\item $i=j=1$. This is also the $I_n^+$ case.

\item $i > j+1$. Similarly, this is the $I_n^-$ case.

\item $i = j+1$. Expanding definitions, we need

\begin{equation*}
y^-_j \geq \texttt{median}\big( y^+_{j+1}, y^-_j, y^-_{j+1} \big).
\end{equation*}

This is true because $y^-_j \geq y^-_{j+1}$, which is the $I_n^-$ case.

\item $i=j>1$. Expanding definitions, we need 

\begin{equation*}
\texttt{median}\big( y^+_j, y^-_{j-1}, y^-_j \big) \geq y^+_{j+1}
\end{equation*}

This follows because $y^+_{j+1} \leq \texttt{min}(y^-_{j-1}, y^+_j)$. \Cref{eq:y_minus_geq_y_plus} implies $y^+_{j+1} \leq y^-_{j-1}$, and $y^+_{j+1} \leq y^+_j$ is the $I_n^+$ case.
\end{enumerate}

\Cref{eq:descending_in_j} therefore holds, so conclusion \ref{hyp:codomain} is established.

\vspace{1em}
\noindent\textbf{4. Conclusion \ref{hyp:main} -- The Borsuk-Ulam condition}

\noindent In this section we will freely make use of the following inequalities, which we have already shown to be true:

\begin{align*}
y^+_j &\geq y^+_{j+1} &
y^-_j &\geq y^-_{j+1} \\
y^+_j &\geq y^-_j &
y^-_{j-1} &\geq y^+_{j+1}
\end{align*}

The top inequalities hold because $\nu_n(\mathbf{x}) \in \Delta_n$ (conclusion \ref{hyp:codomain}). The bottom left inequality is \Cref{eq:nuplus_geq_numinus}. The bottom right inequality is \Cref{eq:y_minus_geq_y_plus}.

Conclusion \ref{hyp:main} applies to $\mathbf{x} \in \partial I_n$. As shown in \Cref{sec:approximation_proof}, for $\mathbf{x} \in I_n^+ \cup I_n^-$, we have

\begin{align*}
\Gamma_n\big( \nu_n( \smn \mathbf{x} ) \big) &= -\Gamma_n\big( \nu_n( \mathbf{x} ) \big).
\end{align*}

The remaining points on the surface are the points $\mathbf{x} \in \partial I_n \backslash (I_n^+ \cup I_n^+)$, i.e. points lying on surface verticals, and for the remainder of the proof we consider only these points. We claim that $\Gamma_n(\nu_n(\mathbf{x})) = \Gamma_n(\nu_n(\mathbf{x}^+))$, i.e. that $\Gamma_n \circ \nu_n$ is constant along surface verticals. As noted in \Cref{sec:approximation_proof}, this will imply that conclusion \ref{hyp:main} holds on all of $\partial I_n$. To prove our claim, we first define the relation of \emph{$\Gamma$-equivalence} on $\Delta_n$, denoted by $\sim_\Gamma$. We define $\mathbf{x}_1 \sim_\Gamma \mathbf{x}_2$ if and only if $\Gamma_n(\mathbf{x}_1) = \Gamma_n(\mathbf{x}_2)$ for every choice of $\phi$. For an alternative view on this relation, given $\mathbf{x} \in \Delta_n$, form $\mathbf{x}_\Gamma$ by deleting pairs of equal coordinates until no matching pairs remain. For example:

\begin{equation*}
(1, \tfrac{1}{2}, \tfrac{1}{2}, \tfrac{1}{2}, 0, \smn\tfrac{1}{2}, \smn\tfrac{1}{2})_\Gamma = (1, \tfrac{1}{2}, 0)
\end{equation*}

Then $\mathbf{x}_1 \sim_\Gamma \mathbf{x}_2$ if and only if $(\mathbf{x}_1)_\Gamma = (\mathbf{x}_2)_\Gamma$. The reverse implication follows from the alternating sum form of $\Gamma$, so that consecutive equal coordinates cancel to $0$. For the forward implication, if $(\mathbf{x}_2)_\Gamma \neq (\mathbf{x}_2)_\Gamma$, then there is some coordinate $x_i$ which appears in $(\mathbf{x}_1)_\Gamma$ but not in $(\mathbf{x}_2)_\Gamma$. Choosing $\phi$ to be non-zero at this co-ordinate, and zero at every other coordinate of $(\mathbf{x}_1)_\Gamma$ and $(\mathbf{x}_2)_\Gamma$, we obtain $\Gamma_n(\mathbf{x}_1) \neq \Gamma_n(\mathbf{x}_2)$.

Our claim is that the relation $\sim_\Gamma$ is preserved under the interpolation between $\nu_n(\mathbf{x}^+)$ and $\nu_n(\mathbf{x}^-)$. We will prove this by showing that the successive endpoints are $\Gamma$-equivalent, and that the interpolation between any two endpoints preserves $\Gamma$-equivalence.

Consider two successive endpoints in our interpolation scheme, $\mathbf{x}_{(I)}$ and $\mathbf{x}_{(I+1)}$. By definition, we have $\nu_n(\mathbf{x}_{(I)})_j = \nu_n(\mathbf{x}_{(I+1)})_j$ if $I > j$ or $I+1 < j$. This leaves $I=j$ or $I+1=j$. That is, $\nu_n(\mathbf{x}_{(I)})$ can differ from $\nu_n(\mathbf{x}_{(I+1)})$ only in the $I$-th and $I+1$-th coordinates. From the definition of the interpolation scheme, the relevant quantities are as follows:

\begin{align*}
\nu_n(\mathbf{x}_{(I)})_I &= m_I \\
\nu_n(\mathbf{x}_{(I)})_{I+1} &= y^+_{I+1} \\
\nu_n(\mathbf{x}_{(I+1)})_I &= y^-_I \\
\nu_n(\mathbf{x}_{(I+1)})_{I+1} &= m_{I+1}
\end{align*}

There are two special cases here. When $I=1$, $\nu_n(\mathbf{x}_{(1)})_1=y_1^+$ instead of $m_1$. When $I=n$, the quantities $y^+_{n+1}$ and $m_{n+1}$ are not defined. For now, we assume $1<I<n$, and we address the special cases at the end of the proof.

We have already shown the following inequalities between the four quantities above -- the left-to-right inequalities are because $\nu_n(\mathbf{x}) \in \Delta_n$, and the top-to-bottom inequalities are instances of \Cref{eq:descending_interpolation}.

\begin{eqnarray}
m_I & \geq & y^+_{I+1} \nonumber \\
\vgeq & & \vgeq \label{eq:inequality_square} \\
y^-_I & \geq & m_{I+1} \nonumber
\end{eqnarray}

Recalling the visualisation of \Cref{fig:interpolation_cartoon}, the four quantities above correspond to two pairs of points joined by vertical arrows. For example, $I=3$ corresponds to the transition between $\nu_n(\mathbf{x}_{(3)})$ and $\nu_n(\mathbf{x}_{(4)})$.

Another way of seeing the above inequalities is to consider the values $\nu_n(\mathbf{x}_{(i)})_j$ as a matrix $\mathcal{N}_{ij}$. The matrix $\mathcal{N}$ is (non-strictly) decreasing downwards and decreasing to the right. In the above array of four values, we wish to show that either both columns are equal, or both rows are equal. That is, we wish to show that one of the following two cases holds:

\begin{enumerate}[leftmargin=3\parindent]
\item $m_I = y^+_{I+1}$ and $m_{I+1} = y^-_I$ (columns of Equation \ref{eq:inequality_square} are equal)
\item $m_I = y^-_I$ and $m_{I+1} = y^+_{I+1}$ (rows of Equation \ref{eq:inequality_square} are equal)
\end{enumerate}

In case 1, linear interpolation between endpoints implies that $\nu_n(\mathbf{x})_I = \nu_n(\mathbf{x})_{I+1}$ for every $\mathbf{x}$ on the line segment between $\mathbf{x}_{(I)}$ and $\mathbf{x}_{(I+1)}$. Since these values are all equal in every other coordinate, and since $\Gamma$-equivalence is not affected by adding a pair of equal coordinates, all of the values $\nu_n(\mathbf{x})$ for $\mathbf{x}$ on this section of the surface vertical are $\Gamma$-equivalent.
In case 2, we have $\nu_n(\mathbf{x}_{(I)}) = \nu_n(\mathbf{x}_{(I+1)})$. Linear interpolation between endpoints therefore trivially gives $\Gamma$-equivalence, because $\nu_n$ is constant between the endpoints $\mathbf{x}_{(I)}$ and $\mathbf{x}_{(I+1)}$. So in either case, $\Gamma$-equivalence is preserved by the interpolation, and therefore $\Gamma_n \circ \nu_n$ is constant along surface verticals.

To prove that either case holds, we just need to show that, in Equation (\ref{eq:inequality_square}), either the top or the left equality holds, and either the bottom or the right equality holds. If this were true with neither the rows nor the columns being equal, then we would be in one of the following contradictory situations:

\begin{align*}
m_I & = y^+_{I+1} & m_I & > y^+_{I+1} \\
\vgreat~ & \phantom{=} \quad\vequal & \vequal~~ & \phantom{=} \quad\vgreat\\
y^-_I & > m_{I+1} & y^-_I & = m_{I+1}
\end{align*}

By definition of $m_{I+1}$ and our inequalities on the $y_I$, it is immediate that we have either the bottom or the right equality -- either $m_{I+1} = y^-_I$, or $m_{I+1} = y^+_{I+1}$.

The case for the top and left equalities is more complex. By definition of $m_I$, and our inequalities on the $y_I$, we have that $m_I = y^-_I$ if either $y^-_I = y^-_{I-1}$ or $y^-_I = y^+_I$. We have $m_I = y^+_{I+1}$ if either $y^+_{I+1} = y_I^+$ or $y^+_{I+1} = y^-_{I-1}$. So for neither the top nor the left equality to hold, all four of these statements must be false. That is, applying our known inequalities for the $y_I$, we have:

\begin{align}
\label{eq:false_inequalities}
\begin{aligned}
y^+_I &> y^-_I \\
y^-_{I-1} &> y^-_I
\end{aligned}
&&
\begin{aligned}
y^+_I &> y^+_{I+1} \\
y^-_{I-1} &> y^+_{I+1}
\end{aligned}
\end{align}

In the following representation (where $a = \text{min}(y^+_I, y^-_{I-1})$), these inequalities say that everything to the left of the diagonal line is strictly greater than everything to the right of the diagonal line. The special case $I=1$ results in essentially the same picture, with $y_1^+$ being the only point to the left of the diagonal line.

\begin{center}
\begin{tikzpicture}[font=\Large]
\matrix [{matrix of math nodes},
         column sep={0em}, row sep=2em
         ,nodes={inner sep=0em, outer sep=0em}
        ] (m)
 {
    y^+_1 & \ldots &   \ldots   &[1em] y^+_I &[1em] y^+_{I+1} & \ldots & ~y^+_n \\
    y^-_1 & \ldots & ~y^-_{I-1} &      y^-_I &       \ldots   & \ldots & ~y^-_n \\
 };

\draw[thick] (m-2-3.south east) -- (m-1-5.north west);

\fill[pattern color=edred,opacity=.2,pattern=south east lines] (m-2-3.south east) -- (m-1-5.north west) -- (m-1-1.north west) -- (m-2-1.south west) -- cycle;
\fill[edblue,opacity=.075] (m-2-3.south east) -- (m-1-5.north west) -- (m-1-7.north east) -- (m-2-7.south east) -- cycle;

\node at ([yshift=0.8em]m-1-5.north west) (anode) {$\mathbf{a}$};
\node at ([xshift=-5em]anode) [color=edred]  {$\textbf{$\geq$} \,\mathbf{a}$};
\node at ([xshift=3em]anode)  [color=edblue!75!black] {$\textbf{$<$} \,\mathbf{a}$};

\end{tikzpicture}
\end{center}

This implies that there are exactly $I$ coordinates of $\mathbf{y}^+$ which are at least $a$, and exactly $I-1$ coordinates of $\mathbf{y}^-$ which are at least $a$. But now recall that we have shown in \Cref{sec:approximation_proof}, specifically \Cref{eq:gamma_equivalent_each_end}, that

\begin{equation*}
\Gamma_n\big( \nu_n( \mathbf{x}^+ ) \big) = \Gamma_n\big( \nu_n( \mathbf{x}^- ) \big).
\end{equation*}

Since this equation holds regardless of the choice of $\phi$, it is equivalent to the statement that $\mathbf{y}^+ \sim_\Gamma \mathbf{y}^-$. We now note the basic but crucial fact that if $\mathbf{y}^+ \sim_\Gamma \mathbf{y}^-$, then for any $a$, the number of occurrences of $a$ in $\mathbf{y}^+$ must have the same parity as the number of occurrences of $a$ in $\mathbf{y}^-$. This parity matching is also true if we consider the number of occurrences of values which are at least $a$. But we have just shown that these numbers have differing parity. This is a contradiction, so \Cref{eq:false_inequalities} must fail. As shown above, this implies that either $m_I = y^+_{I+1}$ or $m_I = y^-_I$, and this in turn implies that either the rows or the columns of Equation (\ref{eq:inequality_square}) must be equal. Therefore $\Gamma \circ \nu$ is constant along the line segment from $\mathbf{x}_{(I)}$ to $\mathbf{x}_{(I+1)}$.

As noted above, there is a final special case to consider, $I=n$. We have shown that $\Gamma \circ \nu$ is constant along the surface vertical between $\mathbf{x}_{(1)}$ and $\mathbf{x}_{(n)}$, but we have not shown that it is constant along the final section of the surface vertical, from $\mathbf{x}_{(n)}$ to $\mathbf{x}_{(n+1)}$. However, we have shown that $\nu_n(\mathbf{x}_{(1)}) \sim_\Gamma \nu_n(\mathbf{x}_{(n)})$ (by transitivity of $\sim_\Gamma$), and that $\nu_n(\mathbf{x}_{(1)}) \sim_\Gamma \nu_n(\mathbf{x}_{(n+1)})$. Applying transitivity of $\sim_\Gamma$ again, this implies that $\nu_n(\mathbf{x}_{(n)}) \sim_\Gamma \nu_n(\mathbf{x}_{(n+1)})$. We also know that $\nu_n(\mathbf{x}_{(n)})$ and $\nu_n(\mathbf{x}_{(n+1)})$ differ in at most one place (the $n$-th coordinate). It is easily seen that if two points are $\Gamma$-equivalent, they must differ in an even number of places. Therefore $\nu_n(\mathbf{x}_{(n)}) = \nu_n(\mathbf{x}_{(n+1)})$, so $\Gamma \circ \nu$ is constant along the final section of the surface vertical, and conclusion \ref{hyp:main} holds.

\end{proof}

%% file: figures/rotate1.tikz
\begin{tikzpicture}[point1/.style={draw, circle, very thick, inner sep=0.5em,fill=edblue!60!white},
                        point2/.style={draw=edred,circle, very thick, inner sep=0.5em,pattern=south east lines, pattern color=edred},
                        mat/.style={matrix of math nodes, font=\large, left delimiter=(, right delimiter=),
                                    inner sep=0.05em, row sep=0.5em, column sep=0.25em,
                                    row 1/.style={nodes={color=edblue}},
                                    row 2/.style={nodes={color=edred}}}]
    \tikzmath{\x1 = 135;}
    \tikzmath{\x2 = \x1 + 180;
              \start1 = \x1 - 20;
              \xend1 = \start1 - 25;
              \start2 = \x2 - 20;
              \xend2 = \start2 - 25;}
    
    \node[point1] at (\x1:1.4) {1};
    \node[point2] at (\x2:1.4) {2};
    
    \matrix[mat, anchor=north] (m) at (0, -2.5) {
            \smn\text{½}\sqrt{2} && \text{½}\sqrt{2} \\
            \text{½}\sqrt{2} && \smn\text{½}\sqrt{2} \\
            };
            
    \draw[->,very thick] (\start1:1.4) arc[start angle=\start1, end angle=\xend1, radius=1.4];
    \draw[->,very thick] (\start2:1.4) arc[start angle=\start2, end angle=\xend2, radius=1.4];
    
    \node[anchor=east,font=\large] at (m.west) {$g(\mathbf{x})=~$};
    
    \node[inner sep=7em] at (0, 0) {};
\end{tikzpicture}

%% file: figures/rotate2.tikz
\begin{tikzpicture}[point1/.style={draw, circle, very thick, inner sep=0.5em,fill=edblue!60!white},
                        point2/.style={draw=edred,circle, very thick, inner sep=0.5em,pattern=south east lines, pattern color=edred},
                        mat/.style={matrix of math nodes, font=\large, left delimiter=(, right delimiter=),
                                    inner sep=0.15em, row sep=0.75em, column sep=0.5em,
                                    row 1/.style={nodes={color=edblue}},
                                    row 2/.style={nodes={color=edred}}}]
    \tikzmath{\x1 = 90;}
    \tikzmath{\x2 = \x1 + 180;
              \start1 = \x1 - 20;
              \xend1 = \start1 - 25;
              \start2 = \x2 - 20;
              \xend2 = \start2 - 25;}
    
    \node[point1] at (\x1:1.4) {1};
    \node[point2] at (\x2:1.4) {2};
    
    \matrix[mat, anchor=north] (m) at (0, -2.5) {
            0 && 1 \\
            0 && \smn1 \\
            };
            
    \draw[->,very thick] (\start1:1.4) arc[start angle=\start1, end angle=\xend1, radius=1.4];
    \draw[->,very thick] (\start2:1.4) arc[start angle=\start2, end angle=\xend2, radius=1.4];
    
    \node[anchor=east,font=\large] at (m.west) {$g(\mathbf{x})=~$};
    
    \node[inner sep=7.5em] at (0, 0) {};
\end{tikzpicture}

%% file: figures/rotate3.tikz
\begin{tikzpicture}[point1/.style={draw, circle, very thick, inner sep=0.5em,fill=edblue!60!white},
                        point2/.style={draw=edred,circle, very thick, inner sep=0.5em,pattern=south east lines, pattern color=edred},
                        mat/.style={matrix of math nodes, font=\large, left delimiter=(, right delimiter=),
                                    inner sep=0.05em, row sep=0.5em, column sep=0.25em,
                                    row 1/.style={nodes={color=edblue}},
                                    row 2/.style={nodes={color=edred}}}]
    \tikzmath{\x1 = 45;}
    \tikzmath{\x2 = \x1 + 180;
              \start1 = \x1 - 20;
              \xend1 = \start1 - 25;
              \start2 = \x2 - 20;
              \xend2 = \start2 - 25;}
    
    \node[point2] at (\x1:1.4) {1};
    \node[point1] at (\x2:1.4) {2};
    
    \matrix[mat, anchor=north] (m) at (0, -2.5) {
            \smn\text{½}\sqrt{2} && \smn\text{½}\sqrt{2} \\
            \text{½}\sqrt{2} && \text{½}\sqrt{2} \\
            };
            
    \draw[->,very thick] (\start1:1.4) arc[start angle=\start1, end angle=\xend1, radius=1.4];
    \draw[->,very thick] (\start2:1.4) arc[start angle=\start2, end angle=\xend2, radius=1.4];
    
    \node[anchor=east,font=\large] at (m.west) {$g(\mathbf{x})=~$};
    
    \node[inner sep=7em] at (0, 0) {};
\end{tikzpicture}

%% file: figures/rotate4.tikz
\begin{tikzpicture}[point1/.style={draw, circle, very thick, inner sep=0.5em,fill=edblue!60!white},
                        point2/.style={draw=edred,circle, very thick, inner sep=0.5em,pattern=south east lines, pattern color=edred},
                        mat/.style={matrix of math nodes, font=\large, left delimiter=(, right delimiter=),
                                    inner sep=0.15em, row sep=0.75em, column sep=0.5em,
                                    row 1/.style={nodes={color=edblue}},
                                    row 2/.style={nodes={color=edred}}}]
    \tikzmath{\x1 = 0;}
    \tikzmath{\x2 = \x1 + 180;
              \start1 = \x1 - 20;
              \xend1 = \start1 - 25;
              \start2 = \x2 - 20;
              \xend2 = \start2 - 25;}
    
    \node[point2] at (\x1:1.4) {1};
    \node[point1] at (\x2:1.4) {2};
    
    \matrix[mat, anchor=north] (m) at (0, -2.5) {
            \smn1 && 0 \\
            1 && 0 \\
            };
            
    \draw[->,very thick] (\start1:1.4) arc[start angle=\start1, end angle=\xend1, radius=1.4];
    \draw[->,very thick] (\start2:1.4) arc[start angle=\start2, end angle=\xend2, radius=1.4];
    
    \node[anchor=east,font=\large] at (m.west) {$g(\mathbf{x})=~$};
    
    \node[inner sep=7.5em] at (0, 0) {};
\end{tikzpicture}

%% file: figures/rotate5.tikz
\begin{tikzpicture}[point1/.style={draw, circle, very thick, inner sep=0.5em,fill=edblue!60!white},
                        point2/.style={draw=edred,circle, very thick, inner sep=0.5em,pattern=south east lines, pattern color=edred},
                        mat/.style={matrix of math nodes, font=\large, left delimiter=(, right delimiter=),
                                    inner sep=0.05em, row sep=0.5em, column sep=0.25em,
                                    row 1/.style={nodes={color=edblue}},
                                    row 2/.style={nodes={color=edred}}}]
    \tikzmath{\x1 = -45;}
    \tikzmath{\x2 = \x1 + 180;
              \start1 = \x1 - 17;
              \xend1 = \start1 - 25;
              \start2 = \x2 - 17;
              \xend2 = \start2 - 25;}
    
    \node[point2] at (\x1:1.4) {1};
    \node[point1] at (\x2:1.4) {2};
    
    \matrix[mat, anchor=north] (m) at (0, -2.5) {
            \smn\text{½}\sqrt{2} && \text{½}\sqrt{2} \\
            \text{½}\sqrt{2} && \smn\text{½}\sqrt{2} \\
            };
    
    \node[inner sep=7em] at (0, 0) {};
    
    \node[anchor=east,font=\large] at (m.west) {$g(\mathbf{x})=~$};
\end{tikzpicture}

%% file: figures/interpolation_cartoon_2.tikz
\begin{tikzpicture}[font=\Large,
                    point/.style={circle,draw,fill,inner sep=0.18em}]
    \begin{scope}[z={(90:3.6em)},x={(0:4em)},y={(50:2em)}]
    
        \begin{scope}[canvas is xy plane at z=1,color=black]
            \draw[ultra thick] (-1,-1) -- (-1, 1) -- (1, 1) -- (1, -1) -- cycle;
            
            \node at (0, 0) [font=\Large,color=black] {$I_5^+$};
        \end{scope}
        
        \begin{scope}[canvas is xy plane at z=-1,color=black]
            \draw[ultra thick] (-1,-1) -- (-1, 1) -- (1, 1) -- (1, -1) -- cycle;
            
            \node at (0, 0) [font=\Large,color=black] {$I_5^-$};
        \end{scope}
        
        \draw[loosely dashed] (1, 1, -1) -- (1, 1, 1);
        \draw[loosely dashed] (1, -1, -0.65) -- (1, -1, 1);
        \draw[loosely dashed] (1, -1, -0.85) -- (1, -1, -1);
        \draw[loosely dashed] (-1, -1, -1) -- (-1, -1, 1);
        \draw[loosely dashed] (-1, 1, -1) -- (-1, 1, 1);
        
        \node at (0, 0, 0) [font=\huge,color=black] {$I_5$};
    
        \begin{scope}[canvas is yz plane at x=1]
            \draw[ultra thick] (-0.5,1) -- (-0.5,-1);
            
            \node[point] at (-0.5, 1) {};
            \node[point,fill=white!50!gray] at (-0.5, 0.6) {};
            \node[point,fill=edred!80!white] at (-0.5, 0.4) {};
            \node[point,fill=white!50!gray] at (-0.5, 0.2) {};
            \node[point,fill=white!50!gray] at (-0.5, -0.2) {};
            \node[point,fill=white!50!gray] at (-0.5, -0.6) {};
            \node[point] at (-0.5, -1) {};
        \end{scope}
        
        \node[font=\large] at (1.3, -0.5, 0.95) {$\mathbf{x}^+$};
        \node[font=\large,text=edred] at (1.22, -0.5, 0.3) {$\mathbf{x}$};
        \node[font=\large] at (1.3, -0.5, -1.05) {$\mathbf{x}^-$};
        
        \begin{scope}[canvas is xz plane at y=-0.5]
            \draw[thick] (1.55, 1.2) -- (0.7, 1.2) -- (0.7, -1.3) -- (1.55, -1.3) -- cycle;
            \draw[thick, dashed] (1.55,  1.2) -- (2.5, 4);
            \draw[thick, dashed] (1.55, -1.3) -- (2.5, -4);
            \draw[thick] (3.5, 4) -- (2.5, 4) -- (2.5, -4) -- (3.5, -4);
            \draw[thick, dashed] (3.5, 4) -- (4.5, 4);
            \draw[thick, dashed] (3.5, -4) -- (4.5, -4);
        \end{scope}
        
        \tikzset{shift={(3.25,-0.5,0)}}
        
        \begin{scope}[canvas is xz plane at y=0,
                      point/.style={circle,draw,fill,inner sep=0.24em}]
            \draw[ultra thick] (0,3.2) -- (0,-3.2);
            
            \node[point] at (0, 3.2) {};
            \node[point,fill=white!50!gray]  at (0, 1.92) {};
            \node[point,fill=edred!80!white] at (0, 1.28) {};
            \node[point,fill=white!50!gray]  at (0, 0.64) {};
            \node[point,fill=white!50!gray]  at (0, -0.64) {};
            \node[point,fill=white!50!gray]  at (0, -1.92) {};
            \node[point] at (0, -3.2) {};
            
            \node[font=\Large] at (0.7,  3.3) {$\mathbf{x}_{(1)}\scriptstyle{\,=\,\mathbf{x}^+}$};
            \node[font=\Large] at (0.4,  2.02) {$\mathbf{x}_{(2)}$};
            
            \node[font=\Large,text=edred] at (1.02, 1.32) {$\mathbf{x} = \tfrac{\mathbf{x}_{(2)} + \mathbf{x}_{(3)}}{2}$};
            
            \node[font=\Large] at (0.4, 0.44) {$\mathbf{x}_{(3)}$};
            \node[font=\Large] at (0.4, -0.84) {$\mathbf{x}_{(4)}$};
            \node[font=\Large] at (0.4, -2.12) {$\mathbf{x}_{(5)}$};
            \node[font=\Large] at (0.7, -3.4) {$\mathbf{x}_{(6)}\scriptstyle{\,=\,\mathbf{x}^-}$};
            
            \draw[very thick,-{Stealth[width=1em,length=1.5em]}] (1, 0) -- (2.75, 0);
            \node[font=\Large] at (1.75, 0.2) {$\nu_5$};
            
            \matrix[{matrix of math nodes}, ampersand replacement=\&,
                    column sep={2em,between origins}, row sep={0em,between origins},
                    matrix anchor=tln.center, anchor=base,
                    font=\large, row 3/.style={text=edred!90!red,nodes={fill=white}}
                    ] (m) at (3.25, 3.3)
 {
    |(tln)| \big( 1 \& \tfrac{1}{2} \&   \tfrac{1}{2}   \&        0         \& \tfrac{\smn1}{2} \big) \&[0.6em] \scriptstyle{=\nu_5(\mathbf{x}_{(1)})} \\[4.508em]
    \big( 1 \& \tfrac{1}{2} \&   \tfrac{1}{2}   \&        0         \& \tfrac{\smn1}{2} \big) \& \scriptstyle{=\nu_5(\mathbf{x}_{(2)})} \\[2.7em]
    \big( 1 \& \tfrac{1}{4} \&   \tfrac{1}{4}   \&        0         \& \tfrac{\smn1}{2} \big) \& \scriptstyle{=\nu_5(\mathbf{x})~~\,} \\[2.7em]
    \big( 1 \&      0       \&        0         \&        0         \& \tfrac{\smn1}{2} \big) \& \scriptstyle{=\nu_5(\mathbf{x}_{(3)})} \\[4.608em]
    \big( 1 \&      0       \& \tfrac{\smn1}{2} \& \tfrac{\smn1}{2} \& \tfrac{\smn1}{2} \big) \& \scriptstyle{=\nu_5(\mathbf{x}_{(4)})} \\[4.608em]
    \big( 1 \&      0       \& \tfrac{\smn1}{2} \&      \smn1       \&      \smn1       \big) \& \scriptstyle{=\nu_5(\mathbf{x}_{(5)})} \\[4.608em]
    \big( 1 \&      0       \& \tfrac{\smn1}{2} \&      \smn1       \&      \smn1       \big) \& \scriptstyle{=\nu_5(\mathbf{x}_{(6)})} \\
 };
            \begin{scope}[on background layer,
                          interp/.style={line width=0.6pt,->}]
                \draw[interp] (tln.south) -- (m-2-1.north);
                \draw[interp] (m-1-2.south) -- (m-2-2.north);
                
                \draw[interp] (m-2-2.south) -- (m-4-2.north);
                \draw[interp] (m-2-3.south) -- (m-4-3.north);
                
                \draw[interp] (m-4-3.south) -- (m-5-3.north);
                \draw[interp] (m-4-4.south) -- (m-5-4.north);
                
                \draw[interp] (m-5-4.south) -- (m-6-4.north);
                \draw[interp] (m-5-5.south) -- (m-6-5.north);
                
                \draw[interp] (m-6-5.south) -- (m-7-5.north);
            \end{scope}
        
        \end{scope}

    \end{scope}
\end{tikzpicture}